%% file: camera_ready.tex
\documentclass{article}

    \PassOptionsToPackage{numbers, compress}{natbib}
\usepackage[final,main]{neurips_2026}
\input{math_commands.tex}

\usepackage[utf8]{inputenc} % allow utf-8 input
\usepackage[T1]{fontenc}    % use 8-bit T1 fonts
\usepackage{hyperref}       % hyperlinks
\usepackage{url}            % simple URL typesetting
\usepackage{booktabs}       % professional-quality tables
\usepackage{amsfonts}       % blackboard math symbols
\usepackage{nicefrac}       % compact symbols for 1/2, etc.
\usepackage{microtype}      % microtypography
\usepackage{xcolor}         % colors
\usepackage{enumitem}
\usepackage{booktabs}
\usepackage{graphicx}
\usepackage{multirow}
\usepackage{caption}
\usepackage{array}
\usepackage{wrapfig}
\usepackage{pifont}
\usepackage{float}
\usepackage{xspace}
\usepackage{subcaption}
\usepackage{colortbl}
\usepackage{amsthm}
\usepackage{placeins}
\newtheorem{thm}{Theorem}[section]
\newcommand{\jiayu}[1]{#1}

\newcommand{\xmark}{\ding{55}\xspace}
\newcommand{\cmark}{\ding{51}\xspace}
\newcommand{\Na}{\textendash}
\definecolor{fractsrow}{RGB}{226,237,250}
\definecolor{headband}{RGB}{240,240,240}
\definecolor{goodgreen}{RGB}{22,118,68}
\definecolor{badred}{RGB}{192,57,43}
\newcommand{\dgood}[1]{\textcolor{goodgreen}{#1}}
\newcommand{\dbad}[1]{\textcolor{badred}{#1}}

\hypersetup{
    colorlinks=true,
    citecolor=[RGB]{31,91,158},
    linkcolor=[RGB]{31,91,158},
    urlcolor=[RGB]{31,91,158}
}
\definecolor{zebra}{RGB}{247,249,252}
\newcommand{\inlinecaption}[3]{%
  \begingroup\captionsetup{hypcap=false}\captionof{#1}{#2}\label{#3}\endgroup}
\makeatletter
\renewenvironment{table}
  {\setlength{\abovecaptionskip}{7pt}%
   \setlength{\belowcaptionskip}{7pt}%
   \@float{table}}
  {\end@float}
\makeatother
\title{FracTS: Hierarchical and Autoregressive Time Series Generation}

\author{
  \textbf{Jiayu Li}$^{2,\dagger,}$\thanks{Equal contributions. Work done while working at Betterdata AI} \quad
  \textbf{Umair Afzal}$^{1,*}$ \quad
  \textbf{Zilong Zhao}$^{1,3,}$\thanks{Correspondence: \texttt{jiayul11@illinois.edu, zilong@betterdata.ai}} \\[0.4ex]
  \textbf{Milad Abdollahzadeh}$^{1}$ \quad
  \textbf{Uzair Javaid}$^{1}$ \quad
  \textbf{Biplab Sikdar}$^{3}$ \\[0.6ex]
  $^{1}$Betterdata AI\quad
  $^{2}$University of Illinois Urbana-Champaign \\[0.2ex]
  $^{3}$National University of Singapore \\[0.6ex]
}

\begin{document}

\maketitle

\begin{abstract}
  Time series are ubiquitous across different domains, and the growing use of synthetic data for privacy-preserving sharing, data augmentation, and stress testing makes \emph{time series generation} increasingly important. Early generative approaches are predominantly autoregressive, which can suffer from low sample diversity, error accumulation, and difficulty modeling long-range dependencies. In this paper, we introduce \textbf{FracTS}, a hierarchical, coarse-to-fine autoregressive model for multivariate time series generation. Inspired by fractal generative modeling developed for images, FracTS adapts the paradigm \jiayu{from 2D} to \jiayu{1D} sequences by learning a multi-scale representation that captures global trajectory structure at coarse scales and local dynamics at fine scales while following the autoregressive nature of time series. 
  \jiayu{In contrast to unstructured images, time series are structured, allowing easy automated or designated temporal feature processing optimized for the model, and global (static) feature extraction describing the instances that generate the corresponding time series in the collection. We incorporate them in the generation process to further improve the generation quality.}
  % We adapt the model to supports both unconditional and conditional generation via static attributes, and incorporate a length predictor to generate time series with variable sequence lengths. 
  The design naturally handles multivariate dependencies and improves long-horizon temporal coherence without requiring fixed-length inputs.
% Evaluation covers (i) marginal distribution (ii) correlation preservation and (iii) ML utility comparing between real and synthetic data. 
Experiments on five datasets (two conditional, three unconditional)
conducted against seven baselines show that FracTS consistently outperforms baselines across \jiayu{marginal distribution and predictive ML utility, and correlation preservation in the unconditional setting}. \jiayu{In particular, }the predictive ML utility is averagely improved by 100.0\% and 75.0--81.0\% for unconditional and conditional time series datasets compared to the best baseline (relative error computed on dataset-averaged raw values).
Code is available at: \url{https://github.com/BetterdataLabs/fracts}.
\end{abstract}
% \jiayu{\textit{Comment: I removed conditional vs unconditional in abstract, it's distracting. I tried to explain that we're working on collections of time series instead of a single time-series also in abstract, as Milad seemed so confused by this previously, and I guess some other folks may be confused too. Please help see if the explanation is clear enough.}}
% \jiayu{\textit{Comment: Also, since now much of the content is not generated by the same person, I observe significant wording inconsistence. We need to check this out and align the choices (e.g., time-series vs time series).}}

\section{Introduction}
Time series data – sequences of observations over time – are ubiquitous across scientific, engineering, and social domains. Such data arise in diverse applications ranging from finance and economics to healthcare and manufacturing \citep{DBLP:journals/natmi/TurbeBLM23,DBLP:journals/jbd/StengerLFKB24}. 
% For example, IoT sensors in factories (engineering/manufacturing), medical devices in hospitals (healthcare), personal smart gadgets, and financial transaction systems (economic/social sciences) all generate vast amounts of sequential time-stamped data. 
This pervasiveness of time series has driven extensive research in tasks like forecasting, classification, and anomaly detection on temporal data.
Generative modeling, in particular, of time series serves several critical purposes in data science and machine learning and widely utilized in real-life applications: simulating plausible trajectories for what-if and stress testing~\citep{DBLP:journals/corr/abs-2506-06380}; augmenting scarce or imbalanced datasets~\citep{DBLP:journals/mbec/MiljkovicMPS24}; privacy-preserving data sharing and benchmarking~\citep{DBLP:journals/jbd/StengerLFKB24}; and robust evaluation of downstream decision systems~\citep{NEURIPS2023_05fb0f4e}. %Effective time series generators must therefore balance multiple desiderata: fidelity to marginal and joint distributions, realistic temporal dynamics (e.g., seasonality and long-range dependence), diversity rather than mode collapse, controllability via conditioning variables, and stability over long horizons, all while remaining computationally efficient.

By the diversity of data sources and applications, time series generation can be posed under several settings. 
\jiayu{In this paper, we focus on the generation of \textbf{\textit{new time series instances}} \textit{(resp. new time series steps)} \textbf{\textit{in a collection}} of time series instances (e.g., new users' transaction history, new companies' stocks). Thus, \textbf{forecasting} or \textbf{imputation}, which are also referred to as ``time series generation'' occasionally, are not discussed in this paper. 
Differentiated by the presence of explicitly known global features of the new instances, we classify the time series generation problems we cover into two major types: \textbf{unconditional} and \textbf{conditional} generation.} 
Early pioneers of time series generative models have mainly been autoregressive predictions of next steps~\citep{tforcing,pforcing,wavenet}, as autoregression aligns closely with the nature of time series, where new values are obtained or observed with dependency on the context of the past. 
However, autoregressive prediction typically suffers from small variance due to its deterministic nature~\citep{timegan}, accumulates error over time steps~\citep{timegan,aecgan}, and lacks understanding of long-range dependency and global information~\citep{ar-limit}. Consequently, many new time series generative models attempt to introduce auxiliary global components~\citep{timegan,diffusionts,tvqvae,aecgan,coscigan,kovae}.
That time series is inherently step-wise autoregressive, but requires auxiliary modules to complement a global understanding, naturally calls for hierarchical modeling approaches to reconcile local and global information in a coarse-to-fine manner. Some existing time series generative models introduce hierarchical components specifically for time series data~\citep{wavenet,tvqvae}. Nevertheless, generative models that are hierarchical by themselves~\citep{laddervae,proggan,hngan,nvae,nfldm,fractal} seem to be a more natural choice, but there is a lack of direct applications of these models on time series, except for limited exceptions~\citep{psagan}. 

In this paper, we introduce \textbf{FracTS} - a \textbf{Frac}tal Generative Model (FGM) ~\citep{fractal} for \textbf{T}ime \textbf{S}eries. 
% time series generative model that 
% FracTS
% is hierarchically autoregressive by design.
%, based on the generative model — fractal generative model (FGM) ~\citep{fractal}. 
Time series data possesses the fractal property like images~\citep{fractal}, as smaller chunks of time series are also time series. Inspired by this, we split long time series into shorter chunks, capturing the inter-and intra-chunk relations at two (full and chunked) fractal levels. This process continues until each chunk has length one, which is then modeled by a special final fractal level that learns across different variates, analogous to channels in images.
%We apply this process recursively, until each chunk has only length 1, which is learned by the special last fractal level learning different variates, analogous to channels in images.

%Time series data also possesses the fractal property like images~\citep{fractal}, as smaller chunks of time series are also time series. Inspired by this, we recursively split long time series into smaller chunks, learning the inter-and intra-chunk data relations at two fractal levels, until the chunks are short enough to be well learned by a simple model. An orthogonal dimension of the fractal property on multivariate time series data is that the projection of a multivariate time series on any subset of variates are still time series data, similarly to the role that channels play in images. Henceforth, we put the separation of different variates as a last fractal level. 

Unlike images where patches do not have specific orders, time series are naturally autoregressive. Thus, we adopt the autoregressive version of FGM~\citep{fractal} instead of the default choice of masked autoregressive version for images.
By a backbone of autoregressive transformer, the model naturally inherits the property of vanilla autoregressive transformers in handling varying sequence lengths~\citep{gpt1}.
% Another key aspect for time series generation is the nature of time series data as a special form of tabular data, on which data processing and encoding play a particularly important roles in model performances~\citep{ctgan,tabdiff,tvqvae,diffusionts}. 
\jiayu{The structured nature of time series, specifically in comparison to the unstructured images, allows standard automated or domain-specific feature processing optimized for the generative models~\citep{ctgan,tabdiff,tvqvae,diffusionts}, including tabular-formatted type-specific data processing of temporal features, and global static feature extraction.}
\begin{figure}[!b]
\centering
\includegraphics[width=\linewidth]{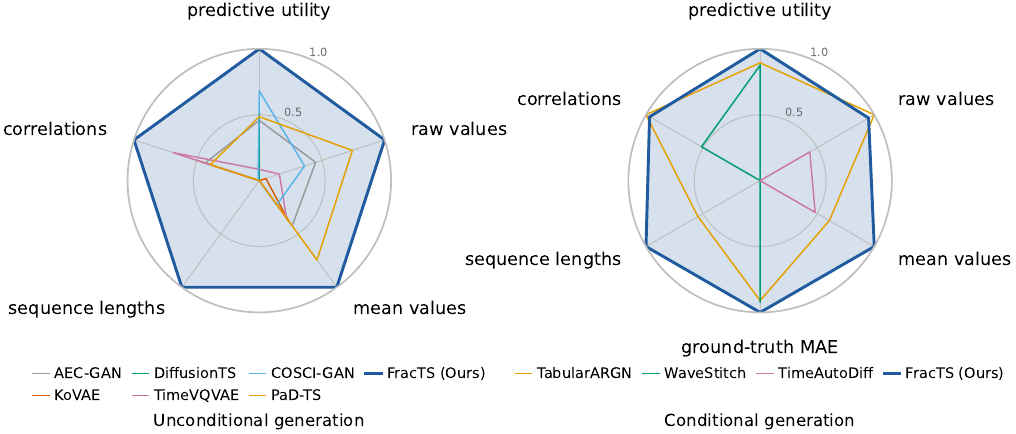}
\caption{Summary comparison of FracTS (ours) and the baselines on dataset-averaged scores of Tables~\ref{tab:uncond_results} and~\ref{tab:walmart_cond_quality}, rescaled per axis so that larger values indicate better performance. Clockwise from top: predictive ML utility at window $W{=}1000$, marginal distribution of raw values, of means and of lengths, and correlation preservation on raw values; conditional adds ground-truth MAE at the bottom (definitions in Section~\ref{sec:setup}).}
\label{fig:radar_summary}
\end{figure}
In this paper, we also introduce an effective data encoding for time series data in our FracTS model. Our main contributions can be summarized as follows:
\begin{enumerate}[leftmargin=14pt,itemsep=2pt,topsep=3pt]
    \item We propose FracTS model, an effective and intuitive hierarchically autoregressive generative model for conditional/unconditional generation of realistic time series data.
    \item We introduce a robust time series data processing \jiayu{and compression} pipeline \jiayu{optimized for transformer models.}
% , paired with a Perceiver-based condition compression module. This combination allows the model to efficiently process and generate heavily conditioned time series data without computational bottlenecks.
    % \item Our proposed FracTS model supports two complex time series settings: it can handle variable length sequences and time series with categorical features. 
    \item \jiayu{Our proposed FracTS model applies on much more flexible data settings for time series data without performance degradation in terms of lengths (variable, or very long) and features (diverse data types and distributions) than most existing models.}
    % {\color{red}[please make this point more detailed. For instance: Our proposed FracTS model can apply on variable length time series]}
    % \item Experiments demonstrate that FracTS consistently outperforms baselines across five datasets. For unconditional generation, FracTS surpasses the best baselines by up to 53.8\% in marginal distribution fidelity, 70.4\% in correlation preservation, and 100\% in predictive ML utility. In conditional generation, improvements reach up to 50.0\% in distribution fidelity and 81.0\% in predictive ML utility, and it reduces 38.35\% mean-absolute error comparing to the best baseline.
    \item Experiments demonstrate that FracTS consistently outperforms baselines across five datasets: compared to the best baseline, the predictive ML utility is improved by 100.0\% on unconditional datasets and by 75.0--81.0\% on conditional datasets, with consistent gains on the marginal distribution and correlation scores as well (Figure~\ref{fig:radar_summary}; per-column values in Tables~\ref{tab:uncond_results} and~\ref{tab:walmart_cond_quality}).
    % \jiayu{\textit{Comment: Can anyone draw a spider graph of the 6 metrics of all models and how in a figure? This 4 lines reading out all numbers is too verbose.}}
\end{enumerate}

\section{Related Works}
\label{sec:related}

\paragraph{Time Series Generation.}
Early time series generative models have mainly been step-wise autoregressive predictions due to the autoregressive nature of time series. For example, teacher forcing~\citep{tforcing} and professor forcing~\citep{pforcing} use recurrent neural networks to generate sequences autoregressively, and WaveNet~\citep{wavenet} autoregressively predicts the distribution of each step based on all previous steps with causal convolutions.
However, autoregressive prediction tends to accumulate errors and lacks global understanding~\citep{timegan,aecgan,ar-limit}. Therefore, recent models often introduce some global components, using modern generative methods such as generative adversarial networks (GAN)~\citep{gan}, variational autoencoders~\citep{vae}, diffusion models~\citep{ddpm}, and autoregressive transformers~\citep{transformer,gpt1}. For instance, TimeGAN~\citep{timegan} uses GAN to generate time series, which contains a global distribution objective on top of the next-step prediction objective, and DiffusionTS~\citep{diffusionts} uses a diffusion model to generate realistic time series data with both slow and rapidly-varying components for trend, seasonality, and error, respectively. Similarly, PaD-TS~\citep{padts} introduces a population-aware training objective to diffusion models, explicitly preserving dataset-level statistics such as cross-dimension correlations alongside individual-level authenticity. Even for some models that stick to autoregressive prediction for generation, specific designs to counter the drawbacks of autoregressive methods are introduced, such as the masked variant of autoregression with vector quantization in TimeVQVAE~\citep{tvqvae} and the error correction module that is trained globally in AEC-GAN~\citep{aecgan}.
Note that the detailed problem settings of different models vary in terms of the lengths of time series, number of variants and their types, and presence and nature of static conditions. Most existing models are not immediately applicable to all settings, while this paper aims to handle a general case that covers all scenarios (more details in Section~\ref{sec:prob}).

\paragraph{Hierarchical Understanding of Time Series and Hierarchical Generative Models.}
\label{sec:related:hier}
Hierarchical modeling is present in many time series generative models, explicitly or implicitly, such as the diluted causal convolution of WaveNet~\citep{wavenet}, and the separation of high and low frequency data in TimeVQVAE~\citep{tvqvae}. Hierarchical modeling is more prevalent in other time series tasks, such as forecasting and interpolation~\citep{dghl,gpht,himtm}.
However, there are hierarchical variants of most modern deep generative models, such as hierarchical hidden Markov models~\citep{hhmm}, VAEs~\citep{laddervae,nvae}, GANs~\citep{proggan,hngan}, diffusion models~\citep{nfldm}, and autoregressive transformers~\citep{fractal}, mostly designed for other data modalities.
There are applications of these hierarchical generative models to time series. For example, MR-HDMM~\citep{mrmts} generates time series via multi-rate hierarchy using deep Markov models, and PSA-GAN~\citep{psagan} generates time series using a progressive self-attention GAN. {Most recently, TimeMAR~\citep{xu2026timemar} proposes a multi-scale autoregressive framework using a dual-path VQ-VAE to disentangle trend and seasonal components for unconditional generation.\footnote{The code for TimeMAR is not publicly available, and thus it is not included as baseline.}} Nevertheless, to the best of our knowledge, there is no time series hierarchical generative model that generates data by fine-grained data units down to one variate of one time step, while natively supporting multivariate mixed-type features and complex static conditions.
% , especially in the realm of natural languages.

\paragraph{Time Series \jiayu{Structured} Data Processing.}
\label{sec:related:data}
Data processing is crucial to tabular modalities.
%, including plain tabular data and time series data. 
% For plain tabular data in particular, customized data processing for categorical and numeric data in tables is typically paired with corresponding input and output model layers and losses to obtain a generative model optimized for tabular data, while the network structure builds primarily on works in other data modalities. 
For instance, CTGAN~\citep{ctgan} uses a  simple GAN network, but the data processing of one-hot encoding for categorical data and variational Gaussian-mixture model decomposition of multimodal numeric data makes its performance outstanding.  TabDiff~\citep{tabdiff} uses
Gaussian and masked diffusion for numeric and categorical features, while the transformer-based denoising network backbone remains relatively standard. Extending this capability to the time domain, TimeAutoDiff~\citep{timeautodiff} employs a unified latent diffusion framework with a VAE backbone to natively model heterogeneous time series features, including categorical and binary variables. However, while time series data typically inherit the data properties of general tabular data, many time series generative models downstate this aspect, even overlooking the presence of categorical features~\citep{coscigan,tvqvae,kovae,diffusionts}, which will be addressed in this paper.
Additionally, some models leverage domain-specific processing, such as TimeVQVAE~\citep{tvqvae} and DiffusionTS~\citep{diffusionts}, which utilize Fourier transforms to capture multi-frequency dynamics.

\section{Problem Formulation}
\label{sec:prob}

\begin{table}[!htbp]
    \centering
    \caption{Functionality and flexibility of representative time series generative models, compared to FracTS. \xmark indicates that a model is \emph{limited to} the corresponding setting; for \emph{limited length}, \xmark means the model contains no specific design for long time series generation. Blank entries indicate that the setting is supported. $^{\dagger}$ marks models whose static
conditioning is restricted to class labels.}
    \label{tab:function}
    \small
    \resizebox{\linewidth}{!}{
    \begin{tabular}{l ccccccc >{\columncolor{fractsrow}}c}
        \toprule
        \rowcolor{headband}
        \textbf{Limited to} & \textbf{TimeGAN} & \textbf{WaveStitch} & \textbf{COSCI-GAN} & \textbf{TimeVQVAE} & \textbf{AEC-GAN} & \textbf{KoVAE} & \textbf{DiffusionTS} & \textbf{FracTS} \\
        \midrule
        i) fixed length    & \xmark &        & \xmark & \xmark          &        & \xmark & \xmark & \\
        ii) single variate &        &        &        & \xmark          &        &        &        & \\
        iii) no static     &        &        & \xmark & \xmark$^{\dagger}$ & \xmark & \xmark & \xmark$^{\dagger}$ & \\
        iv) limited length & \xmark & \xmark & \xmark & \xmark          &        & \xmark & \xmark & \\
        v) mixed type      &        &        & \xmark & \xmark          & \xmark & \xmark & \xmark & \\
        \bottomrule
    \end{tabular}
    }
\end{table}

\paragraph{General Time Series Setting.} Time series appear in several forms. For example, they can be \textit{uni}-or \textit{multi}-\textit{variate}, where the latter is more complex due to the cross-variable and lagged dependencies. Sequences may be \emph{irregularly sampled} (with explicit timestamps, usually \emph{event-based}), have \emph{variable length}, and their values may be \emph{continuous} and \emph{discrete/categorical}. Many time series are also accompanied with \emph{static conditions} describing the sequence as a whole (e.g. demographics, asset metadata), which can also have diverse data types.
In this paper, we use the most generic but complex setting --- multivariate time series generation of various lengths and diverse data types.

{ 
Let $\mathcal{D} = \{ (X_n, c_n) \}_{n=1}^{N}$ be a time series dataset comprising $N$ sequences. Each sequence $X_n = (x_{n,1}, \dots, x_{n,L_n})$ has $L_n$ time steps, where $L_n$ may vary across sequences. We set $L=\max_{n=1,\dots,N}L_n$ as the maximum length. 
At each step $t$, the observation $x_{n,t} \in \mathcal{X}_1 \times \dots \times \mathcal{X}_M$ is a multivariate measurement over $M$ features, whose domains $\mathcal{X}_m$ can be numerical, categorical, or datetime-valued.
Each sequence is accompanied by a static condition vector $c_n = (c_n^{(1)},\dots,c_n^{(M_c)})$ of $M_c$ attributes (categorical, numeric, or datetime).
}

{ 
The goal of generative modeling for time series is to learn a conditional distribution $p_{\theta}(X \mid c)$ that approximates the real data:
\begin{equation}
p_{\theta}(X \mid c) \approx p_{\text{data}}(X \mid c) \quad \text{for any } c \sim p_{\text{data}}(c).
\end{equation}
At generation time, $c$ may be specified by the user or sampled from a learned marginal $q_\phi(c) \approx p_{\text{data}}(c)$. 
Faithful approximation implies that generated sequences preserve key statistical properties of real time series, including:
\begin{itemize}
    \item[i)] the marginal distribution of each feature across all sequences and time steps: $p_{\theta}(X_{:,m}) \approx p_{\text{data}}(X_{:,m})$;
    \item[ii)] the temporal dynamics within a sequence: $p_{\theta}(x_{n,t} \mid x_{n,1:t-1}) \approx p_{\text{data}}(x_{n,t} \mid x_{n,1:t-1})$;
    \item[iii)] the instantaneous cross-feature dependencies: $p_{\theta}(x_{n,t}) \approx p_{\text{data}}(x_{n,t})$.
\end{itemize}
}

{ 
\textbf{Challenges of General Time Series Generation}
Most existing works simplify the setup considerably relative to our general formulation (variable length, multivariate mixed-type features, static conditions). Common restrictions are: i) fixed length $L_n = L$, ii) univariate series ($M=1$), iii) no static attributes ($M_c=0$), iv) short sequences (e.g., $L < 1000$), and v) purely numeric features ($\mathcal{X}_m \subseteq \mathbb{R}$). Table~\ref{tab:function} summarizes these simplifications; only a limited subset of prior works addresses the general case~\citep{wavestitch}, and in this work we relax these restrictions to better mirror real-world scenarios.
}

\newcommand{\rowcell}[2]{\parbox[c][#1][c]{\linewidth}{\centering #2}}
\begin{figure}[!t]
    \centering
    \includegraphics[width=0.95\linewidth]{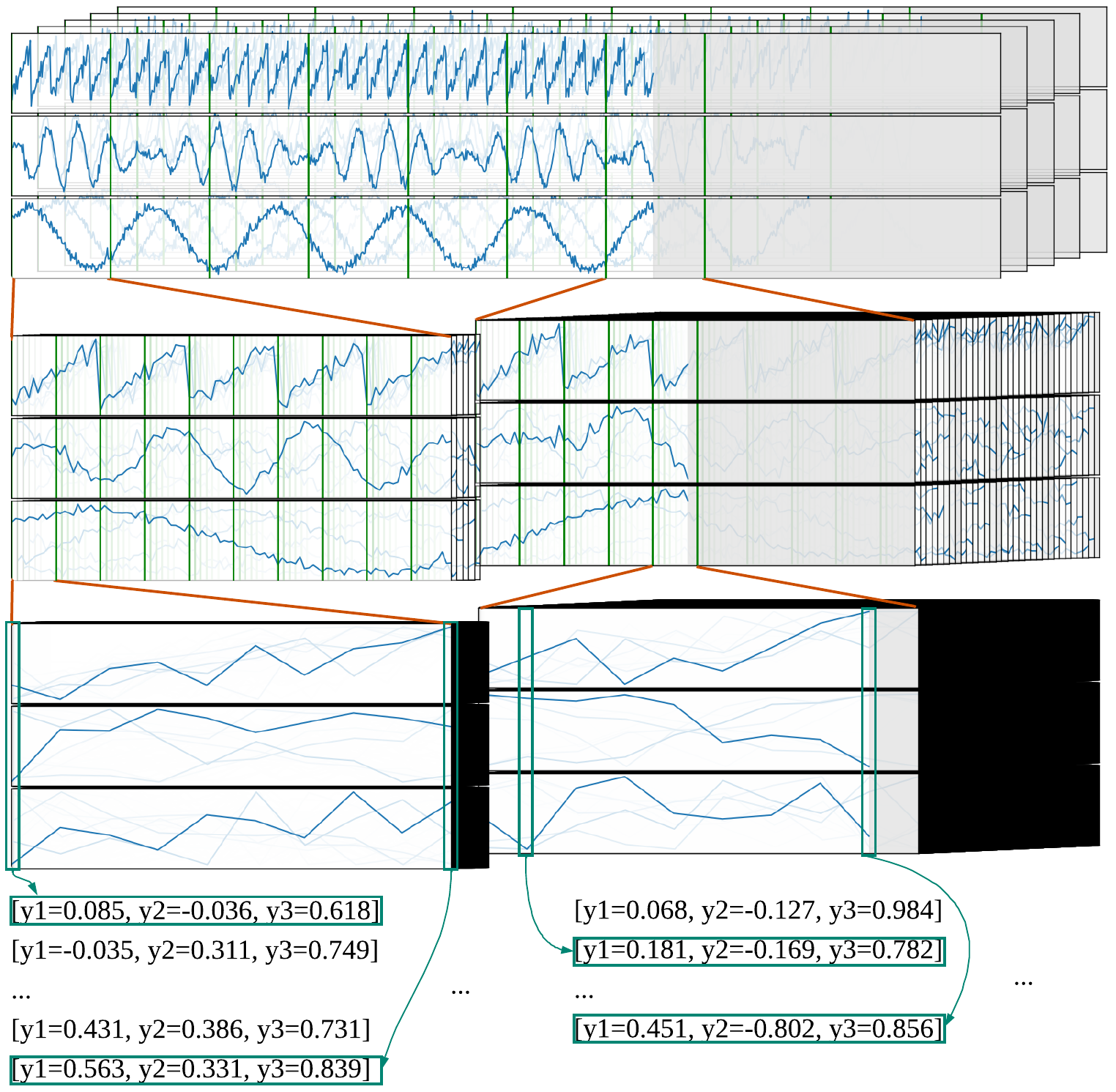}
    \caption{Fractal decomposition and model architecture for FracTS (with $K=4$ fractal levels). At each level except for the last one, a long time series is decomposed into patches of shorter ones.
    The time series dataset can have variable lengths, where the gray spans indicate paddings.  Completely empty patches are skipped for training.
    The last level $i=3$ is not a fractal decomposition of different time steps, but a series of different variables in \textit{one} time step.}
    \label{fig:fractal_ts_diag}
\end{figure}

% Given a time series dataset with $N$ data points $\x_1,\x_2,\dots,\x_N$, where each data point $\x_n$ ($n\in\{1,2,\dots,N\}$) consists of $L_n$ time steps of $M$ features, our goal is to generate another $N'$ time series data points with $M$ features similar to the provided real data. Let the maximum length be $L\ge L_n,\forall n$.
% Occasionally, the time series dataset has static features, where each $\x_n$ is accompanied by $M_s$ static features that describe a general profile of the time series data point. For example, in a dataset about transactions associated with different customers over a period of time, the static features can be the customer's profile, such as gender and age.
% In this paper, we focus on the generation of \textit{new time series data}, instead of forecasting \textit{new time steps} of existing time series.

% Prior works on time series generation usually do not solve the general problem statement above, but rather simplified versions, such as i) fixed length ($L_n=L,\forall n$), ii) single variate ($M=1$), iii) no static features ($M_s=0$), iv) limited time steps (small $L$, e.g., $L<1000$), and v) numerical values only (while time series and static features can be either categorical or numerical in most realistic datasets), with only a limited subset of works not limited to any simplified problem~\citep{wavestitch}. A summary of different prior works and the simplified version of problems they solve is found in Table~\ref{tab:function}. This paper, in comparison, aims to solve the general problem that is limited to any of the above constraints. 

\section{FracTS Model}
\label{sec:model}
% In this section, we introduce the adaptation of the original FGM for time series.
% \subsection{Fractal Generative Model (FGM)}
The underlying motivation of FGM is the existence of a fractal pattern in the data, namely, a recursive self-similar pattern, enabling divide-and-conquer generation. 
% Images can be generated by FGM because they are composed of sub-images, which are also images.
Time series data has such a fractal pattern, as long time series are composed of shorter time series. This hierarchical structure is also utilized in some existing time series models (recall Section~\ref{sec:related:hier}).
% , such as the super resolution in Temporal FiLM~\citep{tfilm}, and the separation of high and low frequency in TimeVQVAE~\citep{tvqvae}.
Nevertheless, time series and images are essentially different in many aspects, including i) time series is essentially 1D while images are 2D, ii) the shape of each data point for time series could vary due to variable time series lengths while it is fixed for most image datasets, iii) the distribution of each value unit is much more complex in time series as time series contain both discrete and continuous attributes with different distribution while the pixel values are integers approximately normally distributed, and same applies to the condition attributes of time series, in contrast to the only class labels for images in FGM. Therefore, FGM needs to be heavily customized to transfer from a visual model to a time series model. 
%We refer to our proposed customization of FGM as FracTS (\textbf{Frac}tal Generative Model for \textbf{T}ime \textbf{S}eries). 
% In this section, we introduce the adaptation of the original FGM for time series to handle the three major differences of time series and images listed above.
% Static conditions of different data points combined can be considered as tabular data, which can be generated by any tabular generative model that handles the data types needed~\citep{ctgan,realtabformer,tabsyn,tabdiff}. This paper focuses on the generation of time series, so we will not elaborate on this tabular model except for an ablation on the effect of the performance of this model.

\subsection{Fractal Architecture for 1D Time Series Data}

In FGM for images, each image is decomposed into patches, where each patch is a sub-image with smaller width and height. For time series, there are no widths and heights, but lengths, so we decompose a time series into patches where each patch is a shorter time series.
The last level of FGM for images is designed for exact values in each channel at one pixel, which we analogize to the exact values of different variables at one time step in time series, and this level has a length of 1 in terms of steps in the time series.
Thus, the FracTS decomposition of time series is shown in Figure~\ref{fig:fractal_ts_diag}.
Formally, in a fractal setting with $K$ levels, let the length of time series handled at fractal level $i=0,1,\dots,K-1$ be $L^{[i]}$ ($L^{[0]}=L$), then at each fractal level, FracTS decomposes a time series of length $L^{[i]}$ into shorter time series of length $L^{[i+1]}$, and the last level has $L^{[K-1]}=1$.
% In this paper, we use the default setting where $L^{[i]}=20L^{[i+1]}=20^{K-i-1}$.
% Since the model sizes can be configured, we only consider the cases where $L^{[i+1]}\mid L^{[i]}$ for simplicity.
% As such, level $i$ contains $L^{[i]}/L^{[i+1]}$ units for the next fractal level, which we refer to as the \textit{fractal factor} $f_i$.
% Therefore, in a fractal model with $K$ fractal levels, including the last level, $L^{[i]}=\prod_{k=i}^{K-2}f_k$. 
Following FGM for images, each fractal level is learned by a sequential model (i.e., transformer~\citep{transformer}).
Let the sequential model's decoder length at fractal level $i$ be $l^{[i]}$, then $l^{[i]}=\left\lceil L^{[i]}/L^{[i+1]}\right\rceil$ for $i\ne K-1$ and $l^{[K-1]}=M$.
% On a dataset $\mathcal{D}$ with $N$ data points, the top level has $N$ data points, but the other levels have more due to patchification. 
Let the number of data points at level $i$ be $N^{[i]}$, then $N^{[i]}=\sum_{n=1}^N\left\lceil L_n/L^{[i]}\right\rceil\le N\left\lceil L/L^{[i]}\right\rceil$.
% The values of the notations above at each fractal level are also indicated in Fig.~\ref{tab:fractal-ts}.
Therefore, we formally learn
\begin{multline}
    \log p(X|c)=\mathbb{E}_n\Bigg[\overbrace{\sum_{i=0}^{K-1}\frac{1}{l^{[i]}}\sum_{j=1}^{l^{[i]}}\log p(X_{n,jL^{[i+1]}+1:(j+1)L^{[i+1]}}\mid X_{1:jL^{[i+1]}},c)}^{\textrm{recursive fractal levels}}\\
    % \sqrt[\left\lfloor l^{[i-1]}/l^{[i]}\right\rfloor-2]{\prod_{j=1}^{\left\lfloor l^{[i-1]}/l^{[i]}\right\rfloor-2}p(X_{n,jl^{[i]}:(j+1)l^{[i]}}|X_{n,1:jl^{[i]}},c)}\\
    % \cdot \sqrt[L_n]{\prod_{t=1}^{L_n} p(X_{n,t,m}|X_{n,1:t-1},X_{n,t,1:m-1},c)}
    +\underbrace{\frac{1}{L_n}\sum_{t=1}^{L_n}\log p(X_{n,t,m}|X_{n,1:t-1},X_{n,t,1:m-1},c)}_{\textrm{final features level}}
    \Bigg]
\end{multline}

% A side effect of changing from 2D to 1D is the change of conditions passed to the next fractal level, as only the immediately preceding and succeeding patches are adjacent in 1D setting, and autoregressive setting allows only the preceding patches as context, which is insufficient to holistically understand the time series context, as
% % Nevertheless, this does not mean less information should be passed to the next level compared to images. 
% many critical time series patterns are embedded in extended horizons beyond adjacent windows, requiring capabilities in long dependencies for models handling them. For instance, whether a stock's value increases is not only related to the stock's value yesterday, but its overall performance in recent months, or even recent years. In other words, farther or even non-adjacent patches also tell important information about the next time steps' values. In contrast, images do not require such long-sightedness due to the inherently localized structure of visual data. Therefore, besides the adjacent patches that were used in the vanilla FGM, we also pass more adjacent and non-adjacent patches as the conditions of the next level. We make the relative positions of patches be passed to the next level a configurable list for each level in FracTS as hyperparameters. 

Another unique feature of 1D time series compared to 2D images is the need for long-term dependencies. Although global dependency is implicitly captured by patch embeddings inferred from a prefix including conditions from previous fractal levels, we argue that global information gradually vanishes down the fractal levels, similarly to the problem with deep neural networks with too many layers. To counter this effect, we adopt a ResNet-like~\citep{resnet} skip connections to directly pass the global condition from the previous fractal levels down to the next levels with a simple projector for space alignment.
% In this paper, we use XXX as the default.

Theorems~\ref{thm:error} and~\ref{thm:scale} (proofs in Appendix~\ref{app:theory}) verify in theory that the hierarchical architecture reduces error accumulation across long sequences (error scale $\Omega(\delta lK)$ versus $\Omega(\delta L)$ for a flattened autoregressive process) and justify its smaller memory usage compared to a flattened transformer with large $L$.

\begin{thm}
\label{thm:error}
    For a sequence of total length $L$ generated by a $K$-level hierarchical autoregressive architecture with branching factor $l$ (where $K=\log_l L$), if the local parameter estimation error per step is $\delta$, the total KL divergence between the ground-truth distribution $p(X_{\cdot,1:L})$ and the generated distribution $p(\hat{X}_{\cdot,1:L})$ satisfies:
    \begin{equation}
        D_{KL}\left(p(X_{\cdot,L}\mid c)\|p(\hat{X}_{\cdot,L}\mid c)\right)=\Omega(\delta lK).
    \end{equation}
\end{thm}

\begin{thm}
    \label{thm:scale}
    Let $B$ be the batch size, $L$ the sequence length, $d$ the hidden dimension, and $l$ the local patch size per level. For a hierarchical model like FracTS with $K = \log_l L$ levels, the total activation memory complexity is $\mathcal{O}(BL(l+d))$.
\end{thm}

\subsection{Learning and Constructing Time Series of Variable Lengths}
\label{sec:model:len}

Unlike fixed-size image datasets, time series datasets often contain variable-length instances due to heterogeneous user behaviors and company histories.
Similarly to the transformers for text~\citep{transformer,gpt1}, we apply attention masks to avoid learning positions with no data. Moreover, models are not trained on completely empty patches, which saves much of training time when $L\gg L_n$ for most $n$.

\newcommand{\qqplot}[1]{\raisebox{-2.2em}{\includegraphics[width=0.4\linewidth, trim={2.2em 2.5em 1em 2.3em}, clip]{figures/qq/#1.pdf}}}
\newcommand{\vheader}[1]{\rowcell{4em}{\rotatebox{90}{#1}}}
\begin{figure}[!htbp]
    \centering
    \includegraphics[width=\linewidth]{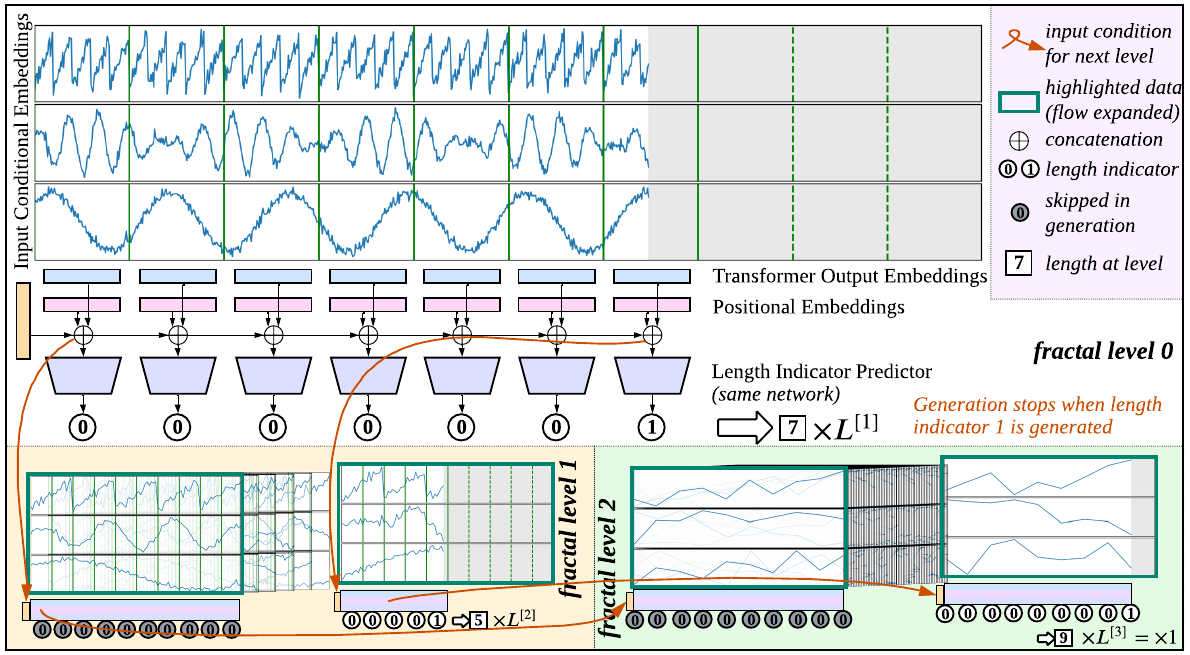}
    \caption{Length construction for variable-length time series data via length indicator predictors (purple trapezoids). At each fractal level, the length indicator is one network (drawn separately for separate calls).
    % Length indicators are predicted from input conditions, current fractal level's output embeddings, and positional embeddings.
    The length indicators are equivalent to is-[EOS] binary signals.
    Details at fractal levels 1-2 are omitted as they are the same as level 0, and the last level, unrelated to lengths, is hidden.
    % The last fractal level is not shown because it is not related to lengths.
    The total length of the example in the figure is $L_n=(7-1)L^{[1]}+(5-1)L^{[2]}+9L^{[3]}=649$.}
    \label{fig:len-ind}
\end{figure}

% \begin{figure}
%     \centering
%     \includegraphics[width=\linewidth]{figures/len-ind.pdf}
%     \caption{Length construction for variable-length time series data via length indicator predictors (purple trapezoids). At each fractal level, the length indicator is one network (drawn separately for separate calls). 
%     % Length indicators are predicted from input conditions, current fractal level's output embeddings, and positional embeddings. 
%     The length indicators are equivalent to is-[EOS] binary signals. 
%     Details at fractal levels 1-2 are omitted as they are the same as level 0, and the last level, unrelated to lengths, is hidden. 
%     % The last fractal level is not shown because it is not related to lengths. 
%     The total length of the example in the figure is $L_n=(7-1)L^{[1]}+(5-1)L^{[2]}+9L^{[3]}=649$.}
%     \label{fig:len-ind}
%     \vspace{-1em}
% \end{figure}

To determine where the time series ends, we use a \textit{length indicator predictor}. The idea is similar to the EOS token for natural languages.
We used an additional simple MLP-based binary classifier to predict whether a position corresponds to the end-of-sequence. We refer to this Boolean signal as the \textit{length indicator}. The process is visualized in Figure~\ref{fig:len-ind}. Since termination signals in time series are often implicit, we additionally provide sequence length as a static feature to the context generator $q_\phi$ and the length predictor. Unlike simple truncation, this design enables termination to depend on local dynamics, such as abrupt changes before a stock leaves the market.

% 
% Technically, from the output embeddings at each position, concatenated with the positional embedding (to emphasize the positional information) and the input conditions (to emphasize global information from upper levels), we predict a Boolean indicating whether this position is the last position with actual values in the sequence, which we call the \textit{length indicator}, via a simple MLP. Each fractal layer has a length indicator predictor module, except for the last one.
% Inputs to the length indicator predictor, along with the length indicator value per se, are passed to the next fractal level as input conditions to obtain the values as well as lengths at the next levels. 
% Fig.~\ref{fig:len-ind} summarizes the process.

\subsection{Handling Multivariate and Heterogeneous Mixed-Typed Features}

% Time series data are mostly stored in tabular format, and can be considered as tabular data with additional sequential correlations. Consequently, it inherits data properties essential to tabular data.
Time series data are typically stored in tabular form and can be viewed as tabular data with additional temporal dependencies.
% We encode each feature into separate positions for the last level's transformer.
To preserve tabular properties, we encode each feature as separate positions for the last level's transformer.
Categorical features are embedded by a vanilla embedding layer. Continuous features are decomposed into two positions: one K-Means binned category embedded similarly to categorical features, and one actual value padded to the embedding dimensions. To decode continuous values back, we take the second position only and pass it through a linear feed-forward layer. This decomposition of continuous value into bin and raw value is inspired by \citep{tabtreeformer} for learning multimodal numeric distributions, but we use the raw numeric value with projections at the second position. Having the auxiliary bin values leading the raw values introduces randomness by multinomial sampling of the leading discrete component. This compensates for the lack of randomness of generation based on embeddings only in transformers~\citep{transformer}.
% The randomness of transformer outputs from autoregressive sampling primarily arises from the multinomial sampling based on the softmax-normalized logits. Therefore, in a pure-numeric time series without static features, generated time series will look very alike.
This encoding method applies to both the input static features and the output time series features. 
% When training, we use cross-entropy loss for categorical positions and mean squared error for numeric positions.

Unlike plain tabular data, time series contain temporal trends. For example, a \$5 increase from yesterday may be as important as the current value itself. We analyzed the data of realistic time series datasets to compare the distributions of raw values and differences to previous time steps. The results are shown in Figure~\ref{fig:diff}, and more in Appendix~\ref{app:diff}. The distributions of differences are closer to normal distributions. In cases where both distributions are far from normal distribution, differences typically avoid one-tailedness, making the result closer to normal distribution. In an extreme but simple case of linearly increasing values, the raw values are uniformly distributed, but differences are constants. The latter is obviously much simpler than the former.
Therefore, to improve the model's capability in capturing the trend, we explicitly insert the difference to the previous time step for numeric features as an additional auxiliary leading numeric feature like the K-Means bins.
% Nevertheless, despite the usefulness of differences, we do not generate values depending merely on the differences because of potentially severe error accumulation over long time series, which is already an issue with long time series using raw values~\citep{aecgan}.
% For simplicity, this difference is not preceded by K-Means bins.

To provide more informative conditions, we compute simple summary statistics of each time series (e.g., mean and standard deviation) as additional static features. However, this produces many condition tokens that can overwhelm the context when directly prepended. To address this, we adopt a Perceiver-style condition compression module~\citep{perciever}, which distills these features into a compact representation, avoiding context overload. The full pipeline is illustrated in Figure~\ref{fig:data-proc}.

% small table wrapped on the right, raised; text flows beside and beneath it
\begin{wraptable}{r}{0.46\linewidth}
\vspace{-18pt}
    \centering
    \renewcommand{\arraystretch}{1.25}
    \captionof{table}{Summary of the datasets used and their core statistics. Notations follow Section~\ref{sec:prob}; $\overline{L_n}$ is the average sequence length.}
    \label{tab:datasets}
    \small
    \resizebox{\linewidth}{!}{
    \begin{tabular}{lrrrrrl}
        \toprule
        \rowcolor{headband}
        \textbf{Dataset} & $\bm{N}$ & $\bm{M_c}$ & $\bm{\overline{L_n}}$ & $\bm{L}$ & $\bm{M}$ & \textbf{Type} \\
        \midrule
        Weather   & 44   & 0 & 1798  & 2171  & 17 & unconditional \\
        Sensor    & 6939 & 0 & 51    & 139   & 17 & unconditional \\
        Stocks    & 1000 & 0 & 1278  & 1511  & 6  & unconditional \\
        Walmart   & 45   & 2 & 9368  & 10474 & 3  & conditional   \\
        Simulated & 100  & 4 & 15000 & 15000 & 15 & conditional   \\
        \bottomrule
    \end{tabular}
    }
\end{wraptable}
% \begin{table}[t]
%     \centering
%     \caption{Summary of the datasets used and their core statistics. Notations follow Section~\ref{sec:prob}; $\overline{L_n}$ is the average sequence length.}
%     \label{tab:datasets}
%     \footnotesize
%     \setlength{\tabcolsep}{4pt}
%     \begin{tabular}{lrrrrrl}
%         \toprule
%         \rowcolor{headband}
%         \textbf{Dataset} & $\bm{N}$ & $\bm{M_c}$ & $\bm{\overline{L_n}}$ & $\bm{L}$ & $\bm{M}$ & \textbf{Type} \\
%         \midrule
%         Weather   & 44   & 0 & 1798  & 2171  & 17 & unconditional \\
%         Sensor    & 6939 & 0 & 51    & 139   & 17 & unconditional \\
%         Stocks    & 1000 & 0 & 1278  & 1511  & 6  & unconditional \\
%         Walmart   & 45   & 2 & 9368  & 10474 & 3  & conditional   \\
%         Simulated & 100  & 4 & 15000 & 15000 & 15 & conditional   \\
%         \bottomrule
%     \end{tabular}
% \end{table}

\section{Experiments}
\label{sec:exp}

\subsection{Experiment Setup}
\label{sec:setup}
\label{sec:exp:setup}

% \subsubsection{Datasets}

% \paragraph{Real-world Multivariate Datasets.}

\paragraph{Datasets.}
We evaluate our models' unconditional generation ability on three real-world multivariate time series datasets from diverse domains, each posing distinct challenges due to their size, sequence length, and data characteristics: \textbf{Weather}~\citep{weatherdataset}, \textbf{Sensor}~\citep{sensordataset}, and \textbf{Stock}~\citep{stocksdataset}.
For conditional generation, we use \textbf{Walmart} dataset~\citep{walmartdataset}.
% More details of the setup can be found in Appendix~\ref{app:dataset}
%
%
% \paragraph{Simulated Multi-Pattern Dataset.}
We also construct a simulated dataset from mathematical functions to capture diverse temporal patterns, including linear, periodic, stepwise, convex, non-convex, modulated, and composite forms. Due to its mathematical design, this dataset is inherently conditional.
% We also create a simulated dataset from mathematical functions to represent a wide variety of temporal patterns and investigate the capability of different models on learning these patterns.
% This dataset is designed to be a conditional one due to its mathematical nature.
% This is a conditional dataset, and it covers linear, periodic, stepwise, convex, non-convex, modulated, and composite patterns.
% Details of the mathematical formulas are in Appendix~\ref{app:dataset:simulate}.
% More details of the datasets are provided in Appendix~\ref{app:dataset}.
Table~\ref{tab:datasets} summarizes the five datasets; more details are provided in Appendix~\ref{app:dataset}, with overall statistics in Appendix~\ref{app:dataset:overall} and the generating formulas of the simulated dataset in Appendix~\ref{app:dataset:simulate}.

\begin{figure}[!t]
\centering
\def\qqpanelH{6.0cm}
\def\qqImgH{138pt}
\begin{subfigure}[t]{0.48\linewidth}
    \centering
    \begin{minipage}[c][\qqpanelH][c]{\linewidth}
    \centering
    \setlength{\tabcolsep}{2pt}
    \resizebox*{!}{\qqImgH}{
    \begin{tabular}{
    >{\centering\arraybackslash}m{0.05\linewidth}
    >{\centering\arraybackslash}m{0.05\linewidth}
    cc
    }
        \toprule
        & & \textbf{Value} & \textbf{Difference} \\
        \midrule
        \vheader{Stock} & \vheader{Open} & \qqplot{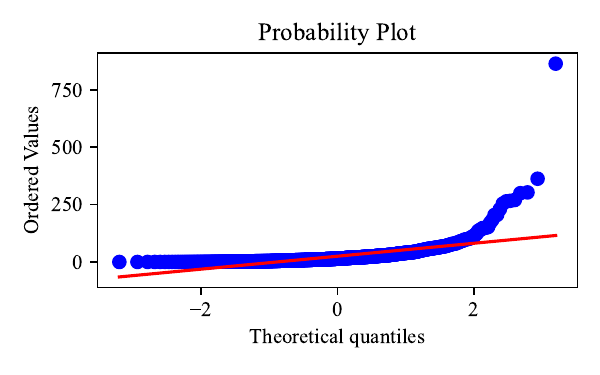} & \qqplot{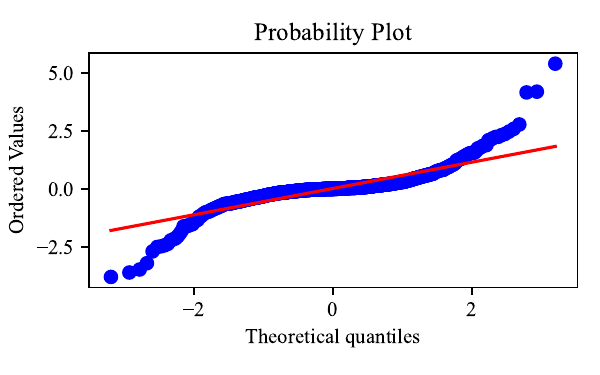} \\
        \vheader{Weather} & \vheader{MinTemp} & \qqplot{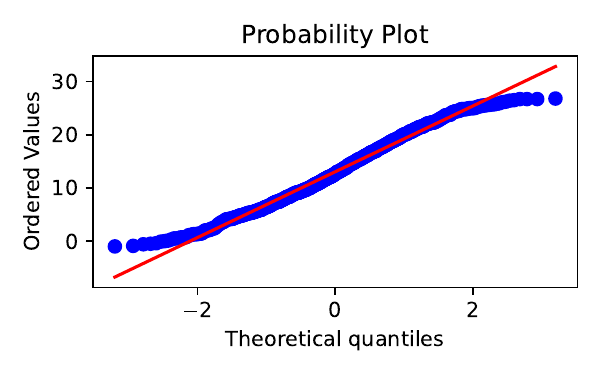} & \qqplot{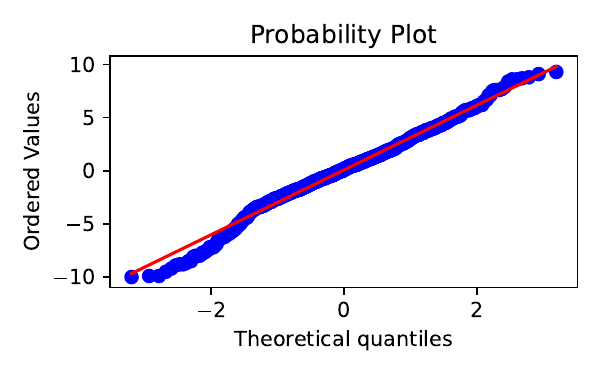}\\
        \vheader{Sensor} & \vheader{Data\_S} & \qqplot{sensor-data_S-values} & \qqplot{sensor-data_S-diff}\\
        \bottomrule
    \end{tabular}
    }
    \end{minipage}
    \caption{Q--Q plots comparing the distributions of raw values and differences of time series against the normal distribution. The closer the blue points lie along the red diagonal line, the more normally distributed the data.}
    \label{fig:diff}
\end{subfigure}\hfill
\begin{subfigure}[t]{0.50\linewidth}
    \centering
    \begin{minipage}[c][\qqpanelH][c]{\linewidth}
    \centering
    \includegraphics[height=\qqImgH]{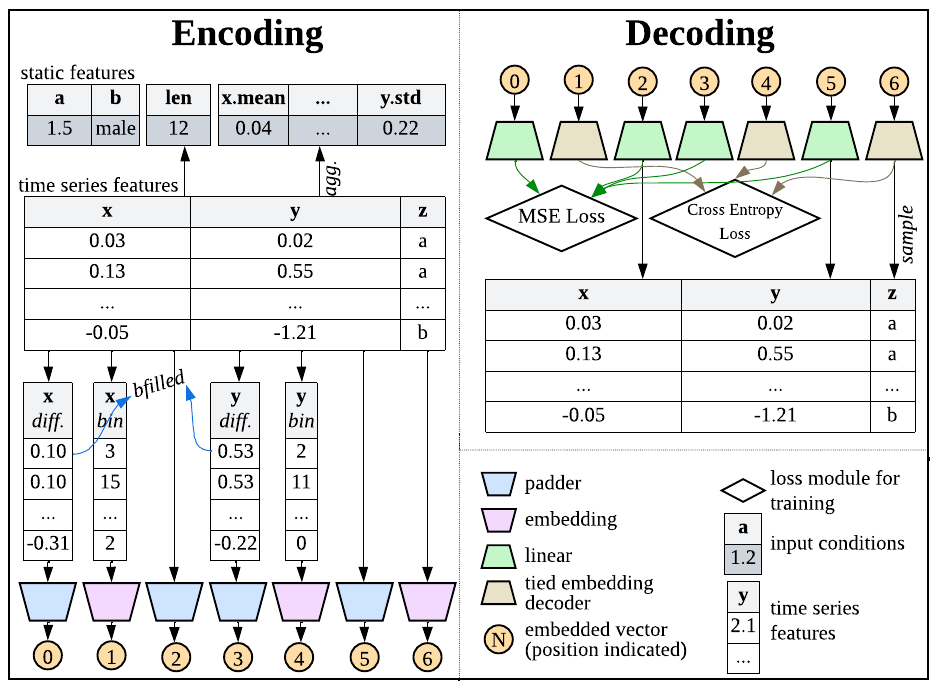}
    \end{minipage}
    \caption{Data processing of time series data. The left and right halves show the data \textbf{encoding} and \textbf{decoding} processes, respectively. Processing of static features, including length and aggregated values, is omitted as they follow that of time series features.}
    \label{fig:data-proc}
\end{subfigure}
\end{figure}

\paragraph{FracTS Implementation.} 
We implement FracTS using a hierarchical architecture with constant patch size $l^{[i]} = 20$ across all fractal levels $i \in \{1, \ldots, L\}$. Each fractal level employs a GPT2 small transformer~\citep{gpt1} as its base model $\mathcal{M}^{[i]}$. 
For optimization, we use AdamW~\citep{loshchilov2018decoupled} with hyperparameters: learning rate $\eta = 2 \times 10^{-4}$, weight decay $\lambda = 0.01$, and gradient clipping threshold $g_{\text{clip}} = 1.0$. Training is performed for a fixed number of $T = 200$ epochs, and the best model is selected based on the lowest validation loss observed during training. We evaluate model performance on the validation set after each epoch to monitor convergence.Further implementation details are provided in Appendix~\ref{app:fracts}, including the datetime and aggregation processing in Appendices~\ref{app:data:datetime} and~\ref{app:data:agg} and the full experimental configuration in Appendix~\ref{app:exp:config},
while component ablations and hyperparameter sensitivity analyses are
provided in Appendix~\ref{app:ablation}.

\paragraph{Baselines.}
For unconditional generation, we compare our model with recent baselines for synthetic time series generation. These include diffusion models: Diffusion-TS~\citep{diffusionts} and PaD-TS~\citep{padts}; GANs: COSCI-GAN~\citep{coscigan} and AEC-GAN~\citep{aecgan}; and VAEs: TimeVQVAE~\citep{tvqvae} and Koopman VAE (KoVAE)~\citep{kovae}. Since most of these methods do not support conditional generation, we use three different baselines for the conditional setting. These include an older but influential model, WaveStitch~\citep{wavestitch}, alongside two recent architectures: TabularARGN~\citep{tabularargn}, an autoregressive network, and TimeAutoDiff~\citep{timeautodiff}, a latent diffusion model. A detailed introduction for each baseline is provided in Appendix~\ref{app:baseline}.

\paragraph{Evaluation Metrics.}

% \paragraph{Unconditional Generation.}

% \label{subsec:uncond}

% \paragraph{Setup:}
% We evaluate unconditional quality by comparing many independent real time series (e.g., different stocks or ECGs) to an equally sized set of synthetic series generated by our model. Each series is a CSV with one or more columns (variables). For every column we extract simple, model-agnostic features that summarize distributional and temporal behavior: \emph{min}, \emph{max}, \emph{mean}, \emph{std}, \emph{skew}, \emph{kurtosis}, and \emph{autocorrelation at lag 1} (and a few similar basics). Aggregating these features across all files gives us a feature distribution for the real set and the synthetic set.

% \paragraph{Metrics:}

% We compute three complementary metrics of distributional similarity:
% (i) Kolmogorov--Smirnov (KS) statistic,
% (ii) Jensen--Shannon (JS) divergence,
% (iii) Cohen's $d$ (standardized mean difference).
% We also assess downstream utility with a forecasting task using two protocols: 
% \emph{TSTR} (Train on Synthetic, Test on Real) and \emph{TRTS} (Train on Real, Test on Synthetic). 
% We report RMSE for both and summarize stability via
% \[
% \Delta\mathrm{RMSE} \;=\; \big| \mathrm{RMSE}_{\text{TSTR}} - \mathrm{RMSE}_{\text{TRTS}} \big|.
% \]
% Smaller values indicate closer alignment of generalization behavior between real and synthetic sets. More details are found in Appendix~\ref{app:eval}.

We evaluate the quality of synthetic time series data in three aspects: (i) \textit{marginal distributions (MD)}: evaluated by Wasserstein distance, normalized
by the standard deviation of the real data, simple aggregation (mean, std., quantiles (5\%, 50\%, 95\%), skewness, auto-corr. (lag 1)) of each column, and the lengths respectively, adapted but augmented from common feature-based measures~\citep{tsgbench}; (ii) \textit{correlation preservation (CP)}: evaluated by MAE on the correlation matrix, commonly used for multivariate time series generation~\citep{coscigan}, but extended beyond raw values to aggregated results similarly to (i); (iii) \textit{predictive ML utility (PM)}: evaluated on the performance downstream application of next step prediction following the \textit{TSTR} (Train on Synthetic, Test on Real) framework~\citep{tsgbench,diffusionts}, for which we report MAE (mean absolute error) scores. We additionally report a windowed discriminative evaluation (discriminative accuracy minus $0.5$, lower is better) in Appendix~\ref{app:discriminative}.

All metrics for unconditional generation are also applied for conditional generation, except that static attributes are taken into consideration for CP and PM.
For simulated dataset in particular, the groundtruth (GT) time series can be calculated based on the generated static conditions according to their mathematical formulas. Thus, we also report the MAE (mean-absolute error) and MSE (mean-squared error) to the GT on this dataset. All results reported are based on three runs.
More details for all metrics are provided in Appendix~\ref{app:eval}.

% global variable change

\begin{table}[t]
\centering
\caption{Unconditional generation performance averaged over the three datasets (lower is better; best in \textbf{bold}). MD and CP are averaged across column pairs, and MAE uses standard-normalized values; PM is the $|\text{TSTR}-\text{TRTR}|$ MAE gap at prediction window $W$. The bottom row ($\Delta$best baseline) reports, for each column, the relative error of FracTS against the best baseline, computed on these dataset-averaged raw values. Per-dataset results with standard deviations over three runs are reported in Appendix~\ref{app:full-results}.}
\label{tab:uncond_results}
\small
\setlength{\tabcolsep}{5pt}
\resizebox{\linewidth}{!}{
\begin{tabular}{ll cccccc cc ccc}
\toprule
 & & \multicolumn{6}{c}{\cellcolor{headband}\textbf{MD} $\downarrow$} & \multicolumn{2}{c}{\cellcolor{headband}\textbf{CP} $\downarrow$} & \multicolumn{3}{c}{\cellcolor{headband}\textbf{PM} $\downarrow$} \\
\cmidrule(lr){3-8}\cmidrule(lr){9-10}\cmidrule(lr){11-13}
 & \textbf{Model} & raw & mean & quantile & skewness & autocorr. & length & raw & agg. & $W{=}10$ & $W{=}100$ & $W{=}1000$ \\
\midrule
 & AEC-GAN & 0.524 & 0.683 & 0.697 & 0.617 & 0.609 & \Na & 0.288 & 0.315 & $4\,052$ & $4\,171$ & $8\,233$ \\
 & KoVAE & 0.675 & 0.708 & 0.673 & 0.784 & 0.701 & 0.956 & 0.261 & 0.489 & $7\,461$ & $7\,680$ & $15\,156$ \\
 & DiffusionTS & 0.695 & 0.932 & 0.915 & 0.686 & 0.775 & 0.956 & 0.466 & 0.542 & $2\,524$ & $2\,598$ & $5\,129$ \\
 & TimeVQVAE & 0.634 & 0.726 & 0.700 & 0.496 & 0.732 & 0.956 & 0.183 & 0.443 & $6\,822$ & $7\,022$ & $13\,861$ \\
 & COSCI-GAN & 0.557 & 0.799 & 0.748 & 0.839 & 0.858 & 0.956 & 0.462 & 0.473 & $2\,392$ & $2\,462$ & $4\,858$ \\
 & PaD-TS & 0.412 & 0.495 & 0.513 & 0.618 & 0.714 & 0.956 & 0.305 & 0.470 & $3\,841$ & $3\,954$ & $7\,804$ \\
\rowcolor{fractsrow}
 & \textbf{FracTS (Ours)} & \textbf{0.314} & \textbf{0.344} & \textbf{0.360} & \textbf{0.270} & \textbf{0.356} & \textbf{0.442} & \textbf{0.054} & \textbf{0.222} & \textbf{0.006} & \textbf{0.007} & \textbf{0.010} \\
\midrule
 & $\Delta$best baseline & \dgood{23.87\%} & \dgood{30.39\%} & \dgood{29.89\%} & \dgood{45.50\%} & \dgood{41.52\%} & \dgood{53.75\%} & \dgood{70.44\%} & \dgood{29.42\%} & \dgood{100.00\%} & \dgood{100.00\%} & \dgood{100.00\%} \\
\bottomrule
\end{tabular}
}
\end{table}

\begin{figure}[!htbp]
  \centering
  \includegraphics[width=\linewidth, keepaspectratio]{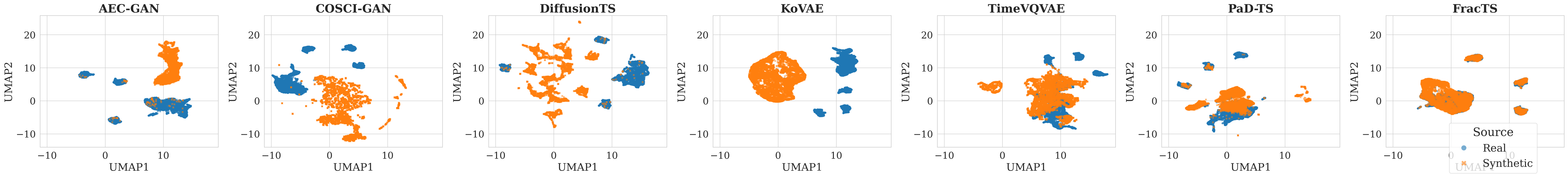}
  \caption{t-SNE of real compared to generated time series data of the Weather dataset.}
  \label{fig:weather_uncond_single_1x6}
\end{figure}
\subsection{Results on Unconditional Generation}

% As shown in Table~\ref{tab:uncond_results}, our model achieves the best scores on almost all metrics. 
{As shown in Table~\ref{tab:uncond_results}, our model achieves the best scores on all metrics, with up to 53.8\%, 70.4\%, and 100.0\% relative improvement from the best baseline in MD, CP, and PM, respectively (relative error computed on dataset-averaged raw values). Figure~\ref{fig:weather_uncond_single_1x6} visualizes the comparison of the outcomes from different models. The corresponding t-SNE visualization for the Sensors dataset is provided in Appendix~\ref{app:window_tsne}.}
This shows {our model} aligns well with the true distribution and works great for downstream tasks.
Additional windowed t-SNE visualizations for all datasets are provided in Appendix~\ref{app:window_tsne}.

% Noticeably, no baselines work on variable-length settings, and even if we add na\"ive zero padding pre- and post-processing (see Appendix~\ref{app:baseline}) to make them applicable, they hardly learn any useful length information, and the prevalence of zero values introduced by padding hinders these models from learning true data distribution.
% The improvement on MD can also be benefited from aggregation as auxiliary features, which provide hints to the model on the data distribution. The improvement on CP is likely also a consequence of our fine-grained last level's learner and the type-specific data processing pipeline. 

Notably, existing baselines do not support variable-length inputs; even with na\"ive zero-padding (Appendix~\ref{app:baseline}), they fail to learn meaningful length information and suffer from distribution distortion due to excessive padding. The gains on MD benefit from aggregated auxiliary features that better characterize data distributions, while improvements on CP stem from our fine-grained last level's learner and type-specific processing pipeline.

% FracTS preserves the marginal distribution better than all baselines. This holds true across all datasets and processing functions, such as raw values, aggregations, and lengths. We see clear improvements in length distribution, raw data, standard deviations, and quantile values. This proves that FracTS generates very diverse data. This is likely because we include aggregation values in the conditions. FracTS also captures the correlations between variables much better. We see improvements ranging from 27\% to 89\% over the best baselines on raw value correlations. We believe this comes from our fine-grained decomposition. The model goes down to the level of each variable at each time step. It then explicitly learns their relations at the final fractal level. 

% Moreover, FracTS shows excellent utility on downstream PM tasks for most domains. For example, it reaches a 100\% improvement on the Stocks dataset. It also cuts the error by over 50\% on the Weather dataset. While one baseline performs slightly better on the Sensor PM task, FracTS clearly provides the best overall results. Finally, the t-SNE plot in Figure~\ref{fig:weather_uncond_single_1x6} visually confirms that our generated data is highly realistic.

{The superior predictive ML utility, across various window sizes show FracTS outstanding capability in learning long time series. While other models are still able to run, their performance degrades significantly compared to their performance on shorter time series. The exceptional performance of FracTS on Stocks dataset further highlights that our model particularly shine when applied on long time series with abundant training data.}

% \jiayu{\textit{Comment: A few very large negative values do stand out. Could anyone collect the data by showing the average performance (maybe w/o std) in main paper, and move current table 3,4 to appendix? }}
% \jiayu{\textit{Comment: may also underline the second best, so that almost all values from our model gets some annotations around.}}
% \jiayu{\textit{Comment: Almost rewritten. The original experiment analysis was too repetitive and verbose, with very few important information that is really expected as an analysis}}
% \jiayu{\textit{Comment: Need to fill the first sentence with actual (mean, not max) values.}}

\begin{table}[t]
\centering
\caption{Conditional generation performance averaged over the two datasets (lower is better; best in \textbf{bold}). MD, CP, and GT are averaged across all column pairs, and MAE is computed on standard-normalized values; CP additionally accounts for the static conditions (agg.+cond.). PM is the $|\text{TSTR}-\text{TRTR}|$ MAE gap at prediction window $W$; GT reports the error against the ground truth of the simulated dataset. The bottom row ($\Delta$best baseline) reports, for each column, the relative error of FracTS against the best baseline, computed on these dataset-averaged raw values. Per-dataset results with standard deviations over three runs are reported in Appendix~\ref{app:full-results}.}
\label{tab:walmart_cond_quality}
\small
\setlength{\tabcolsep}{5pt}
\resizebox{\linewidth}{!}{
\begin{tabular}{ll cccccc cc ccc cc}
\toprule
 & & \multicolumn{6}{c}{\cellcolor{headband}\textbf{MD} $\downarrow$} & \multicolumn{2}{c}{\cellcolor{headband}\textbf{CP} $\downarrow$} & \multicolumn{3}{c}{\cellcolor{headband}\textbf{PM} $\downarrow$} & \multicolumn{2}{c}{\cellcolor{headband}\textbf{GT} $\downarrow$} \\
\cmidrule(lr){3-8}\cmidrule(lr){9-10}\cmidrule(lr){11-13}\cmidrule(lr){14-15}
 & \textbf{Model} & raw & mean & quantile & skewness & autocorr. & length & raw & agg.+cond. & $W{=}10$ & $W{=}100$ & $W{=}1000$ & MSE & MAE \\
\midrule
 & TabularARGN & \textbf{0.154} & 0.731 & 0.710 & 0.748 & 0.994 & 0.278 & \textbf{0.069} & \textbf{0.345} & 0.056 & 0.084 & 0.050 & 1.759 & 0.952 \\
 & WaveStitch & 0.506 & 0.916 & 0.890 & 0.842 & 0.996 & 0.445 & 0.146 & 0.500 & 0.059 & 0.030 & 0.057 & 1.233 & 0.884 \\
 & TimeAutoDiff & 0.353 & 0.769 & 0.782 & 0.797 & 0.898 & 0.445 & 0.228 & 0.421 & 0.406 & 0.420 & 0.392 & 144.260 & 5.103 \\
\rowcolor{fractsrow}
 & \textbf{FracTS (Ours)} & 0.172 & \textbf{0.611} & \textbf{0.645} & \textbf{0.605} & \textbf{0.834} & \textbf{0.139} & 0.074 & 0.409 & \textbf{0.011} & \textbf{0.007} & \textbf{0.009} & \textbf{0.983} & \textbf{0.545} \\
\midrule
 & $\Delta$\textit{best baseline} & \dbad{$-$11.36\%} & \dgood{16.42\%} & \dgood{9.02\%} & \dgood{19.17\%} & \dgood{7.13\%} & \dgood{50.00\%} & \dbad{$-$7.30\%} & \dbad{$-$18.72\%} & \dgood{79.28\%} & \dgood{75.00\%} & \dgood{81.00\%} & \dgood{20.28\%} & \dgood{38.35\%} \\
\bottomrule
\end{tabular}
}
\end{table}

\subsection{Results on Conditional Generation}

FracTS also shows strong results for conditional generation, as shown in  Table~\ref{tab:walmart_cond_quality} and Figure~\ref{fig:simulated_walmart_cond_1x6}. It outperforms the best baseline by up to 50.0\%, 75.0--81.0\% and 20.3--38.4\% in MD, PM and GT, respectively (relative error computed on dataset-averaged raw values). For CP, FracTS and TabularARGN show competitive performance across different scores. Nevertheless,  both methods consistently outperform the remaining baselines. On the three columns with a negative $\Delta$ (MD raw, CP raw and CP agg.+cond.), the two methods win on different datasets and their dataset-averaged scores differ only by 0.005 to 0.064, while FracTS is best or tied-best on all temporal and predictive columns on both datasets; Appendix~\ref{app:full-results} reports the per-dataset scores and standard deviations. We additionally provide windowed t-SNE visualizations for the conditional
Walmart and Simulated datasets at window sizes
$W\in\{10,100,1000\}$ in Appendix~\ref{app:window_tsne}.

The most important metric for simulated dataset is GT. We show the visualized samples of the simulated dataset in Appendix~\ref{app:exp:vis}. FracTS clearly learns these mathematical formulae better. This gain manifests the advantage of our data processing in both static and temporal features, and our model's outstanding ability in learning long temporal and static dependencies. 

% First, FracTS is excellent at capturing ground truth behavior. On the Simulated dataset, it achieves the lowest MSE and MAE scores. This is a clear improvement over the baselines. We credit this success to our specific data processing steps and fine-grained learning approach. Second, FracTS generates data with a highly accurate distribution. On the Walmart dataset, it beats all baselines across every single marginal distribution (MD) metric. On the Simulated dataset, it still leads in mean, skewness, and auto-correlation. 

% The model also handles correlations well. FracTS preserves raw data correlations best on the Walmart dataset. On the Simulated dataset, it achieves the best score for aggregated and conditional correlations. While TabularARGN wins a few specific correlation metrics, FracTS remains highly competitive overall. Finally, the data generated by FracTS is very useful for downstream tasks. It achieves the best Predictive Margin (PM) scores across all time windows on both datasets. This proves that FracTS successfully captures the hidden, true behaviors of conditional time series data.

% \jiayu{\textit{Comment: also almost rewritten for the same reason.}}

\begin{figure}[!htbp]
  \centering
  \begin{minipage}{0.49\textwidth}
    \centering
    \includegraphics[width=\linewidth]{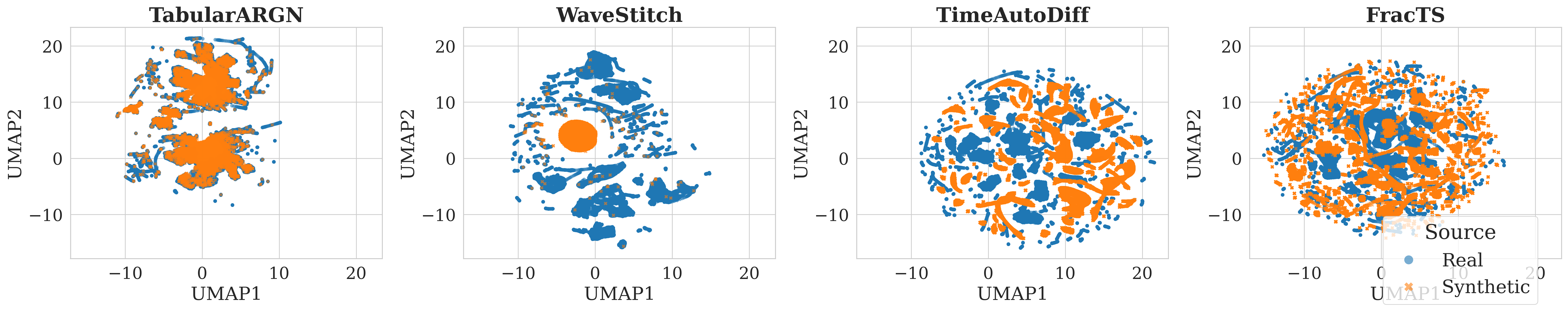}
    \caption*{\small\textbf{(a) Simulated}}
  \end{minipage}
  \hfill
  \begin{minipage}{0.49\textwidth}
    \centering
    \includegraphics[width=\linewidth]{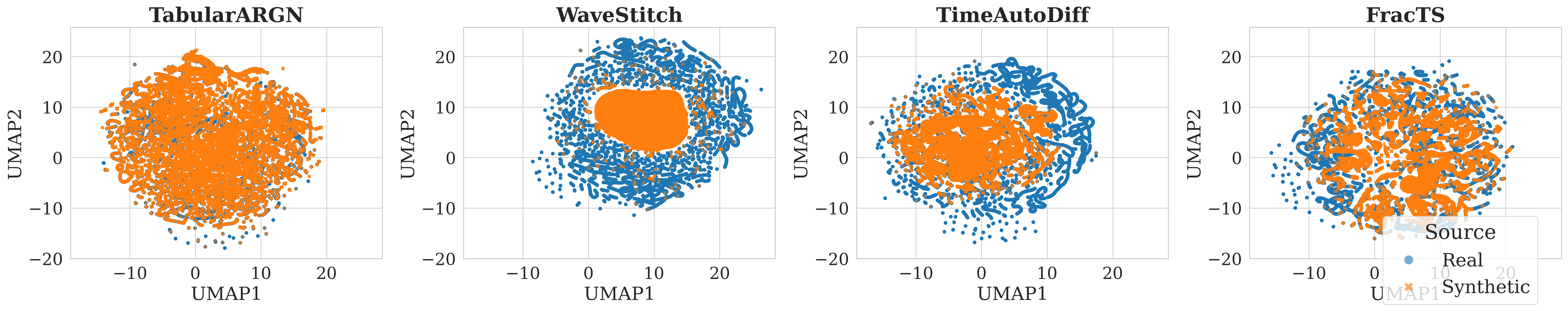}
    \caption*{\small\textbf{(b) Walmart}}
  \end{minipage}
  \caption{t-SNE of real compared to generated time series data on conditional generation datasets.}
  \label{fig:simulated_walmart_cond_1x6}
\end{figure}

\section{Conclusion}
% In this paper, we propose FracTS - a hierarchical and autoregressive time series generative model. FracTS leverages the framework of fractal generative model, which is originally designed for image, to applying for time series. FracTS introduces an effective time series preprocessing strategy which can encode data in a much more representative format for generative model training. FracTS is also engineered to support variable length time series, conditional time series generation and time series with categorical features. Experiments showed that FracTS outperforms seven competitive baselines on five datasets.
% In this paper, we introduced FracTS, a hierarchical and autoregressive generative model for time series. Building on the fractal generative framework originally designed for images, FracTS adapts it effectively to the time series domain. Our approach incorporates a novel preprocessing strategy that, considering the temporal features, encodes data into a more expressive representation. Furthermore, FracTS is designed to handle practical challenges in real-world time series: it supports variable-length sequences, conditional generation, and categorical features in the sequence. Extensive experiments across five datasets demonstrate that FracTS consistently outperforms seven baselines on: i) marginal distributions, ii) correlation preservation and iii) predictive ML utility.
\jiayu{In this paper, we introduced FracTS, a hierarchical and autoregressive generative model for time series. The model allows high-quality time series generation in a much less constrained setup than existing models in terms of the time series lengths and feature types and distributions. Building from a model originally designed for images, FracTS introduces significant adjustments for adaptation to time series domain, including a novel time series data processing pipeline optimized for the transformers backbone architecture, and correspondingly adjusted network structure. Extensive experiments across five datasets demonstrate that FracTS consistently outperforms seven baselines on distribution fidelity and downstream utility. We discuss the limitations of FracTS in Appendix~\ref{app:limitations}.}

% \jiayu{\textit{Comment: rewritten. the original version looks like it's incremental.}}

\clearpage
\bibliography{neurips_2026}
\bibliographystyle{unsrtnat}
%%%%%%%%%%%%%%%%%%%%%%%%%%%%%%%%%%%%%%%%%%%%%%%%%%%%%%%%%%%%

\clearpage
\appendix
\raggedbottom

% \section{Appendix / supplemental material}

\section{More Results of Step-wise Differences versus Raw Values}
\label{app:diff}

In Figure~\ref{fig:diff}, we show the Q-Q plots for the normal distribution of one feature from three datasets each, and for all three features, the distribution of differences is closer to a normal distribution than raw values. These are not individual cases, but a common phenomenon in many datasets. In this section, we show the Q-Q plots of more features from more datasets. Details of datasets will be introduced with experiments, or found in Appendix~\ref{app:dataset}.

\renewcommand{\qqplot}[1]{\raisebox{-2.2em}{\includegraphics[width=0.19\linewidth, trim={2.2em 2.5em 1em 2.3em}, clip]{figures/qq/#1.pdf}}}
\begin{table}[!htbp]
    \centering
    \setlength{\tabcolsep}{3pt}
    \resizebox{\linewidth}{!}{
    \begin{tabular}{
    >{\centering\arraybackslash}m{0.03\linewidth}
    ccccc}
    \toprule
         & \textbf{Stock} & \textbf{Weather} & \textbf{Sensor} & \textbf{Sensor} & \textbf{Sensor}\\
         & Volume & WindGustSpeed & Data\_Sent\_To\_BS & Dist\_To\_CH & Expaned Energy \\
    \midrule
     \vheader{Value} & \qqplot{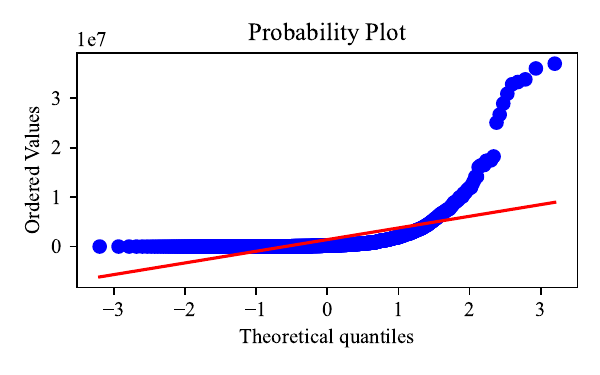} & \qqplot{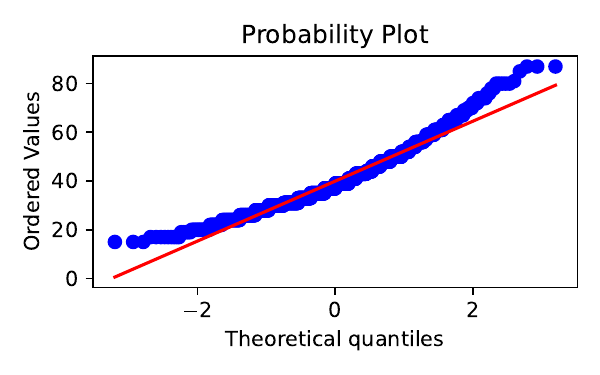} & \qqplot{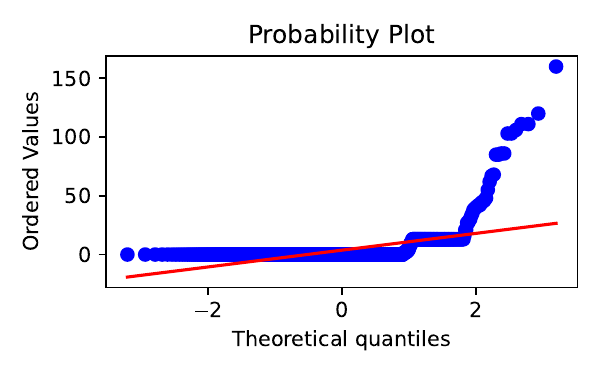} & \qqplot{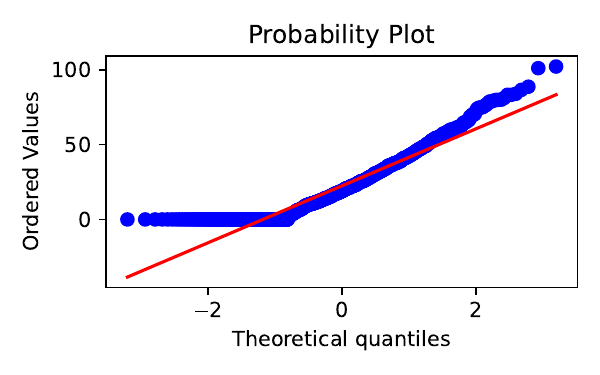} & \qqplot{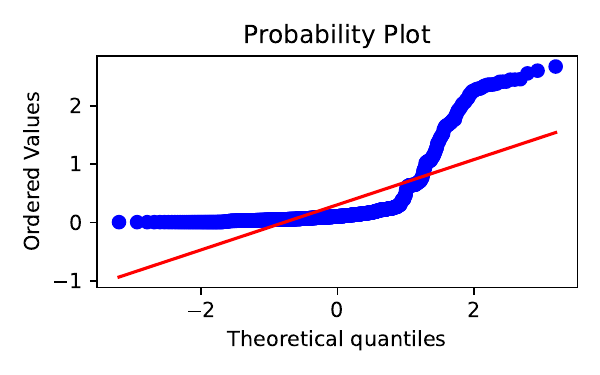}\\
     \vheader{Diff.} & \qqplot{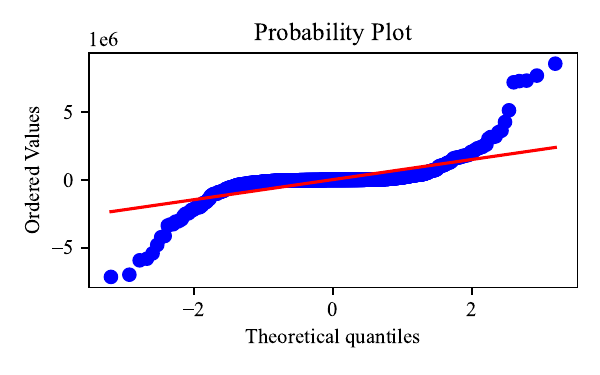} & \qqplot{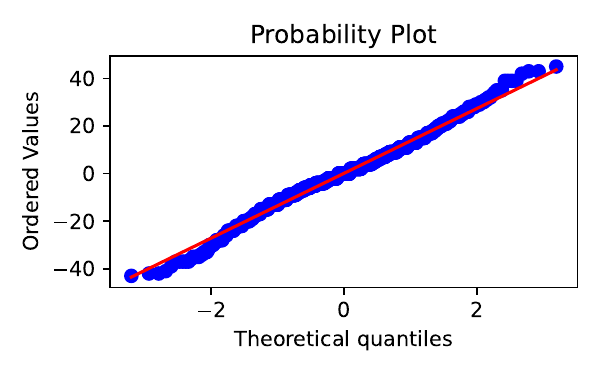}  & \qqplot{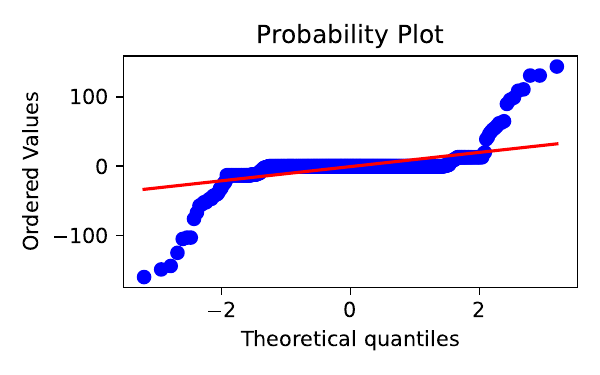} & \qqplot{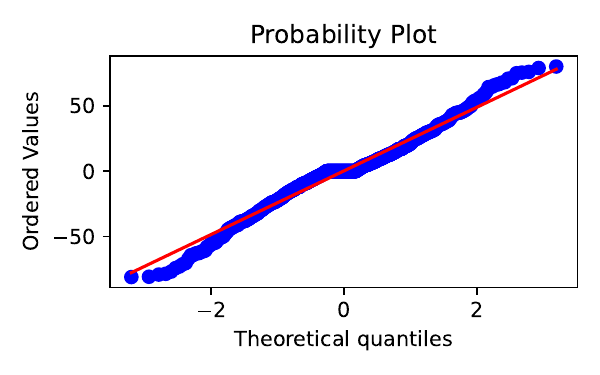} & \qqplot{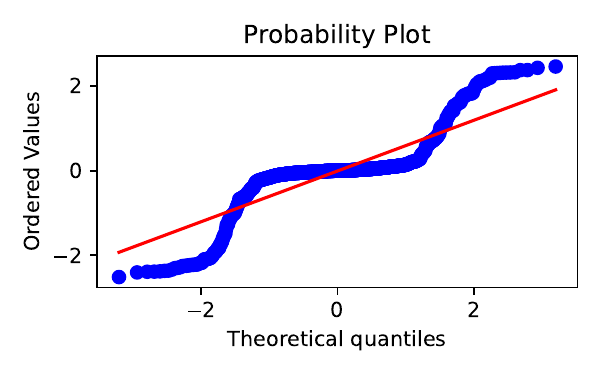}\\
    \bottomrule
    \end{tabular}
    }
    \inlinecaption{figure}{Comparison to normal distribution of the distribution of raw values and differences of consecutive timesteps in Q-Q plots of more features from real-life time series datasets.}{fig:more-diff}
\end{table}

In Figure~\ref{fig:more-diff}, we show more features from the realistic datasets we experimented with. Highly synchronized features are omitted, such as close and high for stocks, as open is already presented. It is clearly shown that the distributions of differences are closer to a normal distribution than the raw values. Indeed, there are cases where neither differences nor raw values are close to normal distribution, computing differences at least transforms the one-sided tailed distributions to two-sided tailed distributions, which is closer to the nature of normal distributions. The distribution of differences is at least not farther from a normal distribution than the distribution of raw values.

\begin{table}[!htbp]
    \centering
    \setlength{\tabcolsep}{3pt}
    \resizebox{\linewidth}{!}{
    \begin{tabular}{
    >{\centering\arraybackslash}m{0.03\linewidth}
    ccccc}
    \toprule
         & \textbf{col2} & \textbf{col4} & \textbf{col9} & \textbf{col12} & \textbf{col15}\\
         & trigonometric & step linear & periodic ramp & composite sin exp & Gabor-like \\
    \midrule
     \vheader{Value} & \qqplot{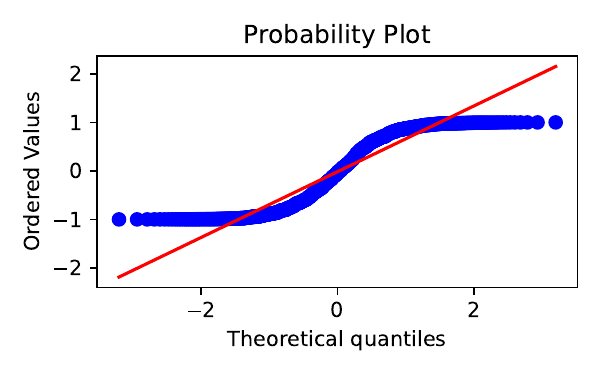} & \qqplot{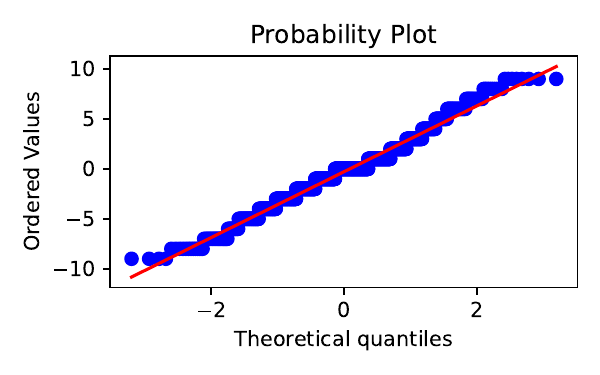} & \qqplot{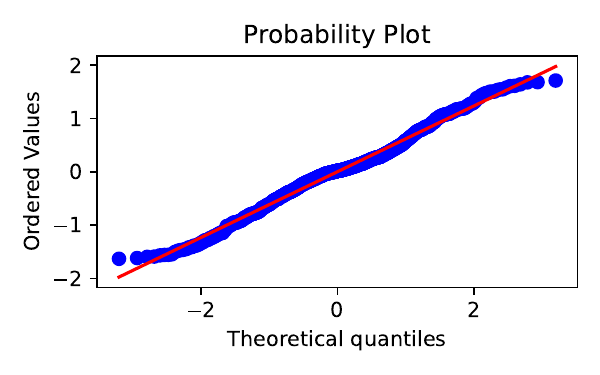} & \qqplot{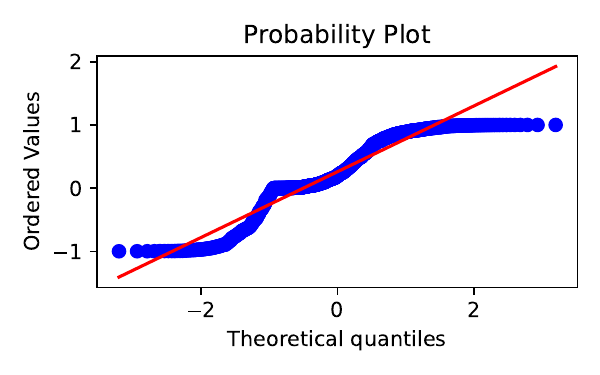} & \qqplot{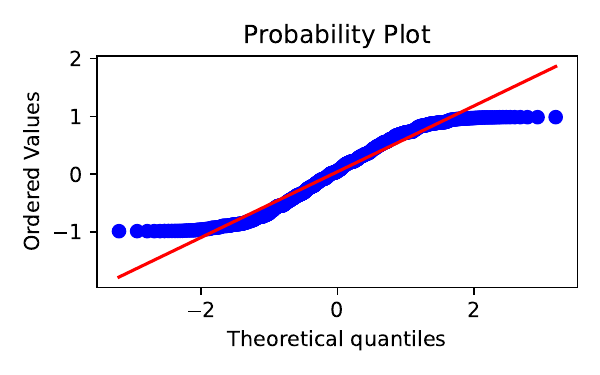}\\
     \vheader{Diff.} & \qqplot{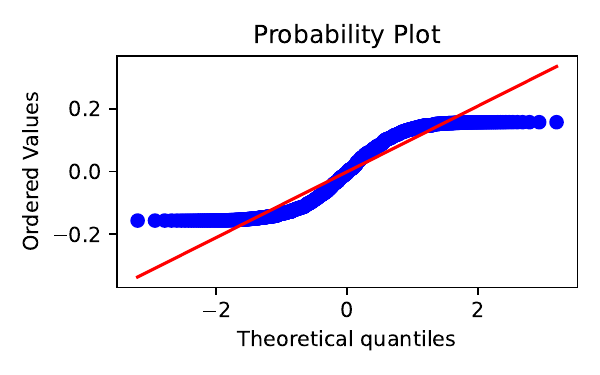} & \qqplot{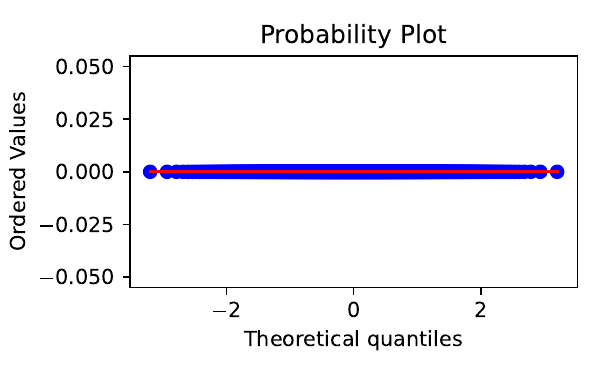} & \qqplot{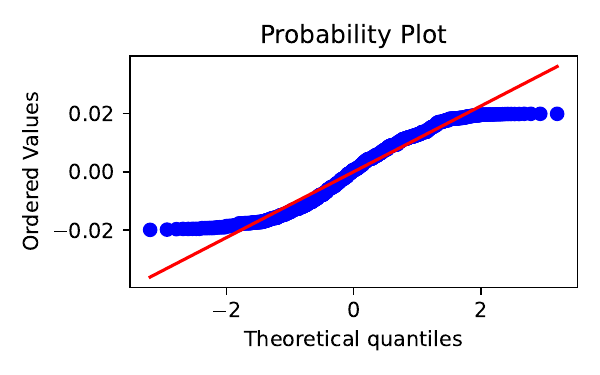} & \qqplot{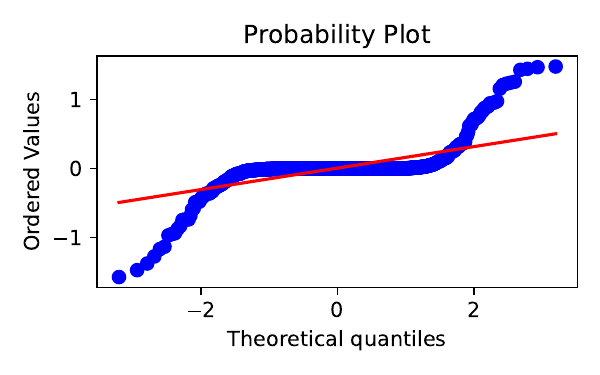} & \qqplot{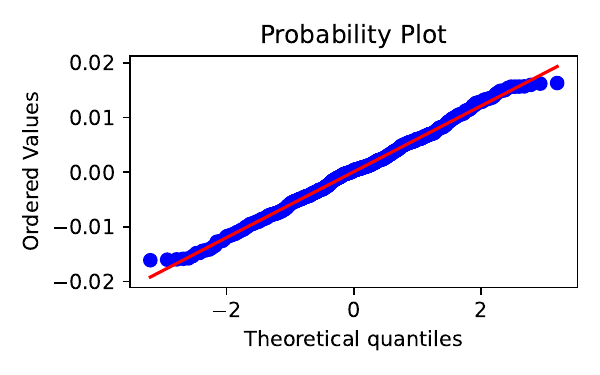}\\
    \bottomrule
    \end{tabular}
    }
    \inlinecaption{figure}{Comparison to normal distribution of the distribution of raw values and differences of consecutive timesteps in Q-Q plots of simulated features from mathematical formulae.}{fig:sim-diff}
\end{table}

Analyzing the distribution of differences versus values based on time series simulated from mathematical formulae, the conclusion above still holds in most cases, as shown in Figure~\ref{fig:sim-diff}, except for very few exceptional cases, such as periodic ramp signals. Nevertheless, periodic ramp signals' step-wise differences have two major clusters for intra- and inter-period changes. If one observes the two groups separately, the distribution of differences would still be more standard than raw values. Such clustering does not exist for raw values. Given the capability of neural networks in classification tasks, it is believed that a well-designed and well-trained neural network is capable of telling the two clusters above, and hence is likely to learn the combination of the two clusters of close-to-normal distributions, better than one cluster of non-normal distribution.

\section{Dataset Setup}
\label{app:dataset}
\subsection{Overall Setup}
\label{app:dataset:overall}
For all datasets, to ensure a fair and efficient comparison with baseline models, we use curated subsets of each dataset. We remove columns with the highest missingness and, for baselines that do not natively support categorical features, we encode those columns during preprocessing. This yields a consistent input space across datasets with varying domains, sequence lengths, and complexity while preserving realistic data characteristics.
All datasets are split at the series level into $8:2$ training and test set.

\subsection{Real-world Datasets Details}

\begin{itemize}
    \item \textbf{Weather Dataset:} The weather dataset~\citep{weatherdataset} contains the weather information, including temperature, humidity, wind speed, and pressure, of various places. The primary challenge lies in learning from very long sequences despite having a limited number of examples, making it essential for models to capture long-range temporal dependencies effectively.

    \item \textbf{Sensor Dataset:} The sensor dataset~\citep{sensordataset} includes 17 sensor readings related to either industrial equipment monitoring or cybersecurity scenarios. The main difficulty here is ensuring generalization across a large number of sequences with potentially unseen or highly diverse behaviors.

    \item \textbf{Stocks Dataset:} The stocks dataset~\citep{stocksdataset} features a variety of financial indicators such as stock prices, trading volume, and market volatility. The key challenge is dealing with the inherent noise, irregularity, and rapid changes that are characteristic of financial time series.

    \item \textbf{Walmart Dataset:} The Walmart dataset~\citep{walmartdataset} contains 2 static features (Type and Size) and 3 time-dependent variables capturing store-level sales statistics. As a conditional dataset, it requires models to leverage both static contextual information and dynamic temporal patterns. The primary challenge lies in effectively integrating static and sequential features over relatively short horizons while maintaining the ability to generalize across diverse departments and store profiles.

\end{itemize}

Due to the limited number of data points for Weather dataset, we use a different data split ($7:3$) for training and testing, to ensure sufficient test samples for reliable evaluation.

\subsection{Simulated Data Configuration}
\label{app:dataset:simulate}

The four static features and their corresponding sampling strategies are $w_1, w_2,b,k$, all uniformly sampled from $[-1, 1]$.

These condition values are the variables of the mathematical expressions.
Time series features are generated using the formulas listed in Table~\ref{tab:simulated-formulas}, where $x$ indicates the time step.

\begin{table}[!htbp]
\centering
\caption{Generating formulas of the 15 time series features of the simulated dataset. Each feature is a function of the time step $x$ and the static conditions $w_1, w_2, b, k$.}
\label{tab:simulated-formulas}
\resizebox{\linewidth}{!}{
\begin{tabular}{cc}
    \toprule
    \multicolumn{1}{l}{\textbf{col1:} Linear trend}  & \multicolumn{1}{l}{\textbf{col2:} Cosine wave with phase shift}  \\
    $0.001kx + b$& $\cos(0.05 \pi x + b)$\\
    \midrule
    \multicolumn{1}{l}{\textbf{col3:} Alternating constant pattern}  & \multicolumn{1}{l}{\textbf{col4:} Rounded step function}  \\
    $\left\lfloor \frac{x + \lfloor25b+0.5\rfloor}{100} \right\rfloor \bmod 2$ & $\lfloor-0.01 w_1 x+0.5\rfloor$\\
    \midrule
    \multicolumn{1}{l}{\textbf{col5:} Convex curve  } & \multicolumn{1}{l}{\textbf{col6:} Non-convex shape (tanh function)  }\\
    $\sqrt{0.01(w_1 + 1) x + 10 + b}$ & $\tanh\left(5k\left(\frac{x}{L_n} - 0.5\right)\right)$ \\
    \midrule
    \multicolumn{1}{l}{\textbf{col7:} Linear and sinusoidal combination  } & \multicolumn{1}{l}{\textbf{col8:} Cosine wave with exponential decay  }\\
    $0.001 w_1  x + w_2 + \sin(0.03  \pi x - 0.05)$ & $\exp(-0.001 x) \cos(0.003  x + b)$\\
    \midrule
    \multicolumn{1}{l}{\textbf{col9:} Periodic ramp signal  } & \multicolumn{1}{l}{\textbf{col10:} Piecewise exponential and linear function }\\
    $0.02 \cdot w_2 \cdot (x \bmod 100) $ & $\begin{cases}
        \exp(0.005 \cdot x), & \text{if } x < L_n/2 \\
        0.003 \cdot w_2 \cdot x, & \text{otherwise}
    \end{cases}$\\
    \midrule
    \multicolumn{1}{l}{ \textbf{col11:} Conditional cosine or sine  } & \multicolumn{1}{l}{\textbf{col12:} Composite function with exponential inside sine  } \\
    $\begin{cases}
        \cos(0.02 \pi  x + w_1), & \text{if } b \geq 0 \\
        \sin(0.02  \pi  x + w_2), & \text{if } b < 0
    \end{cases}$ & $ \sin(\exp(0.01 kx))$ \\
    \midrule
     \multicolumn{2}{l}{\textbf{col13:} Chirp signal~\citep{chirp} with combined frequency growth  }\\
    \multicolumn{2}{c}{$\sin\left(2\pi (0.003  k  x + (0.005 k  x)^2)\right)$}\\
    \midrule
    \multicolumn{2}{l}{\textbf{col14:} Amplitude modulated signal~\citep{am}  } \\
    \multicolumn{2}{c}{$\left(1 + 0.5 \sin\left(\frac{w_1\pi x}{500} \right)\right) \cdot \sin\left( \frac{\pi w_1x}{50}\right)$}\\
    \midrule

    \multicolumn{2}{l}{\textbf{col15:} Gabor-like modulated Gaussian signal~\citep{gabor}  }\\
    \multicolumn{2}{c}{ $\exp\left(-\frac{k^2x^2}{500000}\right) \cdot \sin\left(\frac{kx}{50}\right)$}\\
    \bottomrule
\end{tabular}
}
\end{table}

\section{FracTS Implementation Details}
\label{app:fracts}
\subsection{Details of Data Processing}
\paragraph{Data Requirements.}
In this paper, we focus on clean tabular data. Missing values, mixed numerical and categorical features, and data types other than categorical, numerical, and datetime, such as textual, audio, and vision data, are not included. 

\paragraph{Encoding of Numerical Data.}
All numerical data are standardized by their mean and standard deviation before being passed to the networks. When numerical data are binned, we use 20 bins by default. Both data processors are implemented using \texttt{scikit-learn} preprocessing module.

\label{app:data:datetime}
\paragraph{Datetime Components.} Datetime features are often omitted or simply treated as numeric values using Unix timestamps in most tabular models~\citep{ctgan,tabsyn,tabtreeformer}. However, this is not the optimal solution for time series data, as datetimes are typically critical to time series~\citep{banksformer} especially for event-based time series with event dates. Thus, prior to the numeric value representation of the datetime value, we put some important datetime components, such as month and day of week, to guide the generation of actual values, in addition to the bin of numeric values. 

Datetime components we use include: year (integer), month (integer), day in month (integer), hour (integer), minute (integer), second (integer), millisecond (integer), microsecond (integer), nanosecond (integer), month name (categorical), day of week (integer), day name of week (categorical), days to month end (integer), AM vs. PM (categorical), whether this is month start (categorical), whether this is month end (categorical), whether this is holiday (categorical). For components that do not have different values throughout the entire dataset in the current feature, the component is dropped. For example, a column up to the precision of day will have any component of higher precision, like hour, minute, and second, dropped.

While the exact datetime value can be directly inferred from the numeric representation, it is worthwhile to adapt the datetime value based on the components predicted, due to the limited precision the numeric predictions can achieve. For example, if a date is predicted to be a month's start but the date generated is Sep. 2, then we find the closest month start date Sep. 1, to replace the day of month from 2 to 1.

\label{app:data:agg}
\paragraph{Aggregation Functions.} We first extract the first value of all features in the aggregated auxiliary static features.
For categorical features, we further collect the top-frequency category and the proportion of at most 10 globally most frequent categories.
For numerical features, we collect the mean and standard deviation as additional aggregated information.
For datetime features, we only collect aggregated information on the day name of a week and the month name.

\subsection{Experimental Setup}
\label{app:exp:config}

All experiments were run on NVIDIA H100 GPUs using PyTorch with mixed precision training. Each result is reported as the average of three runs.

We use the same transformer model (GPT-2 small) for all datasets to keep the setup consistent and comparable. GPT-2 small has an embedding size of 768, 12 attention heads, and 12 transformer blocks. For the static features, we generate all samples using CTGAN~\citep{ctgan}. This ensures diverse and realistic static conditions while keeping the process simple and consistent across datasets.

For the simulated dataset specifically, we disabled all data aggregations due to GPU memory limits. We also skipped the condition compression module for this dataset. Since it only contains 4 conditions, the compression step was unnecessary.

Training is done with the AdamW optimizer using a learning rate of $2\times10^{-4}$. We apply a warmup of 10 epochs and train for a total of 200 epochs. Validation is performed after every epoch, and the best model is selected based on the lowest validation loss.

% side-by-side text + table: a small table must not waste a full-width band
\begin{minipage}[t]{0.55\linewidth}
\vspace*{0pt}
The convergence dynamics of our models are illustrated in Figure~\ref{fig:training_loss_curves} and Figure~\ref{fig:validation_loss_curves}, which displays the training and validation losses across all five datasets. To provide a comprehensive view of the computational requirements, we detail the total trainable parameters and the average training time for each dataset in Table~\ref{tab:efficiency}.
\end{minipage}\hfill
\begin{minipage}[t]{0.42\linewidth}
\vspace*{-12pt}
    \centering
    \inlinecaption{table}{Model size and average training time per dataset.}{tab:efficiency}
    \small
    \begin{tabular}{lcc}
        \toprule
        \rowcolor{headband}
        \textbf{Dataset} & \textbf{Params (M)} & \textbf{Time (h)} \\
        \midrule
        Weather   & 729.8  & 1.1 \\
        Sensor    & 432.8  & 4.6 \\
        Stocks    & 617.8  & 7.3 \\
        Walmart   & 1147.3 & 1.8 \\
        Simulated & 732.5  & 3.7 \\
        \bottomrule
    \end{tabular}
\end{minipage}

\paragraph{Training, inference, and memory cost.}
To further characterize computational efficiency, we compare the training
and inference time of FracTS with the most competitive baseline used for each of
the Weather and Walmart datasets, and measure the GPU memory consumption of
FracTS on an NVIDIA H100 GPU with 80\,GB of memory (Tables~\ref{tab:cost},~\ref{tab:runtime_comparison} and~\ref{tab:memory_usage}).
FracTS incurs a higher training cost than the compared baselines.
For inference, PaD-TS is faster on Weather, while FracTS is faster than
TabularARGN on the substantially longer Walmart sequences.

\begin{table}[!htbp]
\centering
\caption{Computational cost of FracTS on the two datasets used in the
efficiency analysis. Memory is measured on an NVIDIA H100 GPU with 80\,GB.
Training batch sizes are selected for data throughput, so the memory figures
are not a controlled comparison against baselines with different batch-size
configurations.}
\label{tab:cost}
\begin{subtable}{0.48\linewidth}
\centering
\caption{Training and inference time.}
\label{tab:runtime_comparison}
\small
\setlength{\tabcolsep}{5pt}
\begin{tabular}{llcc}
\toprule
\textbf{Dataset} & \textbf{Model} & \textbf{Train (s)} & \textbf{Inf.\ (s)} \\
\midrule
Weather & FracTS & 4068 & 240 \\
Weather & PaD-TS & 1560 & 10 \\
\midrule
Walmart & FracTS & 6300 & 241 \\
Walmart & TabularARGN & 240 & 300 \\
\bottomrule
\end{tabular}
\end{subtable}
\hfill
\begin{subtable}{0.46\linewidth}
\centering
\caption{GPU memory consumption.}
\label{tab:memory_usage}
\small
\setlength{\tabcolsep}{5pt}
\begin{tabular}{lccc}
\toprule
\textbf{Dataset} & \textbf{Len.\ $L$} & \textbf{Train (GB)} & \textbf{Inf.\ (GB)} \\
\midrule
Weather & 2171  & 73.0 & 3.5 \\
Walmart & 10474 & 77.0 & 5.5 \\
\bottomrule
\end{tabular}
\end{subtable}
\end{table}

\begin{figure}[!htbp] % pinned in place to preserve source order
    \centering
    % ==========================================
    % ROW 1: TRAINING LOSS
    % ==========================================
    \begin{subfigure}[b]{0.19\textwidth}
        \centering
        \includegraphics[width=\textwidth]{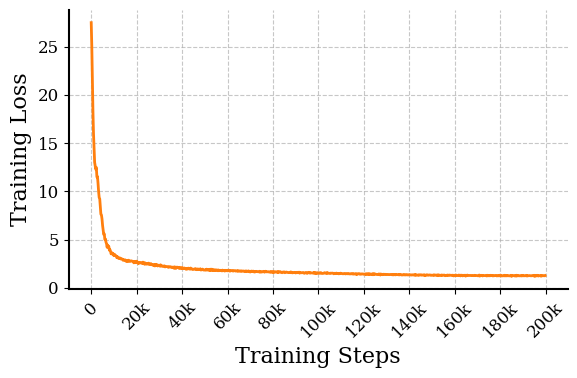}
        \caption{Weather}
    \end{subfigure}
    \hfill
    \begin{subfigure}[b]{0.19\textwidth}
        \centering
        \includegraphics[width=\textwidth]{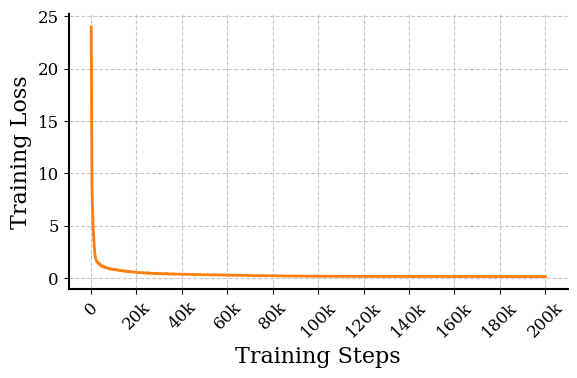}
        \caption{Sensors}
    \end{subfigure}
    \hfill
    \begin{subfigure}[b]{0.19\textwidth}
        \centering
        \includegraphics[width=\textwidth]{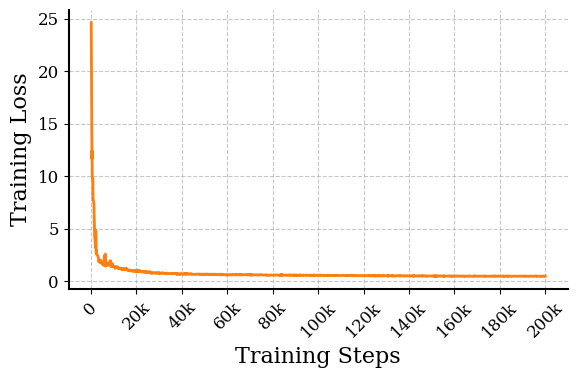}
        \caption{Stocks}
    \end{subfigure}
    \hfill
    \begin{subfigure}[b]{0.19\textwidth}
        \centering
        \includegraphics[width=\textwidth]{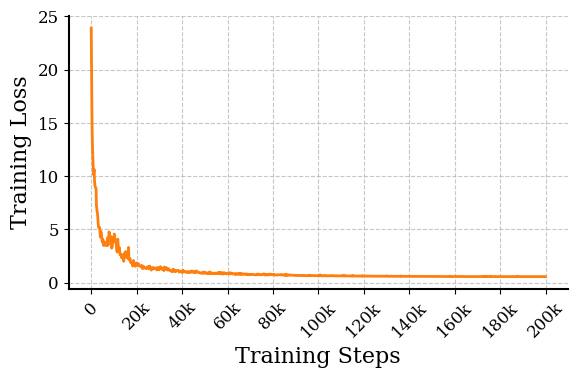}
        \caption{Walmart}
    \end{subfigure}
    \hfill
    \begin{subfigure}[b]{0.19\textwidth}
        \centering
        \includegraphics[width=\textwidth]{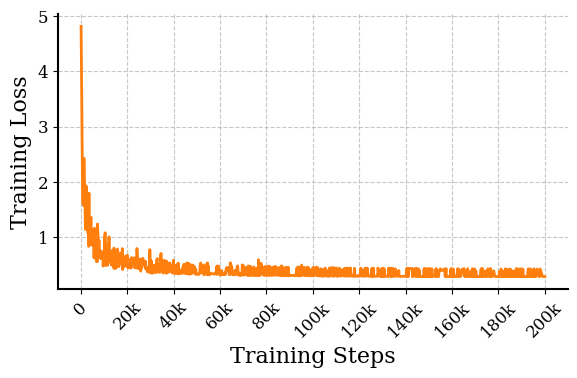}
        \caption{Simulated}
    \end{subfigure}
    \caption{Training loss curves across the five evaluated datasets. The x-axis represents the training steps, demonstrating that the models consistently converge over the course of the training regime.}
    \label{fig:training_loss_curves}
\end{figure}

\begin{figure}[!htbp]
    \centering
    % ==========================================
    % VALIDATION LOSS
    % ==========================================
    \begin{subfigure}[b]{0.19\textwidth}
        \centering
        \includegraphics[width=\textwidth]{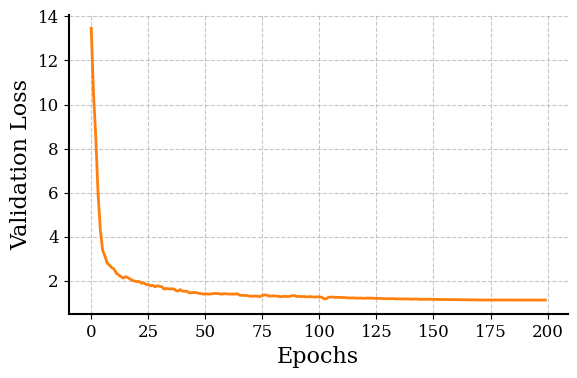}
        \caption{Weather}
    \end{subfigure}
    \hfill
    \begin{subfigure}[b]{0.19\textwidth}
        \centering
        \includegraphics[width=\textwidth]{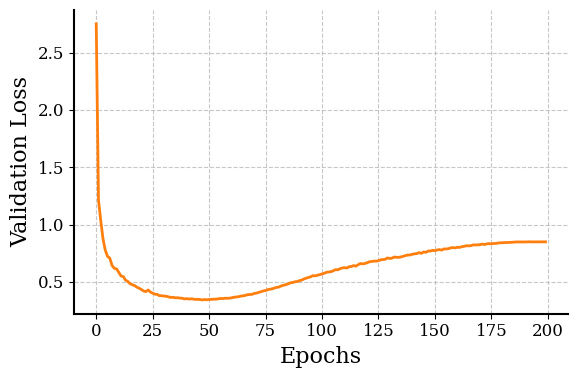}
        \caption{Sensors}
    \end{subfigure}
    \hfill
    \begin{subfigure}[b]{0.19\textwidth}
        \centering
        \includegraphics[width=\textwidth]{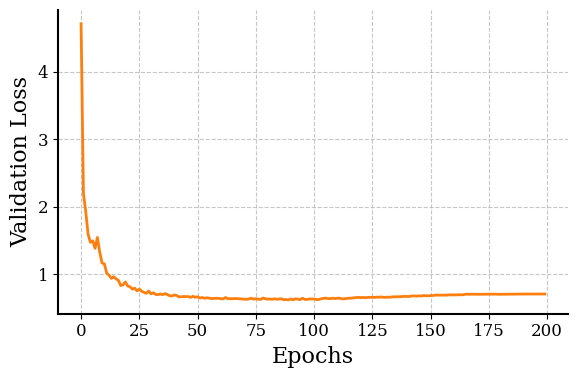}
        \caption{Stocks}
    \end{subfigure}
    \hfill
    \begin{subfigure}[b]{0.19\textwidth}
        \centering
        \includegraphics[width=\textwidth]{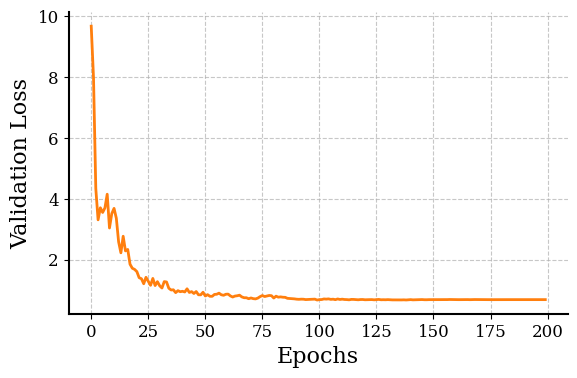}
        \caption{Walmart}
    \end{subfigure}
    \hfill
    \begin{subfigure}[b]{0.19\textwidth}
        \centering
        \includegraphics[width=\textwidth]{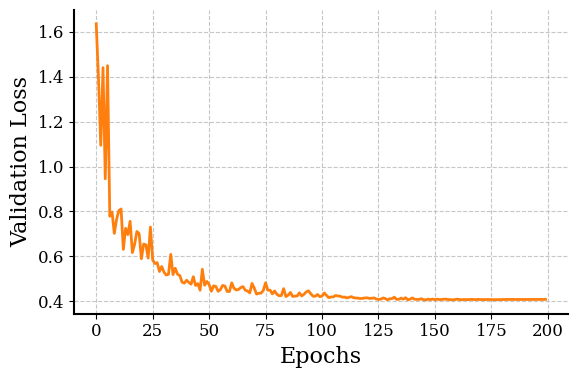}
        \caption{Simulated}
    \end{subfigure}
    \caption{Validation loss curves across the five evaluated datasets. The x-axis represents the training epochs, illustrating that the models consistently converge within the 200-epoch training regime.}
    \label{fig:validation_loss_curves}
\end{figure}

\section{Baselines}
\label{app:baseline}

Following gives the detailed introduction of each baseline.

\paragraph{Generation Length.}
FracTS generates the entire time series and supports variable and long generation lengths without truncation. For baselines that require fixed-length inputs, we pad sequences to the maximum sequence length.

By default, the baseline models construct training samples by forming intervals or overlapping segments from the given time series. However, since our dataset consists of distinct and independent time series, we generally treat each complete time series as a single window during training. To ensure consistent input lengths across samples, shorter series are zero-padded to match the maximum sequence length. After training, we applied a depadding step using a threshold of $1\times10^{-2}$ to remove the trailing padded values. The only exception to this approach is AECGAN; because its algorithm cannot function when an entire time series is treated as a single window, we split each time series into shorter windows specifically for this model. For all baselines, we used their official implementations with default hyperparameters and settings, modifying only the preprocessing pipeline to adapt to these configurations.

\begin{itemize}
    \item \textbf{Diffusion-TS}~\citep{diffusionts}: A diffusion-based generative model that uses an encoder-decoder Transformer for high-quality multivariate time series generation. It reconstructs the sample at each denoising step and applies a Fourier-based loss to enhance the quality and interpretability of generated data.
    \item \textbf{TimeVQVAE}~\citep{tvqvae}: This method uses a vector-quantized variational autoencoder for time series. TimeVQVAE learns discrete representations and models temporal dependencies with a Transformer prior, capturing both long-range and fine-grained temporal patterns.
    \item \textbf{COSCI-GAN}~\citep{coscigan}: A GAN-based model designed for multivariate time series from common sources, such as multi-channel sensor data. COSCI-GAN uses a central discriminator to maintain correlations between different channels and preserve realistic dynamics.
    \item \textbf{AEC-GAN}~\citep{aecgan}: An adversarial model that includes an error-correction mechanism to handle long-term dependencies and avoid drift in synthetic sequences. We defined the past ($p$) and future ($q$) sequence lengths based on the original experimental settings present in their official code. Specifically, we set $p=96$ and $q=192$ for the Stock dataset, $p=18$ and $q=36$ for Sensor, and $p=168$ and $q=336$ for Weather.
    \item \textbf{Koopman VAE (KoVAE)}~\citep{kovae}: A variational autoencoder that uses a linear Koopman operator in the latent space to model time series dynamics. KoVAE supports both regular and irregular time intervals and incorporates domain knowledge through its spectral constraints.
    \item\textbf{WaveStitch}~\citep{wavestitch}: WaveStitch is a conditional time series generation model based on denoising diffusion probabilistic models (DDPMs), which conditions on both static metadata and partially observed signals to flexibly adapt to unseen constraints. To incorporate these static conditioning attributes, we one-hot encode the categorical metadata, tile it across all timesteps, and concatenate it with the time series. Each complete series is treated as a single, fixed-length window via padding or truncation. Crucially, both the forward noising and reverse denoising diffusion processes are applied exclusively to the time series features, while the metadata columns remain clamped to serve as a constant conditioning context.
    \item \textbf{TabularARGN}~\citep{tabularargn}: A flexible auto-regressive generative framework for mixed-type, multivariate, and sequential data that factors the joint distribution into conditional probabilities and supports efficient conditional generation.
    \item \textbf{PaD-TS}~\citep{padts}: A diffusion-based model that focuses on capturing the patterns of the entire dataset instead of just individual samples. It uses a population-aware training method to ensure the generated data preserves the global statistics and correlations of the real data.
    \item \textbf{TimeAutoDiff}~\citep{timeautodiff}: A unified framework that combines a VAE and latent diffusion to handle time series with both continuous and categorical features. It compresses the data into a latent space first to effectively generate complex, mixed-type sequences. TimeAutoDiff expects each time series as a separate sample, with static conditioning attributes (e.g., simulation parameters or store metadata) broadcast across all timesteps and concatenated with the time-varying columns.
\end{itemize}

\section{Evaluation Metric Implementation}
\label{app:eval}
\paragraph{Setup:}
We evaluate quality by comparing many independent real time series (e.g., different stocks or sensors readings) to an equally sized set of synthetic series generated by our model. Each series is a CSV with one or more columns (variables). For every column we extract simple, model-agnostic features that summarize distributional and temporal behavior: \emph{mean}, \emph{std}, \emph{skew}, \emph{quantile}, \emph{length} and \emph{autocorrelation at lag 1}. Aggregating these features across all files gives us feature distributions for the real set and the synthetic set, which are then compared using the KS test.

\paragraph{Ground-truth (simulation):}
For simulated data with known generator $\mathcal{G}$, we compare reference series
$\mathbf{y}^{\star}=\mathcal{G}(\mathbf{c}_{\text{new}})$
to model output $\hat{\mathbf{y}} \sim p(\cdot \mid \mathbf{c}_{\text{new}})$ via:
\[
\mathrm{MAE}=\tfrac{1}{N}\!\sum_{t=1}^N \big|\hat{y}_t-y^{\star}_t\big|,
\qquad
\mathrm{MSE}=\tfrac{1}{N}\!\sum_{t=1}^N (\hat{y}_t-y^{\star}_t)^2.
\]
\paragraph{Utility Testing:} 
For the downstream performance evaluation, we used an LSTM network for time series prediction. To measure how useful the generated data is, we evaluated the models using two settings: Train on Synthetic, Test on Real (TSTR) and Train on Real, Test on Real (TRTR). The LSTM was trained to predict the next time step for all columns simultaneously. We used a sliding window approach with window sizes of 10, 100, and 1000 time steps to test both short- and long-term dependencies. The raw values were used as input features to the network. Finally, we calculated the Mean Absolute Error (MAE) for both settings and reported the absolute gap between their performances ($|\text{TSTR} - \text{TRTR}|$). A smaller gap shows that the synthetic data successfully preserves the predictive ML utility of the real data.

\paragraph{Discriminative Evaluation:}
Standard sequence-level discriminative evaluation assumes fixed-length
inputs and therefore cannot be directly applied to the variable-length
sequences generated in our setting. To enable this evaluation, we use
fixed-length windows of $W\in\{10,100,1000\}$ sampled from real and
synthetic sequences and train a discriminative classifier to distinguish
between them. Following the standard discriminative-score formulation used in our
evaluation, we report discriminative accuracy minus $0.5$, where lower
values are better and $0$ represents ideal indistinguishability.

\section{Supplementary Experiment Results}

\subsection{Windowed t-SNE Visualizations}
\label{app:window_tsne}

The t-SNE visualizations in the main paper operate at the individual
time-step level. To further assess whether the generated data preserve
temporal structure over different horizons, we additionally perform
windowed t-SNE using fixed-length windows of
$W\in\{10,100,1000\}$. Figures~\ref{fig:window_tsne_uncond}, \ref{fig:window_tsne_cond} and~\ref{fig:window_tsne_cond_sim}
show the results for the unconditional datasets and the conditional Walmart
and Simulated datasets, respectively.
Figure~\ref{fig:sensors_uncond_single_1x6} additionally shows the step-level
t-SNE comparison between real and generated data on the Sensors dataset.

% ---------------------------------------------------------
% Unconditional datasets (Weather, Sensors)
% ---------------------------------------------------------
\begin{figure}[!p]
    \centering
    \begin{subfigure}[b]{\linewidth}
        \centering
        \includegraphics[width=\linewidth]{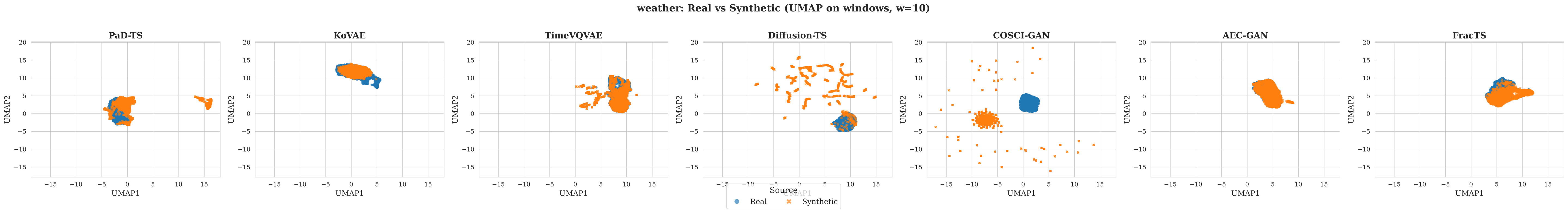}
        \caption{Weather, $W=10$}
    \end{subfigure}
    \begin{subfigure}[b]{\linewidth}
        \centering
        \includegraphics[width=\linewidth]{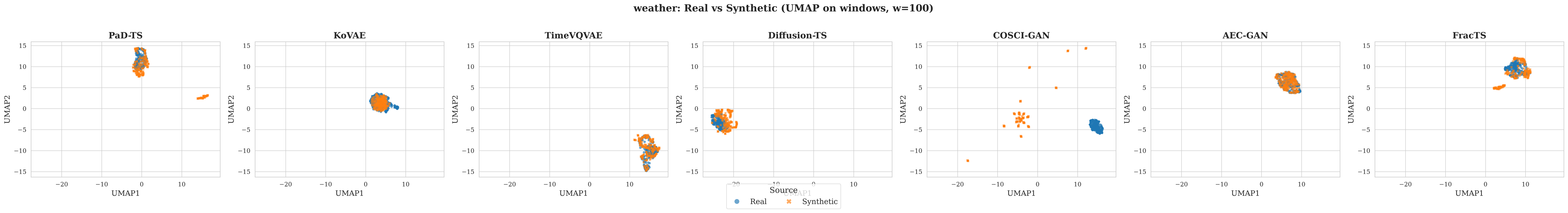}
        \caption{Weather, $W=100$}
    \end{subfigure}
    \begin{subfigure}[b]{\linewidth}
        \centering
        \includegraphics[width=\linewidth]{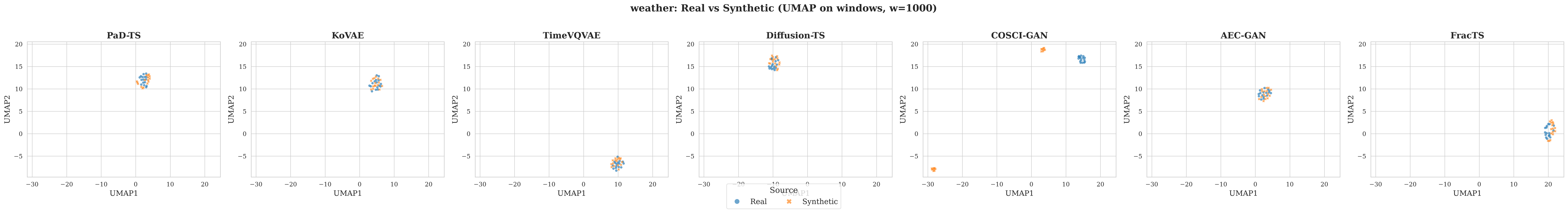}
        \caption{Weather, $W=1000$}
    \end{subfigure}
    \begin{subfigure}[b]{\linewidth}
        \centering
        \includegraphics[width=\linewidth]{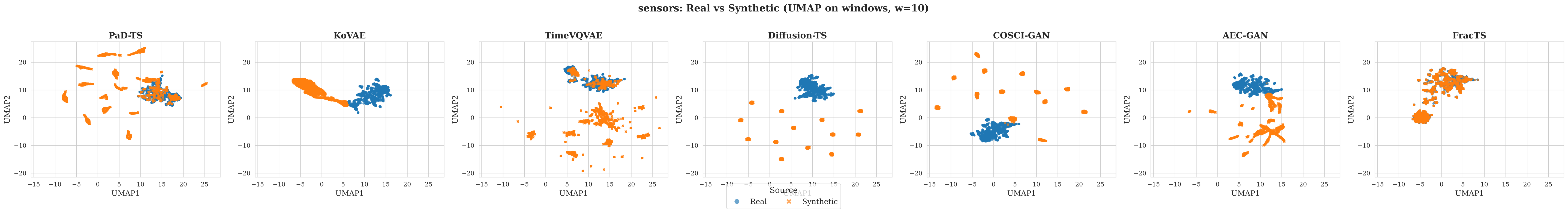}
        \caption{Sensors, $W=10$}
    \end{subfigure}
    \begin{subfigure}[b]{\linewidth}
        \centering
        \includegraphics[width=\linewidth]{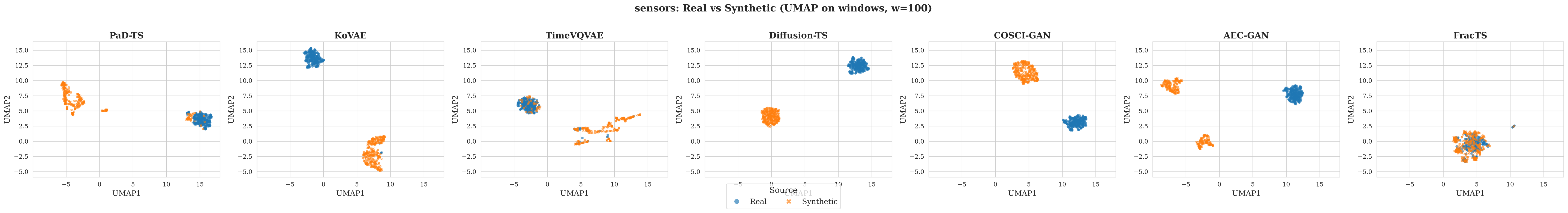}
        \caption{Sensors, $W=100$}
    \end{subfigure}
    \caption{Windowed t-SNE visualizations of real and generated samples for
    the unconditional datasets at different window sizes. ($W=1000$ is omitted
    for Sensors because every Sensors series is shorter than 1000 time steps,
    so no such window exists.)}
    \label{fig:window_tsne_uncond}
\end{figure}

% ---------------------------------------------------------
% Conditional datasets (Walmart, Simulated)
% ---------------------------------------------------------
\begin{figure}[!t]
    \centering
    \begin{subfigure}[b]{\linewidth}
        \centering
        \includegraphics[width=\linewidth]{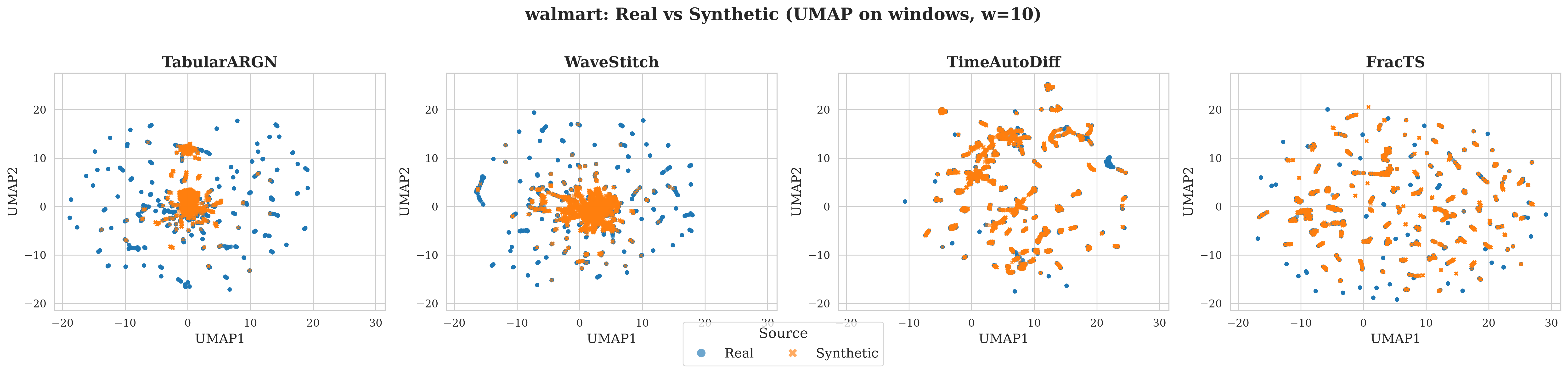}
        \caption{Walmart, $W=10$}
    \end{subfigure}
    \begin{subfigure}[b]{\linewidth}
        \centering
        \includegraphics[width=\linewidth]{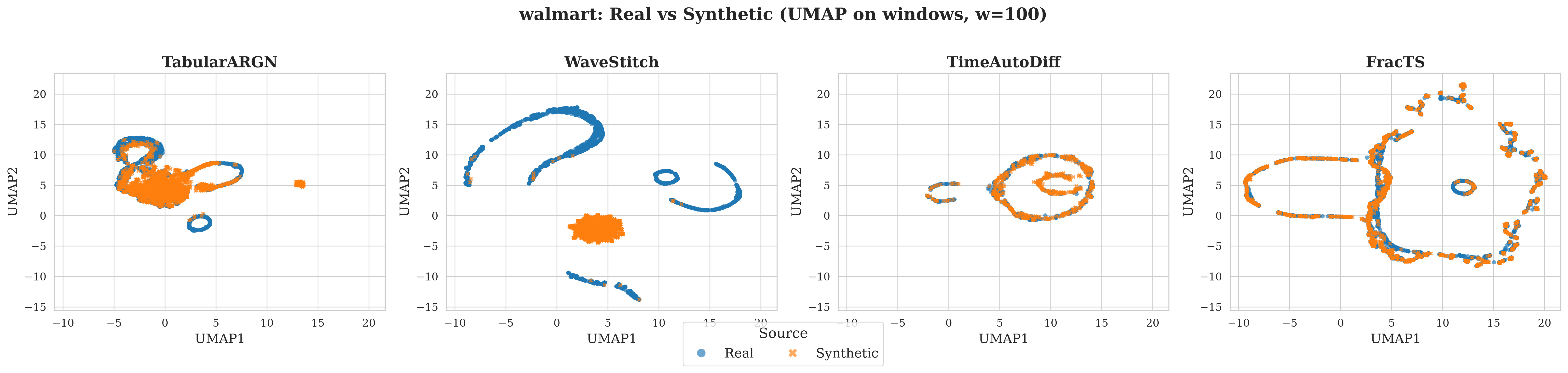}
        \caption{Walmart, $W=100$}
    \end{subfigure}
    \begin{subfigure}[b]{\linewidth}
        \centering
        \includegraphics[width=\linewidth]{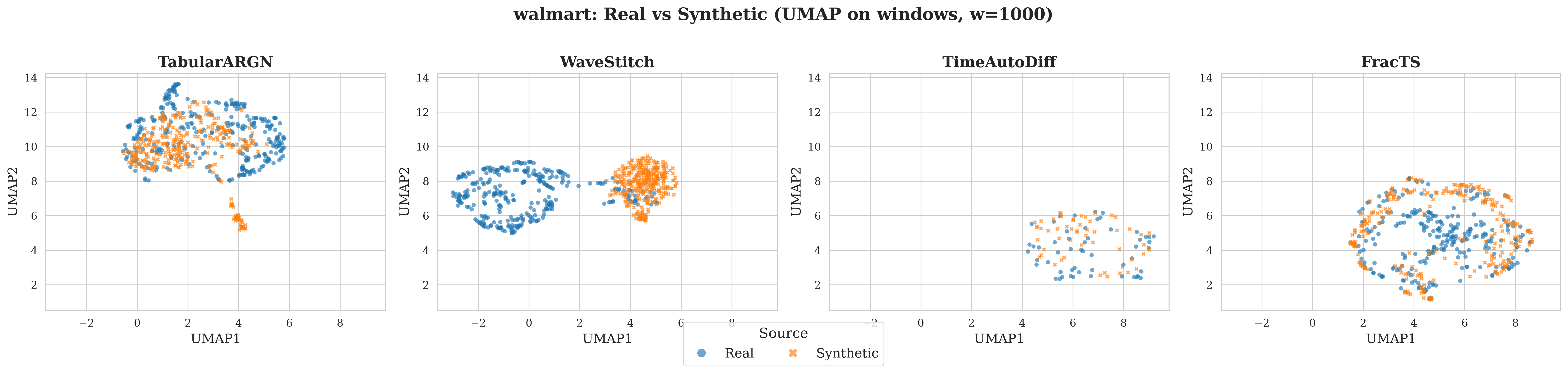}
        \caption{Walmart, $W=1000$}
    \end{subfigure}
    \caption{Windowed t-SNE visualizations of real and generated samples for the conditional Walmart dataset at different window sizes.}
    \label{fig:window_tsne_cond}
\end{figure}

\begin{figure}[!t]
    \centering
    \begin{subfigure}[b]{\linewidth}
        \centering
        \includegraphics[width=\linewidth]{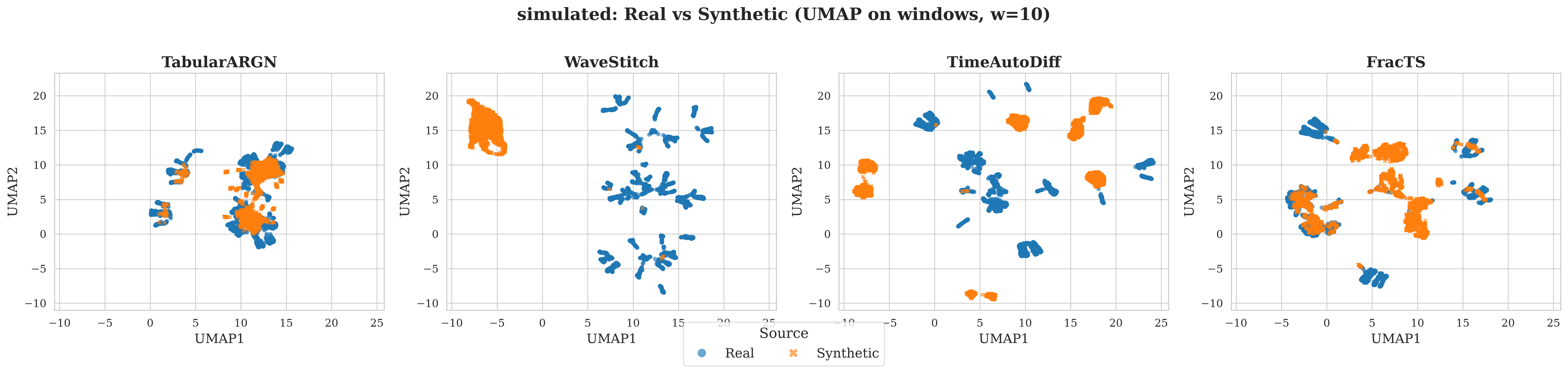}
        \caption{Simulated, $W=10$}
    \end{subfigure}
    \begin{subfigure}[b]{\linewidth}
        \centering
        \includegraphics[width=\linewidth]{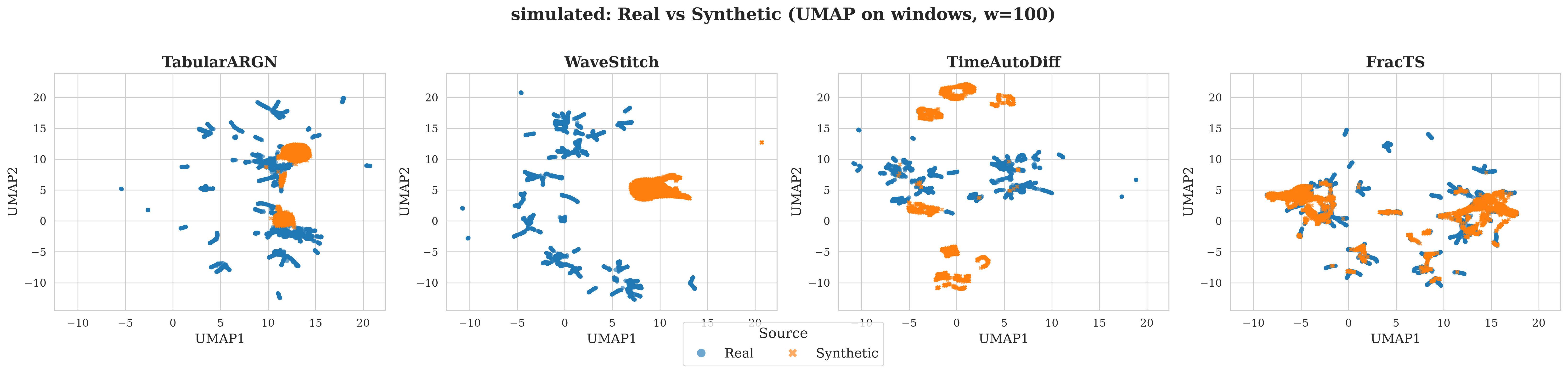}
        \caption{Simulated, $W=100$}
    \end{subfigure}
    \begin{subfigure}[b]{\linewidth}
        \centering
        \includegraphics[width=\linewidth]{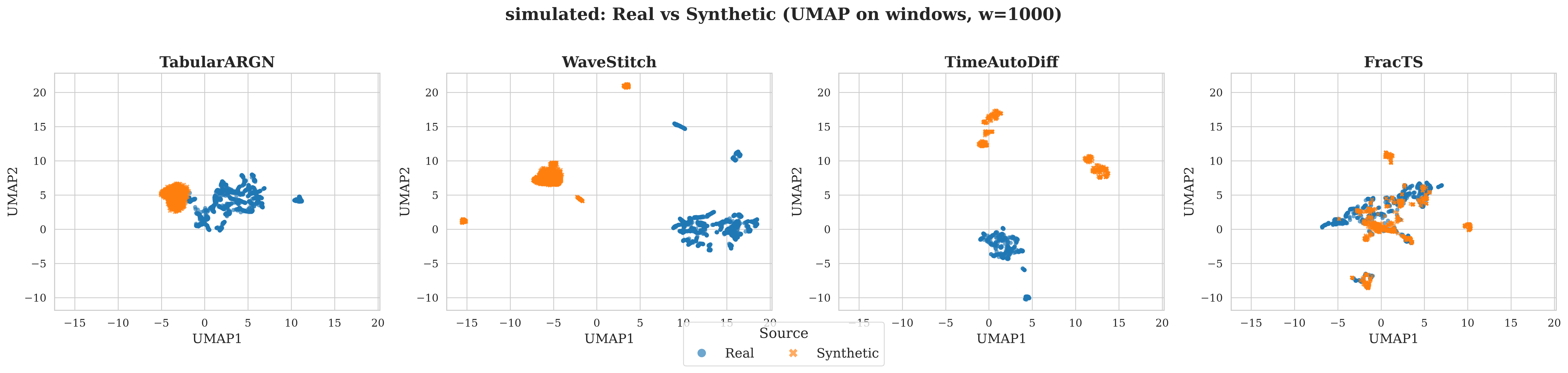}
        \caption{Simulated, $W=1000$}
    \end{subfigure}
    \caption{Windowed t-SNE visualizations of real and generated samples for the conditional Simulated dataset at different window sizes.}
    \label{fig:window_tsne_cond_sim}
\end{figure}

\begin{figure}[!p]
  \centering
  \includegraphics[width=\linewidth, keepaspectratio]{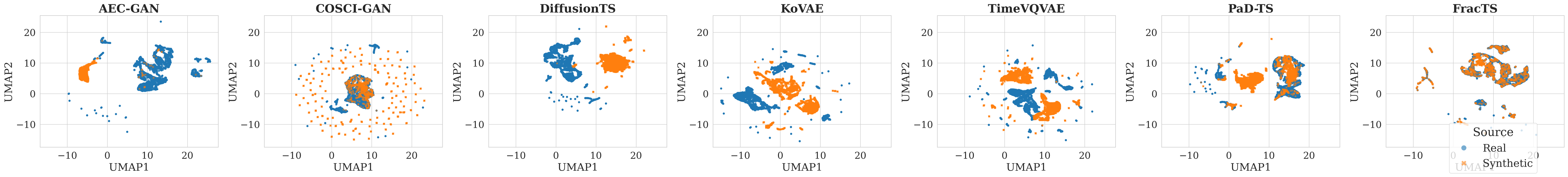}
  \caption{t-SNE of real compared to generated time series data of the Sensors dataset.}
  \label{fig:sensors_uncond_single_1x6}
\end{figure}

\subsection{More Visualized Results on Realistic Datasets}
\label{app:exp:vis}

Table~\ref{tab:simulated_percol_mae_mse} provides a per-column dissection of model performance on the simulated dataset. FracTS demonstrates clear overall superiority, outperforming the baselines in 11 out of the 15 columns. The gains are especially pronounced in capturing fundamental temporal structures, ranging from linear and standard periodic patterns (Cols 1–2) to highly nonlinear or piecewise non-stationary signals. For example, on the alternating constant pattern (Col~3) and the piecewise exponential/linear function (Col~10), FracTS achieves errors that are significantly lower than both WaveStitch and TabularARGN. However, we observe that for highly complex composite and oscillatory signals involving frequency growth or heavy amplitude modulation (Cols 12–15, such as the Chirp and Gabor-like signals), FracTS struggles to maintain this lead, and baselines like WaveStitch perform better. These highly modulated signals are defined by complex analytical formulas
that are difficult to recover exactly from data alone. In these cases,
WaveStitch produces highly fluctuating outputs that can obtain lower
pointwise MAE/MSE values. FracTS produces more realistic overall signal
shapes, but when it does not recover the exact underlying modulation,
the generated sequence can drift from the target over longer horizons,
resulting in larger MAE/MSE values. Despite this limitation on heavily modulated signals, FracTS provides the most accurate and robust generation overall. On average, FracTS achieves an MAE of 0.545 and an MSE of 0.983, substantially outperforming the closest baselines, WaveStitch (0.884 / 1.233) and TabularARGN (0.952 / 1.759). Figure~\ref{fig:selected_comparisons} shows generated versus ground-truth examples for six representative columns. 

\begin{figure}[!p]
  \centering
  \begin{subfigure}[b]{\linewidth}
    \centering
    \includegraphics[width=\linewidth]{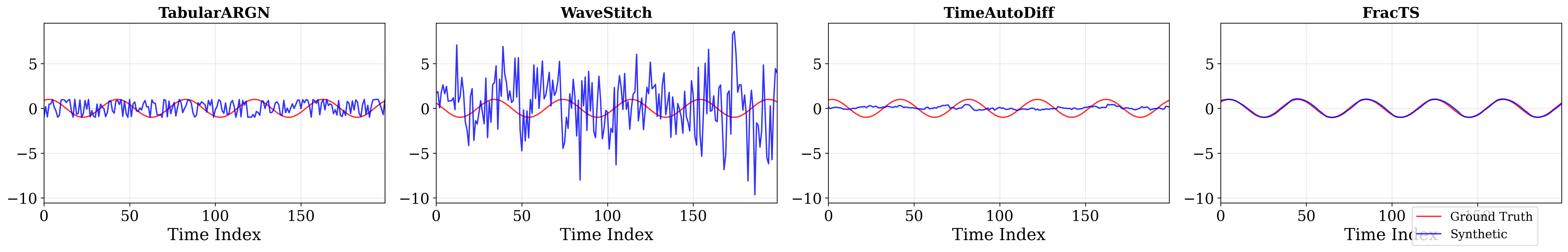}
    \caption{col2: trigonometric}
    \label{fig:sim_vis_col2}
  \end{subfigure}
  \begin{subfigure}[b]{\linewidth}
    \centering
    \includegraphics[width=\linewidth]{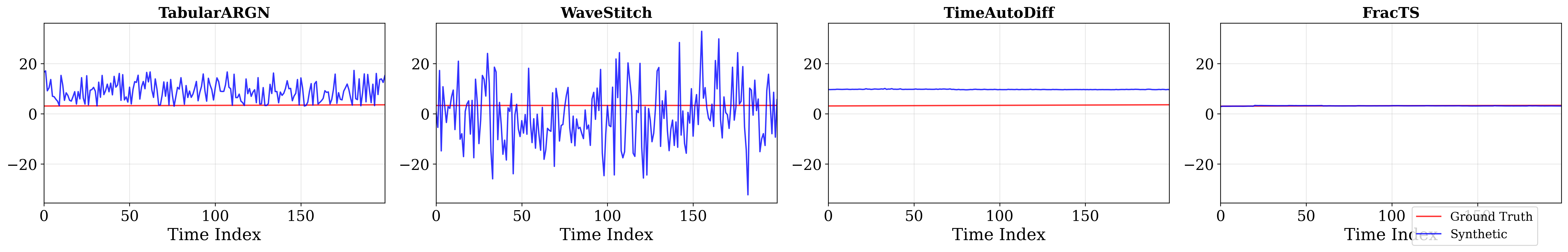}
    \caption{col5: linear}
    \label{fig:sim_vis_col5}
  \end{subfigure}
  \begin{subfigure}[b]{\linewidth}
    \centering
    \includegraphics[width=\linewidth]{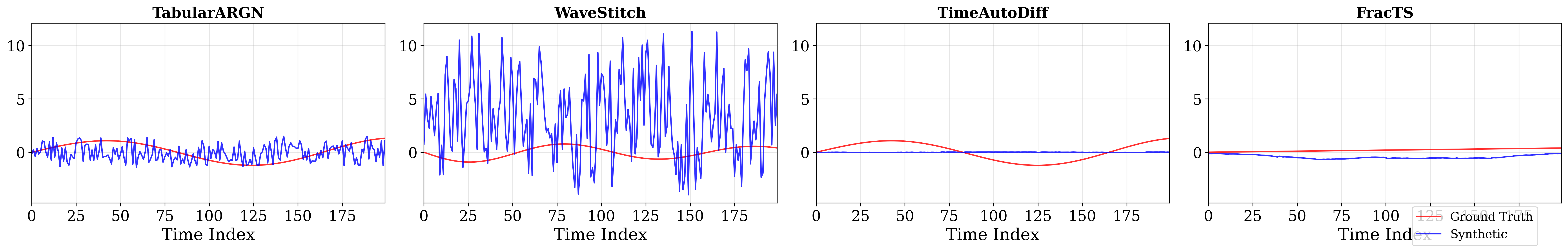}
    \caption{col14: amplitude modulated}
    \label{fig:sim_vis_col14}
  \end{subfigure}
  \begin{subfigure}[b]{\linewidth}
    \centering
    \includegraphics[width=\linewidth]{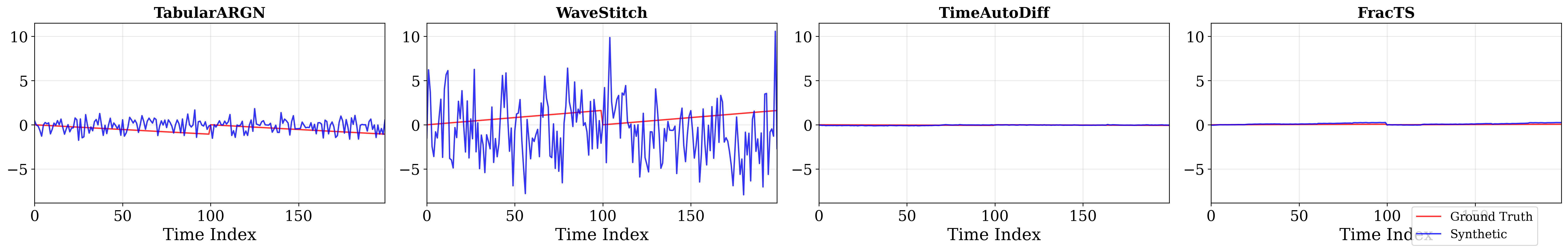}
    \caption{col9: periodic ramp}
    \label{fig:sim_vis_col9}
  \end{subfigure}
  \begin{subfigure}[b]{\linewidth}
    \centering
    \includegraphics[width=\linewidth]{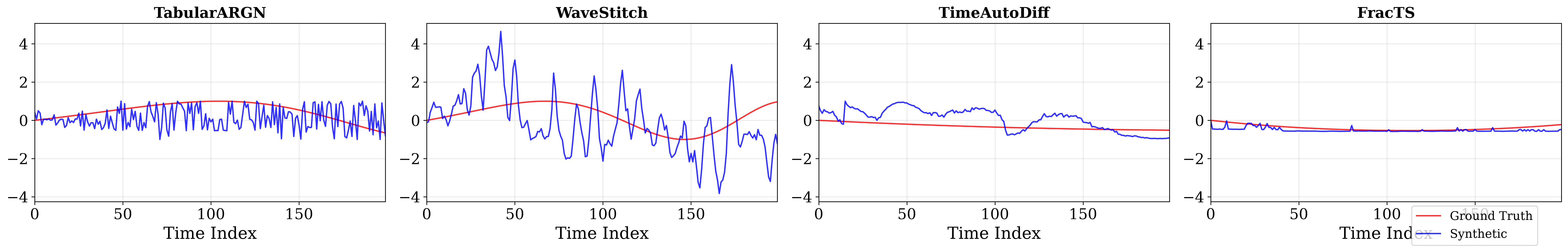}
    \caption{col13: chirp}
    \label{fig:sim_vis_col13}
  \end{subfigure}
  \begin{subfigure}[b]{\linewidth}
    \centering
    \includegraphics[width=\linewidth]{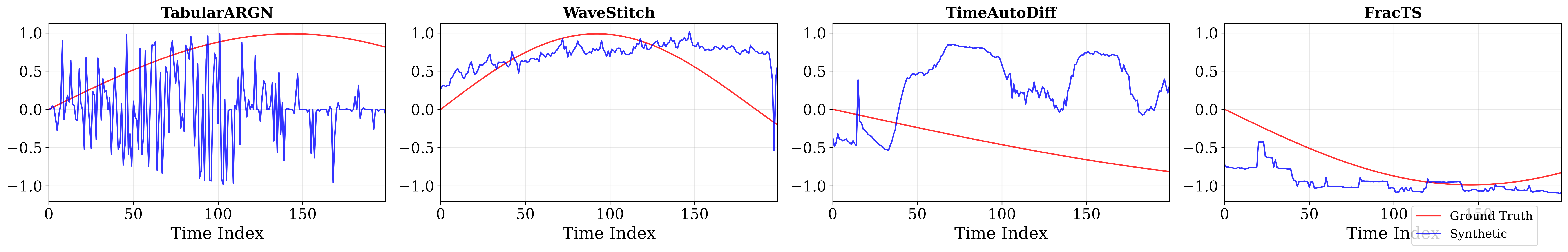}
    \caption{col15: Gabor-like}
    \label{fig:sim_vis_col15}
  \end{subfigure}
  \caption{Examples of generated versus ground-truth simulated data for six
  representative generating formulas (Table~\ref{tab:simulated-formulas}).
  Each panel overlays FracTS samples on the analytic ground truth.}
  \label{fig:selected_comparisons}
\end{figure}

\begin{table}[!htbp]
\centering
\renewcommand{\arraystretch}{0.9}
\caption{Per-column MAE/MSE against ground truth on the simulated dataset (lower is better; best in \textbf{bold}, second best \underline{underlined}). Column numbers correspond to the generating formulas in Table~\ref{tab:simulated-formulas}.}
\label{tab:simulated_percol_mae_mse}
\setlength{\tabcolsep}{4pt}
\resizebox{\linewidth}{!}{
\begin{tabular}{l cc cc cc cc}
\toprule
\rowcolor{headband}
 & \multicolumn{2}{c}{\textbf{FracTS (Ours)}} & \multicolumn{2}{c}{\textbf{TabularARGN}} & \multicolumn{2}{c}{\textbf{WaveStitch}} & \multicolumn{2}{c}{\textbf{TimeAutoDiff}} \\
\cmidrule(lr){2-3}\cmidrule(lr){4-5}\cmidrule(lr){6-7}\cmidrule(lr){8-9}
\textbf{Pattern} & MAE $\downarrow$ & MSE $\downarrow$ & MAE $\downarrow$ & MSE $\downarrow$ & MAE $\downarrow$ & MSE $\downarrow$ & MAE $\downarrow$ & MSE $\downarrow$ \\
\midrule
1. Linear trend            & $\mathbf{0.525}_{\pm 0.102}$ & $\mathbf{0.590}_{\pm 0.250}$ & $0.963_{\pm 0.015}$ & $1.547_{\pm 0.059}$ & $0.885_{\pm 0.083}$ & $1.232_{\pm 0.213}$ & $0.896_{\pm 0.055}$ & $1.466_{\pm 0.208}$ \\
2. Cosine wave             & $\mathbf{0.292}_{\pm 0.047}$ & $\mathbf{0.166}_{\pm 0.055}$ & $1.144_{\pm 0.001}$ & $1.993_{\pm 0.001}$ & $0.914_{\pm 0.097}$ & $1.312_{\pm 0.258}$ & $3.298_{\pm 0.526}$ & $14.337_{\pm 3.971}$ \\
3. Alternating constant    & $\mathbf{0.168}_{\pm 0.027}$ & $\mathbf{0.335}_{\pm 0.055}$ & $1.272_{\pm 0.125}$ & $2.547_{\pm 0.249}$ & $0.881_{\pm 0.055}$ & $1.217_{\pm 0.145}$ & $1.158_{\pm 0.102}$ & $2.696_{\pm 0.484}$ \\
4. Rounded step            & $\mathbf{0.467}_{\pm 0.057}$ & $\mathbf{0.428}_{\pm 0.104}$ & $0.838_{\pm 0.057}$ & $1.240_{\pm 0.157}$ & $0.967_{\pm 0.106}$ & $1.453_{\pm 0.283}$ & $1.444_{\pm 0.062}$ & $3.195_{\pm 0.232}$ \\
5. Convex curve            & $\mathbf{0.157}_{\pm 0.223}$ & $\mathbf{0.139}_{\pm 0.197}$ & $1.024_{\pm 0.008}$ & $1.609_{\pm 0.034}$ & $0.870_{\pm 0.057}$ & $1.187_{\pm 0.135}$ & $1.506_{\pm 0.029}$ & $3.224_{\pm 0.075}$ \\
6. Non-convex (tanh)       & $\mathbf{0.398}_{\pm 0.057}$ & $\mathbf{0.299}_{\pm 0.050}$ & $1.040_{\pm 0.019}$ & $1.625_{\pm 0.068}$ & $0.896_{\pm 0.086}$ & $1.257_{\pm 0.224}$ & $11.625_{\pm 6.562}$ & $251.719_{\pm 214.347}$ \\
7. Linear + sinusoid       & $\mathbf{0.478}_{\pm 0.054}$ & $\mathbf{0.431}_{\pm 0.100}$ & $0.928_{\pm 0.041}$ & $1.482_{\pm 0.124}$ & $0.957_{\pm 0.072}$ & $1.425_{\pm 0.188}$ & $1.540_{\pm 0.081}$ & $3.579_{\pm 0.355}$ \\
8. Cosine + exp.\ decay    & $\mathbf{0.120}_{\pm 0.027}$ & $\mathbf{0.130}_{\pm 0.040}$ & $0.553_{\pm 0.008}$ & $1.460_{\pm 0.008}$ & $0.817_{\pm 0.011}$ & $1.053_{\pm 0.029}$ & $7.478_{\pm 0.952}$ & $455.942_{\pm 125.824}$ \\
9. Periodic ramp           & $\mathbf{0.462}_{\pm 0.105}$ & $\mathbf{0.433}_{\pm 0.206}$ & $0.925_{\pm 0.005}$ & $1.438_{\pm 0.012}$ & $0.917_{\pm 0.085}$ & $1.315_{\pm 0.228}$ & $1.362_{\pm 0.091}$ & $3.528_{\pm 0.523}$ \\
10. Piecewise exp./linear  & $\mathbf{0.096}_{\pm 0.018}$ & $\mathbf{0.594}_{\pm 0.228}$ & $0.282_{\pm 0.001}$ & $1.580_{\pm 0.033}$ & $0.847_{\pm 0.046}$ & $1.152_{\pm 0.126}$ & $3.079_{\pm 0.686}$ & $154.752_{\pm 52.485}$ \\
11. Conditional cos/sin    & $\mathbf{0.497}_{\pm 0.088}$ & $\mathbf{0.505}_{\pm 0.172}$ & $1.098_{\pm 0.031}$ & $1.874_{\pm 0.080}$ & $0.969_{\pm 0.161}$ & $1.457_{\pm 0.439}$ & $1.139_{\pm 0.119}$ & $2.083_{\pm 0.431}$ \\
12. Composite sin--exp     & $0.873_{\pm 0.072}$ & $1.421_{\pm 0.109}$ & $1.053_{\pm 0.060}$ & $1.969_{\pm 0.160}$ & $\mathbf{0.821}_{\pm 0.014}$ & $\mathbf{1.060}_{\pm 0.035}$ & $13.553_{\pm 3.379}$ & $383.759_{\pm 168.309}$ \\
13. Chirp                  & $1.179_{\pm 0.059}$ & $2.132_{\pm 0.236}$ & $1.150_{\pm 0.005}$ & $2.003_{\pm 0.019}$ & $\mathbf{0.824}_{\pm 0.013}$ & $\mathbf{1.068}_{\pm 0.032}$ & $1.525_{\pm 0.361}$ & $3.517_{\pm 1.548}$ \\
14. Amplitude modulated    & $1.288_{\pm 0.139}$ & $2.454_{\pm 0.465}$ & $1.146_{\pm 0.001}$ & $1.980_{\pm 0.005}$ & $\mathbf{0.829}_{\pm 0.017}$ & $\mathbf{1.082}_{\pm 0.038}$ & $23.981_{\pm 5.798}$ & $844.979_{\pm 353.043}$ \\
15. Gabor-like             & $1.184_{\pm 0.124}$ & $4.688_{\pm 0.244}$ & $\mathbf{0.875}_{\pm 0.071}$ & $2.052_{\pm 0.237}$ & $0.883_{\pm 0.079}$ & $\mathbf{1.238}_{\pm 0.205}$ & $2.963_{\pm 0.437}$ & $35.127_{\pm 12.562}$ \\
\midrule
\textit{Average}           & $\mathbf{0.545}_{\pm 0.051}$ & $\mathbf{0.983}_{\pm 0.109}$ & $0.952_{\pm 0.034}$ & $1.759_{\pm 0.034}$ & $0.884_{\pm 0.040}$ & $1.233_{\pm 0.109}$ & $5.103_{\pm 2.088}$ & $144.260_{\pm 103.107}$ \\
\bottomrule
\end{tabular}
}
\end{table}

\subsection{Discriminative Evaluation}
\label{app:discriminative}

Table~\ref{tab:discriminative} reports the windowed discriminative
evaluation described in Appendix~\ref{app:eval}. Table~\ref{tab:disc_main} compares FracTS with
PaD-TS on Weather and TabularARGN on Walmart, and Table~\ref{tab:disc_dimension} repeats the comparison on feature subsets of Weather. Lower values indicate that
real and synthetic samples are more difficult to distinguish.

\begin{table}[!htbp]
\centering
\caption{Windowed discriminative evaluation, measured as
$\mathrm{Accuracy}-0.5$ (lower is better, with $0$ indicating ideal
indistinguishability). (a) compares FracTS against the strongest baseline per
dataset; (b) isolates the effect of the feature subset used for evaluation on
Weather while keeping the generator fixed.}
\label{tab:discriminative}
\begin{subtable}{0.9\linewidth}
\centering
\caption{FracTS versus the strongest baseline per dataset.}
\label{tab:disc_main}
\small
\setlength{\tabcolsep}{4pt}
\begin{tabular}{llccc}
\toprule
\textbf{Dataset} & \textbf{Model}
& $W{=}10$
& $W{=}100$
& $W{=}1000$ \\
\midrule
Weather
& \textbf{FracTS}
& $\mathbf{0.342}_{\pm 0.021}$
& $\mathbf{0.425}_{\pm 0.025}$
& $\mathbf{0.283}_{\pm 0.052}$ \\

Weather
& PaD-TS
& $0.418_{\pm 0.034}$
& $0.431_{\pm 0.036}$
& $0.423_{\pm 0.021}$ \\
\midrule
Walmart
& \textbf{FracTS}
& $\mathbf{0.064}_{\pm 0.056}$
& $\mathbf{0.118}_{\pm 0.113}$
& $\mathbf{0.170}_{\pm 0.125}$ \\

Walmart
& TabularARGN
& $0.266_{\pm 0.033}$
& $0.331_{\pm 0.153}$
& $0.310_{\pm 0.157}$ \\
\bottomrule
\end{tabular}
\end{subtable}

\begin{subtable}{0.9\linewidth}
\centering
\caption{Feature-subset study on Weather.}
\label{tab:disc_dimension}
\footnotesize
\setlength{\tabcolsep}{4pt}
\begin{tabular}{lccc}
\toprule
\textbf{Evaluation setting}
& $W{=}10$
& $W{=}100$
& $W{=}1000$ \\
\midrule
All features (17)
& $0.342$
& $0.425$
& $0.283$ \\

Remove 5 worst-performing features (12)
& $0.296$
& $0.364$
& $0.290$ \\

Retain 5 best-performing features (5)
& $0.130$
& $0.175$
& $0.215$ \\
\bottomrule
\end{tabular}
\end{subtable}
\end{table}

FracTS achieves lower discriminative scores than the corresponding
competitive baselines across all evaluated window sizes on both Weather
and Walmart.
Although FracTS achieves lower discriminative scores than the compared
baselines, the absolute scores on Weather indicate that generated and
real samples remain distinguishable under this evaluation. For example,
a score of $0.425$ at $W=100$ corresponds to a discriminator accuracy of
$0.925$. We therefore additionally examine how this behavior changes
across different subsets of the Weather features.

\paragraph{Effect of dimensionality on discriminative evaluation.}
To further investigate the relatively high discriminative score on the
17-feature Weather dataset, we evaluate the same generated output while
varying only the subset of features used for discriminative evaluation.
The generator remains trained on all 17 features.

Removing only five of the worst-performing features reduces the
discriminative score at $W=100$ from $0.425$ to $0.364$, while retaining
only the five best-captured features reduces it further to $0.175$.
This indicates that the discriminative score on the high-dimensional
Weather dataset is strongly influenced by a subset of challenging
features.
Theorem~\ref{thm:disc_dim} formalizes this dimensionality effect.

\begin{thm}
\label{thm:disc_dim}
Suppose a multivariate time series dataset exhibits $n$ mutually independent core properties (e.g., marginal distributions, feature correlations, autocorrelation structures), a discriminator $D$ flags a property as real with probability $p_d$ when the property is faithfully present and with probability $1-p_d$ otherwise, and a generator $G$ reproduces each property faithfully with probability $p_g$. Under independence across properties and a balanced real/synthetic test set, the expected discriminative accuracy is
\begin{equation}
\mathbb{E}[A(n)] = \tfrac{1}{2} + \tfrac{1}{2}\, p_d^{\,n} - \tfrac{1}{2}\left(p_g p_d + (1-p_g)(1-p_d)\right)^{n}.
\end{equation}
For any generator with imperfect fidelity ($p_g<1$) and any non-random discriminator ($p_d>1/2$), $\mathbb{E}[A(n)]-\tfrac{1}{2}$ is strictly positive for all $n$, and $\mathbb{E}[A(n)]$ converges to $\tfrac{1}{2}$ exponentially fast as $n$ grows.
\end{thm}

\section{Theoretical Analysis}
\label{app:theory}
\noindent\textbf{Theorem~\ref{thm:error}.} \emph{For a sequence of total length $L$ generated by a $K$-level hierarchical autoregressive architecture with branching factor $l$ (where $K=\log_l L$), if the local parameter estimation error per step is $\delta$, the total KL divergence between the ground-truth distribution $p(X_{\cdot,1:L})$ and the generated distribution $p(\hat{X}_{\cdot,1:L})$ satisfies:}
\begin{equation*}
    D_{KL}\left(p(X_{\cdot,L}\mid c)\|p(\hat{X}_{\cdot,L}\mid c)\right)=\Omega(\delta lK).
\end{equation*}

\begin{proof}[Proof of Theorem~\ref{thm:error}]
    According to \citet{aecgan} (Prop. 1), for a conventional autoregressive generation process $AR(1)$ of $l$ steps with maximum eigenvalue $\lambda\ge1$ and parameter discrepancy $\delta$, the divergence between the true local density and the estimated density grows at least nearly with sequence length:
    \begin{equation}
        D_{KL}\left(p(X_{\cdot,l}\mid X_{\cdot,0})\|p(\hat{X}_{\cdot,l}\mid X_{\cdot,0})\right)=\Omega(\delta l)
    \end{equation}
    Our introduction of static conditions can be considered altering $p$ based on the conditions, which does not affect the conclusion.

    In $i$-th level,  let $X^{[i]}_{\cdot,j}\sim X_{\cdot,(j-1)L^{[i+1]}+1:jL^{[i+1]}}$ be the $j$-th (global) step in this level, its parent is then $X^{[i-1]}_{\cdot,\left\lceil j/l\right\rceil}$.
    Between each hierarchical level, consider the information from previous level(s) to be passed and shared for all steps in the current level, then in $i$-th level
    \begin{align}
        D_{KL}^{[i]}&=D_{KL}\left(p(X_{\cdot,j}^{[i]}, X_{\cdot,\left\lceil j/l\right\rceil}^{[i-1]})\|p(\hat{X}_{\cdot,j}^{[i]}, \hat{X}_{\cdot,\left\lceil j/l\right\rceil}^{[i-1]})\right)\notag\\
        &=D_{KL}\left(p(X_{\cdot,\left\lceil j/l\right\rceil})\|p(\hat{X}_{\cdot,\left\lceil j/l\right\rceil})\right)+D_{KL}\left(p(X_{\cdot,j}^{[i]}\mid X_{\cdot,\left\lceil j/l\right\rceil}^{[i-1]})\|p(\hat{X}_{\cdot,j}^{[i]}\mid X_{\cdot,\left\lceil j/l\right\rceil}^{[i-1]})\right)\notag\\
        &=D_{KL}^{[i-1]}+\Omega(\delta (j\mod l))
    \end{align}
    Thus, by induction, we have
    \begin{equation}
        D_{KL}\left(p(X_{\cdot,L}\mid c)\|p(\hat{X}_{\cdot,L}\mid c)\right)=\Omega(\delta lK).
    \end{equation}
\end{proof}

\noindent\textbf{Theorem~\ref{thm:scale}.} \emph{Let $B$ be the batch size, $L$ the sequence length, $d$ the hidden dimension, and $l$ the local patch size per level. For a hierarchical model like FracTS with $K = \log_l L$ levels, the total activation memory complexity is $\mathcal{O}(BL(l+d))$.}

\begin{proof}[Proof of Theorem~\ref{thm:scale}]
    Recall the memory requirement for canonical Transformer~\citep{transformer} is $\mathcal{O}(BL^2+BLd)$.
    The batch size at each level is $B\left\lceil\frac{L}{L^{[i]}}\right\rceil$, sequence length is $\left\lceil\frac{L^{[i]}}{L^{[i+1]}}\right\rceil$
    Then the memory requirement for FracTS is
    \begin{align}
        &\sum_{i=0}^{K-1}\mathcal{O}\left(B\left\lceil\frac{L}{L^{[i]}}\right\rceil\cdot\left\lceil\frac{L^{[i]}}{L^{[i+1]}}\right\rceil^2+B\left\lceil\frac{L}{L^{[i]}}\right\rceil\cdot\left\lceil\frac{L^{[i]}}{L^{[i+1]}}\right\rceil d\right)\notag\\
        ={}&B\sum_{i=0}^{K-1}\mathcal{O}\left(\frac{Ll}{L^{[i+1]}}+\frac{Ld}{L^{[i+1]}}\right)\notag\\
        ={}&BL(l+d)\sum_{i=0}^{K-1}\mathcal{O}\left(\frac{1}{L/l^{i+1}}\right)\notag\\
        ={}&B(l+d)\sum_{i=0}^{K-1}\mathcal{O}\left(l^{i+1}\right)\notag\\
        ={}&Bl(l+d)\cdot\mathcal{O}\left(\frac{l^K-1}{l-1}\right)\notag\\
        ={}&Bl(l+d)\cdot\mathcal{O}\left(\frac{L-1}{l-1}\right)\notag\\
        ={}&\mathcal{O}(BL(l+d))
    \end{align}
\end{proof}

\section{Full Results with Standard Deviations}
\label{app:full-results}

Tables~\ref{tab:uncond_full} and~\ref{tab:cond_full} report the per-dataset results with the standard deviations over the three runs; Tables~\ref{tab:uncond_results} and~\ref{tab:walmart_cond_quality} in the main paper average these values over the datasets.

We now explain the negative $\Delta$ values in Table~\ref{tab:walmart_cond_quality}. For MD raw and CP raw, FracTS wins on Walmart (0.077 against 0.113, and 0.050 against 0.080) while TabularARGN wins on Simulated (0.195 against 0.266, and 0.057 against 0.097). The errors of FracTS on Simulated are about 3.5 and 1.9 times its errors on Walmart on these two columns, so Simulated weighs more in the dataset average, and the average turns negative even though FracTS wins on Walmart. For agg.+cond., FracTS wins on Simulated but only within the run-to-run standard deviation (0.315$\pm$0.007 against 0.319$\pm$0.008), so the two methods tie there; on Walmart, TabularARGN wins clearly (0.370 against 0.503), and this is the only cell where FracTS is not in the top two. The Walmart gap on MD raw (0.077$\pm$0.034 against 0.113$\pm$0.027) is also within the standard deviation, so the two methods tie on that cell as well. Each negative value therefore comes from a single dataset and column against one baseline, while FracTS is best or tied-best on all temporal and predictive columns on both datasets.

\begin{table}[!p]
\centering
\caption{Unconditional generation performance with standard deviations over three runs (lower is better; best in \textbf{bold}).}
\label{tab:uncond_full}
\footnotesize
\setlength{\tabcolsep}{2.5pt}
\resizebox{\linewidth}{!}{
\begin{tabular}{ll cccccc cc ccc}
\toprule
 & & \multicolumn{6}{c}{\cellcolor{headband}\textbf{MD}} & \multicolumn{2}{c}{\cellcolor{headband}\textbf{CP}} & \multicolumn{3}{c}{\cellcolor{headband}\textbf{PM}} \\
\cmidrule(lr){3-8}\cmidrule(lr){9-10}\cmidrule(lr){11-13}
 & \textbf{Model} & raw & mean & quantile & skewness & autocorr. & length & raw & agg. & $W{=}10$ & $W{=}100$ & $W{=}1000$ \\
\midrule
\rowcolors{1}{white}{zebra}
\multirow{7}{*}{\rotatebox{90}{Weather}}
 & AEC-GAN      & $0.503_{\pm 0.000}$ & $0.845_{\pm 0.000}$ & $0.887_{\pm 0.000}$ & $0.613_{\pm 0.000}$ & $0.697_{\pm 0.000}$ & \Na & $0.393_{\pm 0.000}$ & $0.465_{\pm 0.000}$ & $0.076_{\pm 0.000}$ & $0.077_{\pm 0.000}$ & $0.075_{\pm 0.000}$ \\
 & KoVAE        & $0.692_{\pm 0.238}$ & $0.861_{\pm 0.113}$ & $0.877_{\pm 0.099}$ & $0.781_{\pm 0.170}$ & $0.775_{\pm 0.082}$ & $1.000_{\pm 0.000}$ & $0.169_{\pm 0.010}$ & $0.438_{\pm 0.001}$ & $0.286_{\pm 0.291}$ & $0.284_{\pm 0.279}$ & $0.280_{\pm 0.275}$ \\
 & DiffusionTS  & $0.805_{\pm 0.002}$ & $1.000_{\pm 0.000}$ & $1.000_{\pm 0.000}$ & $0.714_{\pm 0.015}$ & $0.821_{\pm 0.045}$ & $1.000_{\pm 0.000}$ & $0.361_{\pm 0.113}$ & $0.523_{\pm 0.045}$ & $0.062_{\pm 0.010}$ & $0.082_{\pm 0.014}$ & $0.069_{\pm 0.020}$ \\
 & TimeVQVAE    & $0.466_{\pm 0.049}$ & $0.720_{\pm 0.069}$ & $0.705_{\pm 0.068}$ & $0.563_{\pm 0.022}$ & $0.840_{\pm 0.012}$ & $1.000_{\pm 0.000}$ & $0.156_{\pm 0.013}$ & $0.511_{\pm 0.028}$ & $0.047_{\pm 0.008}$ & $0.049_{\pm 0.007}$ & $0.045_{\pm 0.006}$ \\
 & COSCI-GAN    & $0.559_{\pm 0.001}$ & $0.879_{\pm 0.009}$ & $0.894_{\pm 0.021}$ & $1.000_{\pm 0.000}$ & $1.000_{\pm 0.000}$ & $1.000_{\pm 0.000}$ & $0.745_{\pm 0.001}$ & $0.677_{\pm 0.140}$ & $0.125_{\pm 0.004}$ & $0.125_{\pm 0.005}$ & $0.135_{\pm 0.005}$ \\
 & PaD-TS       & $0.388_{\pm 0.020}$ & $0.581_{\pm 0.012}$ & $0.679_{\pm 0.036}$ & $0.650_{\pm 0.048}$ & $0.788_{\pm 0.027}$ & $1.000_{\pm 0.000}$ & $0.354_{\pm 0.039}$ & $0.542_{\pm 0.054}$ & $0.023_{\pm 0.002}$ & $0.023_{\pm 0.002}$ & $0.021_{\pm 0.002}$ \\
\rowcolor{fractsrow}
 & \textbf{FracTS (Ours)} & $\mathbf{0.279}_{\pm 0.083}$ & $\mathbf{0.573}_{\pm 0.123}$ & $\mathbf{0.545}_{\pm 0.136}$ & $\mathbf{0.542}_{\pm 0.088}$ & $\mathbf{0.621}_{\pm 0.072}$ & $\mathbf{0.440}_{\pm 0.244}$ & $\mathbf{0.113}_{\pm 0.003}$ & $\mathbf{0.358}_{\pm 0.022}$ & $\mathbf{0.011}_{\pm 0.002}$ & $\mathbf{0.011}_{\pm 0.002}$ & $\mathbf{0.015}_{\pm 0.008}$ \\
\midrule
\rowcolors{1}{white}{zebra}
\multirow{7}{*}{\rotatebox{90}{Stocks}}
 & AEC-GAN      & $0.453_{\pm 0.001}$ & $0.537_{\pm 0.003}$ & $0.529_{\pm 0.002}$ & $0.726_{\pm 0.002}$ & $0.618_{\pm 0.002}$ & \Na & $0.143_{\pm 0.000}$ & $\mathbf{0.129}_{\pm 0.002}$ & $12\,157_{\pm 18}$ & $12\,513_{\pm 19}$ & $16\,466_{\pm 24}$ \\
 & KoVAE        & $0.553_{\pm 0.000}$ & $0.593_{\pm 0.000}$ & $0.489_{\pm 0.000}$ & $0.990_{\pm 0.000}$ & $0.863_{\pm 0.000}$ & $0.885_{\pm 0.000}$ & $0.377_{\pm 0.000}$ & $0.464_{\pm 0.000}$ & $22\,384_{\pm 0}$ & $23\,039_{\pm 0}$ & $30\,312_{\pm 0}$ \\
 & DiffusionTS  & $0.703_{\pm 0.063}$ & $0.912_{\pm 0.026}$ & $0.921_{\pm 0.019}$ & $0.587_{\pm 0.016}$ & $0.865_{\pm 0.040}$ & $0.885_{\pm 0.000}$ & $0.777_{\pm 0.081}$ & $0.579_{\pm 0.025}$ & $7\,573_{\pm 0}$ & $7\,795_{\pm 0}$ & $10\,257_{\pm 0}$ \\
 & TimeVQVAE    & $0.646_{\pm 0.001}$ & $0.656_{\pm 0.009}$ & $0.653_{\pm 0.015}$ & $0.424_{\pm 0.041}$ & $0.796_{\pm 0.040}$ & $0.885_{\pm 0.000}$ & $0.099_{\pm 0.010}$ & $0.375_{\pm 0.007}$ & $20\,467_{\pm 994}$ & $21\,067_{\pm 1\,023}$ & $27\,721_{\pm 1\,346}$ \\
 & COSCI-GAN    & $0.494_{\pm 0.082}$ & $0.641_{\pm 0.197}$ & $0.633_{\pm 0.205}$ & $0.761_{\pm 0.121}$ & $0.925_{\pm 0.035}$ & $0.885_{\pm 0.000}$ & $0.268_{\pm 0.086}$ & $0.311_{\pm 0.033}$ & $7\,175_{\pm 1\,598}$ & $7\,385_{\pm 1\,645}$ & $9\,716_{\pm 2\,164}$ \\
 & PaD-TS       & $0.292_{\pm 0.020}$ & $0.327_{\pm 0.005}$ & $0.328_{\pm 0.015}$ & $0.735_{\pm 0.017}$ & $0.785_{\pm 0.031}$ & $0.885_{\pm 0.000}$ & $0.289_{\pm 0.012}$ & $0.397_{\pm 0.024}$ & $11\,524_{\pm 615}$ & $11\,862_{\pm 633}$ & $15\,607_{\pm 833}$ \\
\rowcolor{fractsrow}
 & \textbf{FracTS (Ours)} & $\mathbf{0.225}_{\pm 0.017}$ & $\mathbf{0.232}_{\pm 0.026}$ & $\mathbf{0.230}_{\pm 0.034}$ & $\mathbf{0.139}_{\pm 0.028}$ & $\mathbf{0.335}_{\pm 0.039}$ & $\mathbf{0.675}_{\pm 0.026}$ & $\mathbf{0.023}_{\pm 0.004}$ & $0.148_{\pm 0.003}$ & $\mathbf{0.001}_{\pm 0.001}$ & $\mathbf{0.001}_{\pm 0.001}$ & $\mathbf{0.005}_{\pm 0.004}$ \\
\midrule
\rowcolors{1}{white}{zebra}
\multirow{7}{*}{\rotatebox{90}{Sensor}}
 & AEC-GAN      & $0.615_{\pm 0.000}$ & $0.667_{\pm 0.000}$ & $0.676_{\pm 0.001}$ & $0.511_{\pm 0.000}$ & $0.513_{\pm 0.001}$ & \Na & $0.328_{\pm 0.000}$ & $0.351_{\pm 0.000}$ & $0.006_{\pm 0.000}$ & $0.037_{\pm 0.000}$ & \Na \\
 & KoVAE        & $0.780_{\pm 0.112}$ & $0.671_{\pm 0.257}$ & $0.653_{\pm 0.257}$ & $0.581_{\pm 0.065}$ & $0.465_{\pm 0.165}$ & $0.984_{\pm 0.000}$ & $0.236_{\pm 0.047}$ & $0.566_{\pm 0.014}$ & $0.159_{\pm 0.213}$ & $0.247_{\pm 0.132}$ & \Na \\
 & DiffusionTS  & $0.577_{\pm 0.000}$ & $0.885_{\pm 0.000}$ & $0.823_{\pm 0.000}$ & $0.756_{\pm 0.000}$ & $0.639_{\pm 0.000}$ & $0.984_{\pm 0.000}$ & $0.260_{\pm 0.000}$ & $0.523_{\pm 0.000}$ & $0.162_{\pm 0.000}$ & $0.178_{\pm 0.000}$ & \Na \\
 & TimeVQVAE    & $0.790_{\pm 0.015}$ & $0.801_{\pm 0.016}$ & $0.742_{\pm 0.010}$ & $0.501_{\pm 0.031}$ & $0.559_{\pm 0.032}$ & $0.984_{\pm 0.000}$ & $0.293_{\pm 0.018}$ & $0.443_{\pm 0.004}$ & $0.045_{\pm 0.002}$ & $0.055_{\pm 0.004}$ & \Na \\
 & COSCI-GAN    & $0.619_{\pm 0.005}$ & $0.877_{\pm 0.014}$ & $0.718_{\pm 0.024}$ & $0.757_{\pm 0.014}$ & $0.649_{\pm 0.004}$ & $0.984_{\pm 0.000}$ & $0.372_{\pm 0.007}$ & $0.430_{\pm 0.001}$ & $0.055_{\pm 0.004}$ & $0.131_{\pm 0.008}$ & \Na \\
 & PaD-TS       & $0.556_{\pm 0.006}$ & $0.576_{\pm 0.004}$ & $0.532_{\pm 0.014}$ & $0.470_{\pm 0.010}$ & $0.569_{\pm 0.010}$ & $0.984_{\pm 0.000}$ & $0.271_{\pm 0.008}$ & $0.470_{\pm 0.008}$ & $\mathbf{0.004}_{\pm 0.001}$ & $\mathbf{0.004}_{\pm 0.001}$ & \Na \\
\rowcolor{fractsrow}
 & \textbf{FracTS (Ours)} & $\mathbf{0.437}_{\pm 0.007}$ & $\mathbf{0.228}_{\pm 0.019}$ & $\mathbf{0.304}_{\pm 0.021}$ & $\mathbf{0.130}_{\pm 0.017}$ & $\mathbf{0.113}_{\pm 0.007}$ & $\mathbf{0.212}_{\pm 0.048}$ & $\mathbf{0.026}_{\pm 0.003}$ & $\mathbf{0.161}_{\pm 0.003}$ & $0.007_{\pm 0.001}$ & $0.008_{\pm 0.000}$ & \Na \\
\bottomrule
\end{tabular}
}
\end{table}

\begin{table}[!t]
\centering
\caption{Conditional generation performance with standard deviations over three runs (lower is better; best in \textbf{bold}).}
\label{tab:cond_full}
\footnotesize
\setlength{\tabcolsep}{2.5pt}
\resizebox{\linewidth}{!}{
\begin{tabular}{ll cccccc cc ccc cc}
\toprule
 & & \multicolumn{6}{c}{\cellcolor{headband}\textbf{MD}} & \multicolumn{2}{c}{\cellcolor{headband}\textbf{CP}} & \multicolumn{3}{c}{\cellcolor{headband}\textbf{PM}} & \multicolumn{2}{c}{\cellcolor{headband}\textbf{GT}} \\
\cmidrule(lr){3-8}\cmidrule(lr){9-10}\cmidrule(lr){11-13}\cmidrule(lr){14-15}
 & \textbf{Model} & raw & mean & quantile & skewness & autocorr. & length & raw & agg.+cond. & $W{=}10$ & $W{=}100$ & $W{=}1000$ & MSE & MAE \\
\midrule
\rowcolors{1}{white}{zebra}
\multirow{4}{*}{\rotatebox{90}{Walmart}}
 & TabularARGN  & $0.113_{\pm 0.027}$ & $0.812_{\pm 0.116}$ & $0.775_{\pm 0.109}$ & $0.741_{\pm 0.098}$ & $1.000_{\pm 0.000}$ & $0.556_{\pm 0.000}$ & $0.080_{\pm 0.002}$ & $\mathbf{0.370}_{\pm 0.047}$ & $\mathbf{0.007}_{\pm 0.001}$ & $0.069_{\pm 0.001}$ & $\mathbf{0.005}_{\pm 0.001}$ & \Na & \Na \\
 & WaveStitch   & $0.653_{\pm 0.015}$ & $1.000_{\pm 0.000}$ & $0.944_{\pm 0.096}$ & $0.932_{\pm 0.069}$ & $1.000_{\pm 0.000}$ & $0.889_{\pm 0.000}$ & $0.163_{\pm 0.021}$ & $0.619_{\pm 0.045}$ & $0.081_{\pm 0.008}$ & $0.018_{\pm 0.008}$ & $0.076_{\pm 0.006}$ & \Na & \Na \\
 & TimeAutoDiff & $0.247_{\pm 0.182}$ & $0.778_{\pm 0.185}$ & $0.802_{\pm 0.113}$ & $0.667_{\pm 0.161}$ & $0.852_{\pm 0.148}$ & $0.889_{\pm 0.000}$ & $0.237_{\pm 0.005}$ & $0.435_{\pm 0.019}$ & $0.142_{\pm 0.083}$ & $0.181_{\pm 0.133}$ & $0.153_{\pm 0.096}$ & \Na & \Na \\
\rowcolor{fractsrow}
 & \textbf{FracTS (Ours)} & $\mathbf{0.077}_{\pm 0.034}$ & $\mathbf{0.583}_{\pm 0.167}$ & $\mathbf{0.633}_{\pm 0.130}$ & $\mathbf{0.549}_{\pm 0.043}$ & $\mathbf{0.781}_{\pm 0.088}$ & $\mathbf{0.278}_{\pm 0.073}$ & $\mathbf{0.050}_{\pm 0.006}$ & $0.503_{\pm 0.018}$ & $\mathbf{0.007}_{\pm 0.002}$ & $\mathbf{0.003}_{\pm 0.003}$ & $\mathbf{0.005}_{\pm 0.002}$ & \Na & \Na \\
\midrule
\rowcolors{1}{white}{zebra}
\multirow{4}{*}{\rotatebox{90}{Simulated}}
 & TabularARGN  & $\mathbf{0.195}_{\pm 0.004}$ & $0.650_{\pm 0.032}$ & $\mathbf{0.644}_{\pm 0.030}$ & $0.756_{\pm 0.021}$ & $0.988_{\pm 0.003}$ & $\mathbf{0.000}_{\pm 0.000}$ & $\mathbf{0.057}_{\pm 0.002}$ & $0.319_{\pm 0.008}$ & $0.104_{\pm 0.017}$ & $0.099_{\pm 0.016}$ & $0.095_{\pm 0.015}$ & $1.759_{\pm 0.080}$ & $0.952_{\pm 0.034}$ \\
 & WaveStitch   & $0.358_{\pm 0.016}$ & $0.832_{\pm 0.045}$ & $0.837_{\pm 0.041}$ & $0.752_{\pm 0.030}$ & $0.993_{\pm 0.000}$ & $\mathbf{0.000}_{\pm 0.000}$ & $0.129_{\pm 0.007}$ & $0.381_{\pm 0.009}$ & $0.037_{\pm 0.002}$ & $0.042_{\pm 0.002}$ & $0.038_{\pm 0.004}$ & $1.233_{\pm 0.109}$ & $0.884_{\pm 0.040}$ \\
 & TimeAutoDiff & $0.458_{\pm 0.030}$ & $0.760_{\pm 0.044}$ & $0.762_{\pm 0.037}$ & $0.927_{\pm 0.035}$ & $0.944_{\pm 0.004}$ & $\mathbf{0.000}_{\pm 0.000}$ & $0.219_{\pm 0.014}$ & $0.407_{\pm 0.016}$ & $0.669_{\pm 0.092}$ & $0.659_{\pm 0.088}$ & $0.630_{\pm 0.081}$ & $144.260_{\pm 103.107}$ & $5.103_{\pm 2.088}$ \\
\rowcolor{fractsrow}
 & \textbf{FracTS (Ours)} & $0.266_{\pm 0.048}$ & $\mathbf{0.639}_{\pm 0.016}$ & $0.658_{\pm 0.029}$ & $\mathbf{0.661}_{\pm 0.017}$ & $\mathbf{0.887}_{\pm 0.027}$ & $\mathbf{0.000}_{\pm 0.000}$ & $0.097_{\pm 0.006}$ & $\mathbf{0.315}_{\pm 0.007}$ & $\mathbf{0.016}_{\pm 0.015}$ & $\mathbf{0.012}_{\pm 0.011}$ & $\mathbf{0.014}_{\pm 0.011}$ & $\mathbf{0.983}_{\pm 0.109}$ & $\mathbf{0.545}_{\pm 0.051}$ \\
\bottomrule
\end{tabular}
}
\end{table}

\begin{table}[!t]
\centering
\caption{Ablation of FracTS components, reporting PM (predictive ML utility) MAE at prediction windows $W$ on Weather and Walmart (lower is better). \Na{} denotes ``same as the default setting'' (top row, highlighted).}
\label{tab:ablation}
    \setlength{\tabcolsep}{4pt}
    \resizebox{\linewidth}{!}{
    \begin{tabular}{lllllll ccc ccc}
    \toprule
    & & & & & & & \multicolumn{3}{c}{\textbf{Weather} $\downarrow$} & \multicolumn{3}{c}{\textbf{Walmart} $\downarrow$} \\
    \cmidrule(lr){8-10} \cmidrule(lr){11-13}
    \textbf{Data proc.} & \textbf{Var.-length handling} & \textbf{Length source} & \textbf{Global cond.} & $l^{[i]}$ & \textbf{Backbone} & \textbf{Cond.\ model $q_\phi$} & $W{=}10$ & $W{=}100$ & $W{=}1000$ & $W{=}10$ & $W{=}100$ & $W{=}1000$ \\
    \midrule

    % --- Default Row ---
    \textbf{FracTS (default)} & att.\ mask & predictor & \cmark & 20 & GPT-2 & CTGAN & $0.010$ & $0.010$ & $0.010$ & $0.009$ & $0.007$ & $0.002$ \\
    \midrule

    % --- Ablation Rows ---
     w/o differences   & \Na & \Na & \Na & \Na & \Na & \Na            & $0.021$ & $0.020$ & $0.022$ & $0.027$ & $0.028$ & $0.009$ \\
     w/o bins          & \Na & \Na & \Na & \Na & \Na & \Na            & $0.104$ & $0.109$ & $0.118$ & $0.003$ & $0.024$ & $0.017$ \\
     $k=2$ bins       & \Na & \Na & \Na & \Na & \Na & \Na            & $0.018$ & $0.024$ & $0.034$ & $0.003$ & $0.006$ & $0.008$ \\
    $k=200$ bins   & \Na & \Na & \Na & \Na & \Na & \Na             & $0.034$ & $0.035$ & $0.034$ & $0.009$ & $0.007$ & $0.005$ \\
     w/o aggregation   & \Na & \Na & \Na & \Na & \Na & \Na            & $0.011$ & $0.011$ & $0.014$ & $0.021$ & $0.016$ & $0.011$ \\
     w/o all above     & \Na & \Na & \Na & \Na & \Na & \Na            & $0.027$ & $0.028$ & $0.031$ & $0.033$ & $0.040$ & $0.041$ \\
     \Na               & padding & \Na & \Na & \Na & \Na & \Na        & $0.010$ & $0.010$ & $0.010$ & $0.007$ & $0.011$ & $0.016$ \\
     \Na               & \Na & truncation & \Na & \Na & \Na & \Na     & $0.011$ & $0.012$ & $0.014$ & $0.007$ & $0.009$ & $0.036$ \\
     \Na               & \Na & \Na & \xmark & \Na & \Na & \Na         & $0.020$ & $0.018$ & $0.020$ & $0.004$ & $0.001$ & $0.006$ \\
     \Na               & \Na & \Na & \Na & 10 & \Na & \Na             & $0.020$ & $0.019$ & $0.020$ & $0.008$ & $0.003$ & $0.026$ \\
     \Na               & \Na & \Na & \Na & 100 & \Na & \Na            & $0.014$ & $0.015$ & $0.156$ & $0.078$ & $0.083$ & $0.086$ \\
     \Na & \Na & \Na & \Na & $L$ & \Na & \Na & $0.200$ & \textit{no data} & \textit{no data} & \textit{OOM} & \textit{OOM} & \textit{OOM} \\
     \Na               & \Na & \Na & \Na & \Na & Llama & \Na          & $0.019$ & $0.020$ & $0.019$ & $0.005$ & $0.003$ & $0.002$ \\
     \Na               & \Na & \Na & \Na & \Na & \Na & REaLTabFormer  & $0.010$ & $0.011$ & $0.013$ & $0.007$ & $0.000$ & $0.006$ \\
    \bottomrule
    \end{tabular}
    }
\end{table}

\section{Ablation Study}
\label{app:ablation}
Table~\ref{tab:ablation} shows the ablation experiment results.
In particular, column ``Data proc.'' shows the importance of difference data processing steps to the result. ``Var. len. by'' shows the difference of ways handling variable length input, by na\"ive padding or passing attention mask as input to the transformer. ``Len. by'' shows the source of length predicted. Results by varying these values show the effectiveness of the core design components of FracTS introduced in Section~\ref{sec:model}.

Column ``$l^{[i]}$'' refers to the length of the transformer at each level. In particular, $l^{[i]}=L$ corresponds to a single-level,
non-hierarchical setting (apart from the final layer for raw values). As shown in Table~\ref{tab:ablation}, this setting substantially degrades
performance on Weather and becomes infeasible on the much longer Walmart sequences, highlighting the importance of the hierarchical decomposition. These results demonstrate the effect of core hyperparameters.

Column ``Cond. model $q_\phi$'' refers to the tabular model used for condition generation, and ``Backbone'' refers to the transformer backbone. The results highlight both the impact of strong tabular backbones for condition generation and the influence of transformer size on model performance.

For K-Means discretization of numerical features, we use $k=20$ bins by
default. We additionally evaluate $k=2$ and $k=200$. A very small number
of bins provides limited modal guidance, whereas a very large number of
bins reduces the ordinal relationship between adjacent numerical values
and can degrade performance.

The effect of the Perceiver-based condition compression module is evaluated
separately in Table~\ref{tab:perceiver_ablation}.

For $l^{[i]}=L$, the model reduces to a single-level, non-hierarchical
transformer. On Weather, this setting degrades substantially at $W=10$
and does not generate sufficiently long sequences for $W=100$ and
$W=1000$. On Walmart, training results in out-of-memory (OOM) errors.
These results show that flattening the hierarchy becomes increasingly
problematic as sequence length grows.

\paragraph{Condition compression.}
We further evaluate the Perceiver-based condition compression module on
the Weather dataset. The module compresses the 53-dimensional condition
representation to 16 dimensions. As shown in
Table~\ref{tab:perceiver_ablation}, this reduction preserves predictive ML utility with almost no degradation across the three prediction windows.
We omit this ablation on Walmart because its condition representation is
already low-dimensional, so the Perceiver module provides little effective
compression in this setting.

% \begin{center}
% \inlinecaption{table}{Ablation of the Perceiver-based condition compression module on Weather.
% Values report PM (predictive ML utility) MAE; lower is better.}{tab:perceiver_ablation}
% \small
% \setlength{\tabcolsep}{3pt}
% \begin{tabular}{lccc}
% \toprule
% \textbf{Perceiver compression}
% & $W{=}10$
% & $W{=}100$
% & $W{=}1000$ \\
% \midrule
% Yes (default) & $0.010$ & $0.010$ & $0.010$ \\
% No            & $0.010$ & $0.011$ & $0.011$ \\
% \bottomrule
% \end{tabular}
% \end{center}
\begin{wraptable}{r}{0.52\linewidth}
\centering
\inlinecaption{table}{Ablation of the Perceiver-based condition compression module on Weather.
Values report PM (predictive ML utility) MAE; lower is better.}{tab:perceiver_ablation}
\small
\setlength{\tabcolsep}{3pt}
\begin{tabular}{lccc}
\toprule
\textbf{Perceiver compression}
& $W{=}10$
& $W{=}100$
& $W{=}1000$ \\
\midrule
Yes (default) & $0.010$ & $0.010$ & $0.010$ \\
No            & $0.010$ & $0.011$ & $0.011$ \\
\bottomrule
\end{tabular}
\end{wraptable}

%%%%%%%%%%%%%%%%%%%%%%%%%%%%%%%%%%%%%%%%%%%%%%%%%%%%%%%%%%%%
\section{Limitations}
\label{app:limitations}

\paragraph{Computational overhead.}
FracTS uses multiple transformer models across its hierarchical levels and
therefore incurs additional computational cost. Training runtime, inference
latency, and GPU memory usage are reported in Appendix~\ref{app:fracts}.
While FracTS trades some computational efficiency for its ability to model
long, variable-length time series, its computational requirements remain
practical under the settings evaluated in this work.

\paragraph{Sensitivity to hierarchy depth.}
The hierarchy depth is determined by the local patch size $l^{[i]}$ relative
to the sequence length. Using a fixed appropriate patch size, such as the
default $l^{[i]}=20$, automatically adapts the number of hierarchy levels to
the scale of the dataset. As shown in Appendix~\ref{app:ablation}, performance
degrades mainly under more extreme choices, such as excessively large patch
sizes or the completely flattened single-level setting $l^{[i]}=L$.

\paragraph{Failure modes and scope.}
FracTS is designed primarily for long and variable-length time series.
Applying the model to short, fixed-length sequences may introduce unnecessary
architectural complexity compared with conventional generative models.
Conversely, for extremely long sequences, GPU memory limitations can arise
from the growing effective batch sizes across hierarchical levels; this can
be addressed in practical implementations through level-wise gradient
accumulation or updates.

FracTS also assumes clean inputs represented using standard tabular data types,
including numerical, categorical, and datetime features. It does not directly
handle arbitrary textual metadata, identifier fields, or raw missing values.
For missing values, we assume that standard preprocessing, such as simple
imputation or explicit missing-value indicators, is applied before training.

\paragraph{Irregularly sampled time series.}
FracTS can represent irregularly sampled sequences through its support for
variable sequence lengths and explicit datetime features. For example, the
Stock dataset contains irregularities due to market closures, weekends, and
trading holidays. By using precise timestamps as input features, together
with temporal step differences during preprocessing, FracTS does not require
uniformly spaced observations.

\paragraph{Privacy.}
We evaluate privacy using the Distance-to-Closest-Record (DCR) metric
across different window sizes. Higher values indicate lower privacy
leakage risk, with $0.5$ representing an ideal balance and $1$
representing ideal privacy preservation. The results are reported in
Table~\ref{tab:dcr_privacy}.

\begin{center}
\inlinecaption{table}{
Distance-to-Closest-Record (DCR) privacy evaluation across different
window sizes. Higher values indicate lower privacy leakage risk.
}{tab:dcr_privacy}
\setlength{\tabcolsep}{8pt}
\begin{tabular}{llccc}
\toprule
\textbf{Dataset} & \textbf{Model}
& $W{=}10$
& $W{=}100$
& $W{=}1000$ \\
\midrule
Weather
& FracTS
& $0.515_{\pm 0.010}$
& $0.273_{\pm 0.058}$
& $0.291_{\pm 0.106}$ \\

Weather
& PaD-TS
& $0.592_{\pm 0.035}$
& $0.364_{\pm 0.036}$
& $0.125_{\pm 0.046}$ \\
\midrule
Walmart
& FracTS
& $0.649_{\pm 0.009}$
& $0.314_{\pm 0.022}$
& $0.200_{\pm 0.025}$ \\

Walmart
& TabularARGN
& $0.942_{\pm 0.034}$
& $0.893_{\pm 0.029}$
& $0.849_{\pm 0.081}$ \\
\bottomrule
\end{tabular}
\end{center}

There is an inherent trade-off between fidelity and privacy. FracTS
generally obtains lower DCR values than the compared baselines, although
it performs better in some long-horizon settings, such as compared with
PaD-TS at $W=1000$. Privacy risk is limited for shorter windows
($W=10$), while longer windows show increased leakage risk. This is
partly due to the smaller number of samples available for DCR evaluation
at $W=100$ and $W=1000$, which makes high-fidelity synthetic samples
naturally closer to real samples.

Because FracTS uses a standard transformer backbone, differential privacy
mechanisms such as DP-SGD can be incorporated when stricter privacy
guarantees are required.

%%%%%%%%%%%%%%%%%%%%%%%%%%%%%%%%%%%%%%%%%%%%%%%%%%%%%%%%%%%%

\clearpage
\section*{NeurIPS Paper Checklist}

\begin{enumerate}

\item {\bf Claims}
    \item[] Question: Do the main claims made in the abstract and introduction accurately reflect the paper's contributions and scope?
    \item[] Answer: \answerYes{}
    \item[] Justification: The abstract and introduction clearly outline the introduction of FracTS, its hierarchical and autoregressive nature, and its performance across five datasets compared to seven baselines.
    \item[] Guidelines:
    \begin{itemize}
        \item The answer NA means that the abstract and introduction do not include the claims made in the paper.
        \item The abstract and/or introduction should clearly state the claims made, including the contributions made in the paper and important assumptions and limitations. A No or NA answer to this question will not be perceived well by the reviewers. 
        \item The claims made should match theoretical and experimental results, and reflect how much the results can be expected to generalize to other settings. 
        \item It is fine to include aspirational goals as motivation as long as it is clear that these goals are not attained by the paper. 
    \end{itemize}

\item {\bf Limitations}
    \item[] Question: Does the paper discuss the limitations of the work performed by the authors?
    \item[] Answer: \answerYes{}
   \item[] Justification: The paper includes a dedicated limitations section
(Appendix~\ref{app:limitations}) discussing computational overhead,
sensitivity to hierarchy depth, failure modes and scope, handling of
irregularly sampled and missing-value data, and privacy considerations.
It also discusses performance limitations on highly complex composite
and oscillatory signals.
    \item[] Guidelines:
    \begin{itemize}
        \item The answer NA means that the paper has no limitation while the answer No means that the paper has limitations, but those are not discussed in the paper. 
        \item The authors are encouraged to create a separate "Limitations" section in their paper.
        \item The paper should point out any strong assumptions and how robust the results are to violations of these assumptions (e.g., independence assumptions, noiseless settings, model well-specification, asymptotic approximations only holding locally). The authors should reflect on how these assumptions might be violated in practice and what the implications would be.
        \item The authors should reflect on the scope of the claims made, e.g., if the approach was only tested on a few datasets or with a few runs. In general, empirical results often depend on implicit assumptions, which should be articulated.
        \item The authors should reflect on the factors that influence the performance of the approach. For example, a facial recognition algorithm may perform poorly when image resolution is low or images are taken in low lighting. Or a speech-to-text system might not be used reliably to provide closed captions for online lectures because it fails to handle technical jargon.
        \item The authors should discuss the computational efficiency of the proposed algorithms and how they scale with dataset size.
        \item If applicable, the authors should discuss possible limitations of their approach to address problems of privacy and fairness.
        \item While the authors might fear that complete honesty about limitations might be used by reviewers as grounds for rejection, a worse outcome might be that reviewers discover limitations that aren't acknowledged in the paper. The authors should use their best judgment and recognize that individual actions in favor of transparency play an important role in developing norms that preserve the integrity of the community. Reviewers will be specifically instructed to not penalize honesty concerning limitations.
    \end{itemize}

\item {\bf Theory Assumptions and Proofs}
    \item[] Question: For each theoretical result, does the paper provide the full set of assumptions and a complete (and correct) proof?
    \item[] Answer: \answerYes{}
    \item[] Justification: The paper provides two theoretical results in Section~\ref{sec:model}: an analysis of error propagation in the hierarchical autoregressive architecture and an analysis of its activation memory complexity. The assumptions are stated in the theorem statements, and proofs are provided in Appendix~\ref{app:theory}.
    \item[] Guidelines:
    \begin{itemize}
        \item The answer NA means that the paper does not include theoretical results. 
        \item All the theorems, formulas, and proofs in the paper should be numbered and cross-referenced.
        \item All assumptions should be clearly stated or referenced in the statement of any theorems.
        \item The proofs can either appear in the main paper or the supplemental material, but if they appear in the supplemental material, the authors are encouraged to provide a short proof sketch to provide intuition. 
        \item Inversely, any informal proof provided in the core of the paper should be complemented by formal proofs provided in appendix or supplemental material.
        \item Theorems and Lemmas that the proof relies upon should be properly referenced. 
    \end{itemize}

    \item {\bf Experimental Result Reproducibility}
    \item[] Question: Does the paper fully disclose all the information needed to reproduce the main experimental results of the paper to the extent that it affects the main claims and/or conclusions of the paper (regardless of whether the code and data are provided or not)?
    \item[] Answer: \answerYes{}
    \item[] Justification: The paper provides implementation details, including the architecture (GPT2 small), optimization parameters (AdamW, learning rate, weight decay), and specific hyperparameters for different datasets.
    \item[] Guidelines:
    \begin{itemize}
        \item The answer NA means that the paper does not include experiments.
        \item If the paper includes experiments, a No answer to this question will not be perceived well by the reviewers: Making the paper reproducible is important, regardless of whether the code and data are provided or not.
        \item If the contribution is a dataset and/or model, the authors should describe the steps taken to make their results reproducible or verifiable. 
        \item Depending on the contribution, reproducibility can be accomplished in various ways. For example, if the contribution is a novel architecture, describing the architecture fully might suffice, or if the contribution is a specific model and empirical evaluation, it may be necessary to either make it possible for others to replicate the model with the same dataset, or provide access to the model. In general. releasing code and data is often one good way to accomplish this, but reproducibility can also be provided via detailed instructions for how to replicate the results, access to a hosted model (e.g., in the case of a large language model), releasing of a model checkpoint, or other means that are appropriate to the research performed.
        \item While NeurIPS does not require releasing code, the conference does require all submissions to provide some reasonable avenue for reproducibility, which may depend on the nature of the contribution. For example
        \begin{enumerate}
            \item If the contribution is primarily a new algorithm, the paper should make it clear how to reproduce that algorithm.
            \item If the contribution is primarily a new model architecture, the paper should describe the architecture clearly and fully.
            \item If the contribution is a new model (e.g., a large language model), then there should either be a way to access this model for reproducing the results or a way to reproduce the model (e.g., with an open-source dataset or instructions for how to construct the dataset).
            \item We recognize that reproducibility may be tricky in some cases, in which case authors are welcome to describe the particular way they provide for reproducibility. In the case of closed-source models, it may be that access to the model is limited in some way (e.g., to registered users), but it should be possible for other researchers to have some path to reproducing or verifying the results.
        \end{enumerate}
    \end{itemize}

\item {\bf Open access to data and code}
    \item[] Question: Does the paper provide open access to the data and code, with sufficient instructions to faithfully reproduce the main experimental results, as described in supplemental material?
    \item[] Answer: \answerYes{}
    \item[] Justification: The code is publicly available at
\url{https://github.com/zhao-zilong/fracts}.
    \item[] Guidelines:
    \begin{itemize}
        \item The answer NA means that paper does not include experiments requiring code.
        \item Please see the NeurIPS code and data submission guidelines (\url{https://nips.cc/public/guides/CodeSubmissionPolicy}) for more details.
        \item While we encourage the release of code and data, we understand that this might not be possible, so “No” is an acceptable answer. Papers cannot be rejected simply for not including code, unless this is central to the contribution (e.g., for a new open-source benchmark).
        \item The instructions should contain the exact command and environment needed to run to reproduce the results. See the NeurIPS code and data submission guidelines (\url{https://nips.cc/public/guides/CodeSubmissionPolicy}) for more details.
        \item The authors should provide instructions on data access and preparation, including how to access the raw data, preprocessed data, intermediate data, and generated data, etc.
        \item The authors should provide scripts to reproduce all experimental results for the new proposed method and baselines. If only a subset of experiments are reproducible, they should state which ones are omitted from the script and why.
        \item At submission time, to preserve anonymity, the authors should release anonymized versions (if applicable).
        \item Providing as much information as possible in supplemental material (appended to the paper) is recommended, but including URLs to data and code is permitted.
    \end{itemize}

\item {\bf Experimental Setting/Details}
    \item[] Question: Does the paper specify all the training and test details (e.g., data splits, hyperparameters, how they were chosen, type of optimizer, etc.) necessary to understand the results?
    \item[] Answer: \answerYes{}
    \item[] Justification: Section~\ref{sec:exp} and Appendix~\ref{app:exp:config} detail the datasets, data processing pipeline, model architecture, and the specific evaluation setup used.
    \item[] Guidelines:
    \begin{itemize}
        \item The answer NA means that the paper does not include experiments.
        \item The experimental setting should be presented in the core of the paper to a level of detail that is necessary to appreciate the results and make sense of them.
        \item The full details can be provided either with the code, in appendix, or as supplemental material.
    \end{itemize}

\item {\bf Experiment Statistical Significance}
    \item[] Question: Does the paper report error bars suitably and correctly defined or other appropriate information about the statistical significance of the experiments?
    \item[] Answer: \answerYes{}
    \item[] Justification: Performance tables, such as Table~\ref{tab:uncond_full}, include mean values and standard deviations for the reported metrics.
    \item[] Guidelines:
    \begin{itemize}
        \item The answer NA means that the paper does not include experiments.
        \item The authors should answer "Yes" if the results are accompanied by error bars, confidence intervals, or statistical significance tests, at least for the experiments that support the main claims of the paper.
        \item The factors of variability that the error bars are capturing should be clearly stated (for example, train/test split, initialization, random drawing of some parameter, or overall run with given experimental conditions).
        \item The method for calculating the error bars should be explained (closed form formula, call to a library function, bootstrap, etc.)
        \item The assumptions made should be given (e.g., Normally distributed errors).
        \item It should be clear whether the error bar is the standard deviation or the standard error of the mean.
        \item It is OK to report 1-sigma error bars, but one should state it. The authors should preferably report a 2-sigma error bar than state that they have a 96\% CI, if the hypothesis of Normality of errors is not verified.
        \item For asymmetric distributions, the authors should be careful not to show in tables or figures symmetric error bars that would yield results that are out of range (e.g. negative error rates).
        \item If error bars are reported in tables or plots, The authors should explain in the text how they were calculated and reference the corresponding figures or tables in the text.
    \end{itemize}

\item {\bf Experiments Compute Resources}
    \item[] Question: For each experiment, does the paper provide sufficient information on the computer resources (type of compute workers, memory, time of execution) needed to reproduce the experiments?
    \item[] Answer: \answerYes{}
    \item[] Justification: The paper specifies that all experiments were run on NVIDIA H100 GPUs using PyTorch with mixed precision training.
    \item[] Guidelines:
    \begin{itemize}
        \item The answer NA means that the paper does not include experiments.
        \item The paper should indicate the type of compute workers CPU or GPU, internal cluster, or cloud provider, including relevant memory and storage.
        \item The paper should provide the amount of compute required for each of the individual experimental runs as well as estimate the total compute. 
        \item The paper should disclose whether the full research project required more compute than the experiments reported in the paper (e.g., preliminary or failed experiments that didn't make it into the paper). 
    \end{itemize}
    
\item {\bf Code Of Ethics}
    \item[] Question: Does the research conducted in the paper conform, in every respect, with the NeurIPS Code of Ethics \url{https://neurips.cc/public/EthicsGuidelines}?
    \item[] Answer: \answerYes{}
    \item[] Justification: The paper follows standard research practices and includes a discussion on the ethical use of data collection .
    \item[] Guidelines:
    \begin{itemize}
        \item The answer NA means that the authors have not reviewed the NeurIPS Code of Ethics.
        \item If the authors answer No, they should explain the special circumstances that require a deviation from the Code of Ethics.
        \item The authors should make sure to preserve anonymity (e.g., if there is a special consideration due to laws or regulations in their jurisdiction).
    \end{itemize}

\item {\bf Broader Impacts}
    \item[] Question: Does the paper discuss both potential positive societal impacts and negative societal impacts of the work performed?
    \item[] Answer: \answerYes{}
    \item[] Justification: The paper notes that synthetic data generation is important for privacy-preserving sharing, but also mentions the necessity of ethical data collection practices .
    \item[] Guidelines:
    \begin{itemize}
        \item The answer NA means that there is no societal impact of the work performed.
        \item If the authors answer NA or No, they should explain why their work has no societal impact or why the paper does not address societal impact.
        \item Examples of negative societal impacts include potential malicious or unintended uses (e.g., disinformation, generating fake profiles, surveillance), fairness considerations (e.g., deployment of technologies that could make decisions that unfairly impact specific groups), privacy considerations, and security considerations.
        \item The conference expects that many papers will be foundational research and not tied to particular applications, let alone deployments. However, if there is a direct path to any negative applications, the authors should point it out. For example, it is legitimate to point out that an improvement in the quality of generative models could be used to generate deepfakes for disinformation. On the other hand, it is not needed to point out that a generic algorithm for optimizing neural networks could enable people to train models that generate Deepfakes faster.
        \item The authors should consider possible harms that could arise when the technology is being used as intended and functioning correctly, harms that could arise when the technology is being used as intended but gives incorrect results, and harms following from (intentional or unintentional) misuse of the technology.
        \item If there are negative societal impacts, the authors could also discuss possible mitigation strategies (e.g., gated release of models, providing defenses in addition to attacks, mechanisms for monitoring misuse, mechanisms to monitor how a system learns from feedback over time, improving the efficiency and accessibility of ML).
    \end{itemize}
    
\item {\bf Safeguards}
    \item[] Question: Does the paper describe safeguards that have been put in place for responsible release of data or models that have a high risk for misuse (e.g., pretrained language models, image generators, or scraped datasets)?
    \item[] Answer: \answerNA{}
    \item[] Justification: The paper focuses on a general architectural framework for time series and does not release high-risk pretrained language or image models that require specific safety filters.
    \item[] Guidelines:
    \begin{itemize}
        \item The answer NA means that the paper poses no such risks.
        \item Released models that have a high risk for misuse or dual-use should be released with necessary safeguards to allow for controlled use of the model, for example by requiring that users adhere to usage guidelines or restrictions to access the model or implementing safety filters. 
        \item Datasets that have been scraped from the Internet could pose safety risks. The authors should describe how they avoided releasing unsafe images.
        \item We recognize that providing effective safeguards is challenging, and many papers do not require this, but we encourage authors to take this into account and make a best faith effort.
    \end{itemize}

\item {\bf Licenses for existing assets}
    \item[] Question: Are the creators or original owners of assets (e.g., code, data, models), used in the paper, properly credited and are the license and terms of use explicitly mentioned and properly respected?
    \item[] Answer: \answerYes{}
    \item[] Justification: The paper cites the original sources for all five datasets used (e.g., Sensor, Weather, Stock, Walmart).
    \item[] Guidelines:
    \begin{itemize}
        \item The answer NA means that the paper does not use existing assets.
        \item The authors should cite the original paper that produced the code package or dataset.
        \item The authors should state which version of the asset is used and, if possible, include a URL.
        \item The name of the license (e.g., CC-BY 4.0) should be included for each asset.
        \item For scraped data from a particular source (e.g., website), the copyright and terms of service of that source should be provided.
        \item If assets are released, the license, copyright information, and terms of use in the package should be provided. For popular datasets, \url{paperswithcode.com/datasets} has curated licenses for some datasets. Their licensing guide can help determine the license of a dataset.
        \item For existing datasets that are re-packaged, both the original license and the license of the derived asset (if it has changed) should be provided.
        \item If this information is not available online, the authors are encouraged to reach out to the asset's creators.
    \end{itemize}

\item {\bf New Assets}
    \item[] Question: Are new assets introduced in the paper well documented and is the documentation provided alongside the assets?
    \item[] Answer: \answerYes{}
    \item[] Justification: The paper introduces a new simulated dataset with 15 columns of diverse temporal patterns, which is documented in Section~\ref{sec:exp:setup} and Appendix~\ref{app:dataset:simulate}.
    \item[] Guidelines:
    \begin{itemize}
        \item The answer NA means that the paper does not release new assets.
        \item Researchers should communicate the details of the dataset/code/model as part of their submissions via structured templates. This includes details about training, license, limitations, etc. 
        \item The paper should discuss whether and how consent was obtained from people whose asset is used.
        \item At submission time, remember to anonymize your assets (if applicable). You can either create an anonymized URL or include an anonymized zip file.
    \end{itemize}

\item {\bf Crowdsourcing and Research with Human Subjects}
    \item[] Question: For crowdsourcing experiments and research with human subjects, does the paper include the full text of instructions given to participants and screenshots, if applicable, as well as details about compensation (if any)?
    \item[] Answer: \answerNA{}
    \item[] Justification: The research does not involve crowdsourcing or human subjects.
    \item[] Guidelines:
    \begin{itemize}
        \item The answer NA means that the paper does not involve crowdsourcing nor research with human subjects.
        \item Including this information in the supplemental material is fine, but if the main contribution of the paper involves human subjects, then as much detail as possible should be included in the main paper. 
        \item According to the NeurIPS Code of Ethics, workers involved in data collection, curation, or other labor should be paid at least the minimum wage in the country of the data collector. 
    \end{itemize}

\item {\bf Institutional Review Board (IRB) Approvals or Equivalent for Research with Human Subjects}
    \item[] Question: Does the paper describe potential risks incurred by study participants, whether such risks were disclosed to the subjects, and whether Institutional Review Board (IRB) approvals (or an equivalent approval/review based on the requirements of your country or institution) were obtained?
    \item[] Answer: \answerNA{}
    \item[] Justification: The research does not involve human subjects, so IRB approval is not applicable.
    \item[] Guidelines:
    \begin{itemize}
        \item The answer NA means that the paper does not involve crowdsourcing nor research with human subjects.
        \item Depending on the country in which research is conducted, IRB approval (or equivalent) may be required for any human subjects research. If you obtained IRB approval, you should clearly state this in the paper. 
        \item We recognize that the procedures for this may vary significantly between institutions and locations, and we expect authors to adhere to the NeurIPS Code of Ethics and the guidelines for their institution. 
        \item For initial submissions, do not include any information that would break anonymity (if applicable), such as the institution conducting the review.
    \end{itemize}

\end{enumerate}

\end{document}

%% file: math_commands.tex
\usepackage{amsmath,amsfonts,bm}

\def\eqref#1{equation~\ref{#1}}
\def\1{\bm{1}}

\DeclareMathAlphabet{\mathsfit}{\encodingdefault}{\sfdefault}{m}{sl}
\SetMathAlphabet{\mathsfit}{bold}{\encodingdefault}{\sfdefault}{bx}{n}

%% file: neurips_2026.bib
@misc{fractal,
      title={Fractal Generative Models}, 
      author={Tianhong Li and Qinyi Sun and Lijie Fan and Kaiming He},
      year={2025},
      eprint={2502.17437},
      archivePrefix={arXiv},
      primaryClass={cs.LG},
      url={https://arxiv.org/abs/2502.17437}, 
}

@inproceedings{transformer,
 author = {Vaswani, Ashish and Shazeer, Noam and Parmar, Niki and Uszkoreit, Jakob and Jones, Llion and Gomez, Aidan N and Kaiser, \L ukasz and Polosukhin, Illia},
 booktitle = {Advances in Neural Information Processing Systems},
 editor = {I. Guyon and U. Von Luxburg and S. Bengio and H. Wallach and R. Fergus and S. Vishwanathan and R. Garnett},
 pages = {},
 publisher = {Curran Associates, Inc.},
 title = {Attention is All you Need},
 url = {https://proceedings.neurips.cc/paper_files/paper/2017/file/3f5ee243547dee91fbd053c1c4a845aa-Paper.pdf},
 volume = {30},
 year = {2017}
}

@inproceedings{ctgan,
 author = {Xu, Lei and Skoularidou, Maria and Cuesta-Infante, Alfredo and Veeramachaneni, Kalyan},
 booktitle = {Advances in Neural Information Processing Systems},
 editor = {H. Wallach and H. Larochelle and A. Beygelzimer and F. d\textquotesingle Alch\'{e}-Buc and E. Fox and R. Garnett},
 pages = {},
 publisher = {Curran Associates, Inc.},
 title = {Modeling Tabular data using Conditional {GAN}},
 url = {https://proceedings.neurips.cc/paper_files/paper/2019/file/254ed7d2de3b23ab10936522dd547b78-Paper.pdf},
 volume = {32},
 year = {2019}
}

@inproceedings{
tabsyn,
title={Mixed-Type Tabular Data Synthesis with Score-based Diffusion in Latent Space},
author={Hengrui Zhang and Jiani Zhang and Zhengyuan Shen and Balasubramaniam Srinivasan and Xiao Qin and Christos Faloutsos and Huzefa Rangwala and George Karypis},
booktitle={The Twelfth International Conference on Learning Representations},
year={2024},
url={https://openreview.net/forum?id=4Ay23yeuz0}
}

@InProceedings{tvqvae,
  title = 	 {Vector Quantized Time Series Generation with a Bidirectional Prior Model},
  author =       {Lee, Daesoo and Malacarne, Sara and Aune, Erlend},
  booktitle = 	 {Proceedings of The 26th International Conference on Artificial Intelligence and Statistics},
  pages = 	 {7665--7693},
  year = 	 {2023},
  editor = 	 {Ruiz, Francisco and Dy, Jennifer and van de Meent, Jan-Willem},
  volume = 	 {206},
  series = 	 {Proceedings of Machine Learning Research},
  month = 	 {25--27 Apr},
  publisher =    {PMLR},
  url = 	 {https://proceedings.mlr.press/v206/lee23d.html}
}

@inproceedings{tabdiff,
title={{TabDiff}: a Mixed-type Diffusion Model for Tabular Data Generation},
author={Juntong Shi and Minkai Xu and Harper Hua and Hengrui Zhang and Stefano Ermon and Jure Leskovec},
booktitle={The Thirteenth International Conference on Learning Representations},
year={2025},
url={https://openreview.net/forum?id=swvURjrt8z}
}

@misc{tabtreeformer,
      title={{TabTreeFormer}: Tabular Data Generation Using Hybrid Tree-Transformer}, 
      author={Jiayu Li and Bingyin Zhao and Zilong Zhao and Uzair Javaid and Kevin Yee and Biplab Sikdar},
      year={2025},
      eprint={2501.01216},
      archivePrefix={arXiv},
      primaryClass={cs.LG},
      url={https://arxiv.org/abs/2501.01216}, 
}

@inproceedings{banksformer,
author = {Nickerson, Kyle and Tricco, Terrence and Kolokolova, Antonina and Shoeleh, Farzaneh and Robertson, Charles and Hawkin, John and Hu, Ting},
title = {Banksformer: A Deep Generative Model for Synthetic Transaction Sequences},
year = {2022},
isbn = {978-3-031-26421-4},
publisher = {Springer-Verlag},
address = {Berlin, Heidelberg},
url = {https://doi.org/10.1007/978-3-031-26422-1_8},
doi = {10.1007/978-3-031-26422-1_8},
booktitle = {Machine Learning and Knowledge Discovery in Databases: European Conference, ECML PKDD 2022, Grenoble, France, September 19–23, 2022, Proceedings, Part VI},
pages = {121–136},
numpages = {16},
location = {Grenoble, France}
}

@article{aecgan, title={{AEC-GAN}: Adversarial Error Correction {GAN}s for Auto-Regressive Long Time-Series Generation}, volume={37}, url={https://ojs.aaai.org/index.php/AAAI/article/view/26208}, DOI={10.1609/aaai.v37i8.26208}, number={8}, journal={Proceedings of the AAAI Conference on Artificial Intelligence}, author={Wang, Lei and Zeng, Liang and Li, Jian}, year={2023}, month={Jun.}, pages={10140-10148} }

@misc{weatherdataset,
  author       = {Arunava Kr. Chakraborty},
  title        = {Australia Weather Data},
  year         = {2023},
  howpublished = {\url{https://www.kaggle.com/datasets/arunavakrchakraborty/australia-weather-data}},
  note         = {Data originally sourced from the Australian Bureau of Meteorology (\url{http://www.bom.gov.au/climate/data/}). Accessed: 2025-07-03}
}

@article{sensordataset,
    title={{WSN-DS}: A dataset for intrusion detection systems in wireless sensor networks},
    author={Almomani, Iman and Al-Kasasbeh, Bassam and Al-Akhras, Mousa},
    journal={Journal of Sensors},
    volume={2016},
    year={2016},
    publisher={Hindawi},
    note={Dataset available at: \url{https://www.kaggle.com/datasets/bassamkasasbeh1/wsnds}}
}

@misc{stocksdataset,
  author       = {Boris Marjanovic},
  title        = {Huge Stock Market Dataset},
  year         = {2018},
  howpublished = {\url{https://www.kaggle.com/datasets/borismarjanovic/price-volume-data-for-all-us-stocks-etfs?select=Stocks}},
  note         = {Accessed: 2025-07-03}
}

@misc{walmartdataset,
  author       = {Aslan Ahmedov},
  title        = {Walmart Sales Forecast},
  year         = {2021},
  howpublished = {\url{https://www.kaggle.com/datasets/aslanahmedov/walmart-sales-forecast?select=stores.csv}},
  note         = {Accessed: 2025-10-15}
}

@inproceedings{diffusionts,
title={Diffusion-{TS}: Interpretable Diffusion for General Time Series Generation},
author={Xinyu Yuan and Yan Qiao},
booktitle={The Twelfth International Conference on Learning Representations},
year={2024},
url={https://openreview.net/forum?id=4h1apFjO99}
}

@inproceedings{coscigan,
 author = {Seyfi, Ali and Rajotte, Jean-Francois and Ng, Raymond},
 booktitle = {Advances in Neural Information Processing Systems},
 editor = {S. Koyejo and S. Mohamed and A. Agarwal and D. Belgrave and K. Cho and A. Oh},
 pages = {32777--32788},
 publisher = {Curran Associates, Inc.},
 title = {Generating multivariate time series with {CO}mmon {S}ource {C}oord{I}nated {GAN} ({COSCI-GAN})},
 url = {https://proceedings.neurips.cc/paper_files/paper/2022/file/d3408794e41dd23e34634344d662f5e9-Paper-Conference.pdf},
 volume = {35},
 year = {2022}
}

@inproceedings{kovae,
title={Generative Modeling of Regular and Irregular Time Series Data via {K}oopman {VAE}s},
author={Ilan Naiman and N. Benjamin Erichson and Pu Ren and Michael W. Mahoney and Omri Azencot},
booktitle={The Twelfth International Conference on Learning Representations},
year={2024},
url={https://openreview.net/forum?id=eY7sLb0dVF}
}

@ARTICLE{chirp,
  author={Klauder, J. R. and Price, A. C. and Darlington, S. and Albersheim, W. J.},
  journal={The Bell System Technical Journal}, 
  title={The theory and design of chirp radars}, 
  year={1960},
  volume={39},
  number={4},
  pages={745-808},
  doi={10.1002/j.1538-7305.1960.tb03942.x}}

@article{am,
  author = {Fessenden, Reginald A.},
  title = {Wireless Telephony},
  journal = {Transactions of the American Institute of Electrical Engineers},
  volume = {XXVII},
  number = {Part 1},
  pages = {553--629},
  year = {1908},
  month = {June 29}
}

@article{gabor,
author = {D. Gabor },
title = {Theory of communication. Part 1: The analysis of information},
journal = {Journal of the Institution of Electrical Engineers - Part III: Radio and Communication Engineering},
volume = {93},
issue = {26},
pages = {429-441},
year = {1946},
doi = {10.1049/ji-3-2.1946.0074},

URL = {https://digital-library.theiet.org/doi/abs/10.1049/ji-3-2.1946.0074},
eprint = {https://digital-library.theiet.org/doi/pdf/10.1049/ji-3-2.1946.0074}
}

@inproceedings{timegan,
 author = {Yoon, Jinsung and Jarrett, Daniel and van der Schaar, Mihaela},
 booktitle = {Advances in Neural Information Processing Systems},
 editor = {H. Wallach and H. Larochelle and A. Beygelzimer and F. d\textquotesingle Alch\'{e}-Buc and E. Fox and R. Garnett},
 pages = {},
 publisher = {Curran Associates, Inc.},
 title = {Time-series Generative Adversarial Networks},
 url = {https://proceedings.neurips.cc/paper_files/paper/2019/file/c9efe5f26cd17ba6216bbe2a7d26d490-Paper.pdf},
 volume = {32},
 year = {2019}
}

@inproceedings{wavenet,
  title     = {{WaveNet}: A Generative Model for Raw Audio},
  author    = {Aäron {van den Oord} and Sander Dieleman and Heiga Zen and Karen Simonyan and Oriol Vinyals and Alex Graves and Nal Kalchbrenner and Andrew Senior and Koray Kavukcuoglu},
  year      = {2016},
  booktitle = {9th ISCA Workshop on Speech Synthesis Workshop (SSW 9)},
  pages     = {125},
}

@misc{tforcing,
      title={Generating Sequences With Recurrent Neural Networks}, 
      author={Alex Graves},
      year={2014},
      eprint={1308.0850},
      archivePrefix={arXiv},
      primaryClass={cs.NE},
      url={https://arxiv.org/abs/1308.0850}, 
}

@inproceedings{pforcing,
author = {Goyal, Anirudh and Lamb, Alex and Zhang, Ying and Zhang, Saizheng and Courville, Aaron and Bengio, Yoshua},
title = {Professor forcing: a new algorithm for training recurrent networks},
year = {2016},
publisher = {Curran Associates Inc.},
address = {Red Hook, NY, USA},
booktitle = {Advances in Neural Information Processing Systems},
pages = {4608–4616},
numpages = {9},
location = {Barcelona, Spain},
volume={29},
url={https://proceedings.neurips.cc/paper_files/paper/2016/file/16026d60ff9b54410b3435b403afd226-Paper.pdf},
editor={D. Lee and M. Sugiyama and U. Luxburg and I. Guyon and R. Garnett}
}

@inproceedings{ar-limit,
    title = "Limitations of Autoregressive Models and Their Alternatives",
    author = "Lin, Chu-Cheng  and
      Jaech, Aaron  and
      Li, Xin  and
      Gormley, Matthew R.  and
      Eisner, Jason",
    editor = "Toutanova, Kristina  and
      Rumshisky, Anna  and
      Zettlemoyer, Luke  and
      Hakkani-Tur, Dilek  and
      Beltagy, Iz  and
      Bethard, Steven  and
      Cotterell, Ryan  and
      Chakraborty, Tanmoy  and
      Zhou, Yichao",
    booktitle = "Proceedings of the 2021 Conference of the North American Chapter of the Association for Computational Linguistics: Human Language Technologies",
    month = jun,
    year = "2021",
    address = "Online",
    publisher = "Association for Computational Linguistics",
    url = "https://aclanthology.org/2021.naacl-main.405/",
    doi = "10.18653/v1/2021.naacl-main.405",
    pages = "5147--5173"
}

@inproceedings{gan,
 author = {Goodfellow, Ian J. and Pouget-Abadie, Jean and Mirza, Mehdi and Xu, Bing and Warde-Farley, David and Ozair, Sherjil and Courville, Aaron and Bengio, Yoshua},
 booktitle = {Advances in Neural Information Processing Systems},
 editor = {Z. Ghahramani and M. Welling and C. Cortes and N. Lawrence and K.Q. Weinberger},
 pages = {},
 publisher = {Curran Associates, Inc.},
 title = {Generative Adversarial Nets},
 url = {https://proceedings.neurips.cc/paper_files/paper/2014/file/f033ed80deb0234979a61f95710dbe25-Paper.pdf},
 volume = {27},
 year = {2014}
}

@inproceedings{ddpm,
 author = {Ho, Jonathan and Jain, Ajay and Abbeel, Pieter},
 booktitle = {Advances in Neural Information Processing Systems},
 editor = {H. Larochelle and M. Ranzato and R. Hadsell and M.F. Balcan and H. Lin},
 pages = {6840--6851},
 publisher = {Curran Associates, Inc.},
 title = {Denoising Diffusion Probabilistic Models},
 url = {https://proceedings.neurips.cc/paper_files/paper/2020/file/4c5bcfec8584af0d967f1ab10179ca4b-Paper.pdf},
 volume = {33},
 year = {2020}
}

@InProceedings{mrmts,
  title = 	 {Hierarchical Deep Generative Models for Multi-Rate Multivariate Time Series},
  author =       {Che, Zhengping and Purushotham, Sanjay and Li, Guangyu and Jiang, Bo and Liu, Yan},
  booktitle = 	 {Proceedings of the 35th International Conference on Machine Learning},
  pages = 	 {784--793},
  year = 	 {2018},
  editor = 	 {Dy, Jennifer and Krause, Andreas},
  volume = 	 {80},
  series = 	 {Proceedings of Machine Learning Research},
  month = 	 {10--15 Jul},
  publisher =    {PMLR},
  url = 	 {https://proceedings.mlr.press/v80/che18a.html},
}

@InProceedings{dghl,
  title = 	 { Deep Generative model with Hierarchical Latent Factors for Time Series Anomaly Detection },
  author =       {Challu, Cristian I. and Jiang, Peihong and Nian Wu, Ying and Callot, Laurent},
  booktitle = 	 {Proceedings of The 25th International Conference on Artificial Intelligence and Statistics},
  pages = 	 {1643--1654},
  year = 	 {2022},
  editor = 	 {Camps-Valls, Gustau and Ruiz, Francisco J. R. and Valera, Isabel},
  volume = 	 {151},
  series = 	 {Proceedings of Machine Learning Research},
  month = 	 {28--30 Mar},
  publisher =    {PMLR},
  url = 	 {https://proceedings.mlr.press/v151/challu22a.html}
}

@inproceedings{himtm,
author = {Zhao, Shubao and Jin, Ming and Hou, Zhaoxiang and Yang, Chengyi and Li, Zengxiang and Wen, Qingsong and Wang, Yi},
title = {{HiMTM}: Hierarchical Multi-Scale Masked Time Series Modeling with Self-Distillation for Long-Term Forecasting},
year = {2024},
isbn = {9798400704369},
publisher = {Association for Computing Machinery},
address = {New York, NY, USA},
url = {https://doi.org/10.1145/3627673.3679741},
doi = {10.1145/3627673.3679741},
booktitle = {Proceedings of the 33rd ACM International Conference on Information and Knowledge Management},
pages = {3352–3362},
numpages = {11},
location = {Boise, ID, USA},
series = {CIKM '24}
}

@inproceedings{laddervae,
author = {S\o{}nderby, Casper Kaae and Raiko, Tapani and Maal\o{}e, Lars and S\o{}nderby, S\o{}ren Kaae and Winther, Ole},
title = {Ladder variational autoencoders},
year = {2016},
publisher = {Curran Associates Inc.},
address = {Red Hook, NY, USA},
booktitle = {Advances in Neural Information Processing Systems},
pages = {3745–3753},
numpages = {9},
location = {Barcelona, Spain},
volume={29},
editor={D. Lee and M. Sugiyama and U. Luxburg and I. Guyon and R. Garnett}
}

@inproceedings{nvae,
 author = {Vahdat, Arash and Kautz, Jan},
 booktitle = {Advances in Neural Information Processing Systems},
 editor = {H. Larochelle and M. Ranzato and R. Hadsell and M.F. Balcan and H. Lin},
 pages = {19667--19679},
 publisher = {Curran Associates, Inc.},
 title = {{NVAE}: A Deep Hierarchical Variational Autoencoder},
 url = {https://proceedings.neurips.cc/paper_files/paper/2020/file/e3b21256183cf7c2c7a66be163579d37-Paper.pdf},
 volume = {33},
 year = {2020}
}

@InProceedings{hngan,
author = {Zhang, Zizhao and Xie, Yuanpu and Yang, Lin},
title = {Photographic Text-to-Image Synthesis With a Hierarchically-Nested Adversarial Network},
booktitle = {Proceedings of the IEEE Conference on Computer Vision and Pattern Recognition (CVPR)},
month = {June},
year = {2018}
}

@inproceedings{proggan,
title={Progressive Growing of {GAN}s for Improved Quality, Stability, and Variation},
author={Tero Karras and Timo Aila and Samuli Laine and Jaakko Lehtinen},
booktitle={International Conference on Learning Representations},
year={2018},
url={https://openreview.net/forum?id=Hk99zCeAb},
}

@inproceedings{
psagan,
title={{PSA}-{GAN}: Progressive Self Attention {GAN}s for Synthetic Time Series},
author={Paul Jeha and Michael Bohlke-Schneider and Pedro Mercado and Shubham Kapoor and Rajbir Singh Nirwan and Valentin Flunkert and Jan Gasthaus and Tim Januschowski},
booktitle={International Conference on Learning Representations},
year={2022},
url={https://openreview.net/forum?id=Ix_mh42xq5w}
}

@inproceedings{vae,
  author       = {Diederik P. Kingma and
                  Max Welling},
  editor       = {Yoshua Bengio and
                  Yann LeCun},
  title        = {Auto-Encoding Variational Bayes},
  booktitle    = {2nd International Conference on Learning Representations, {ICLR} 2014,
                  Banff, AB, Canada, April 14-16, 2014, Conference Track Proceedings},
  year         = {2014},
  url          = {http://arxiv.org/abs/1312.6114},
  bibsource    = {dblp computer science bibliography, https://dblp.org}
}

@InProceedings{nfldm,
    author    = {Kim, Seung Wook and Brown, Bradley and Yin, Kangxue and Kreis, Karsten and Schwarz, Katja and Li, Daiqing and Rombach, Robin and Torralba, Antonio and Fidler, Sanja},
    title     = {{NeuralField-LDM}: Scene Generation With Hierarchical Latent Diffusion Models},
    booktitle = {Proceedings of the IEEE/CVF Conference on Computer Vision and Pattern Recognition (CVPR)},
    month     = {June},
    year      = {2023},
    pages     = {8496-8506}
}

@misc{wavestitch,
      title={WaveStitch: Flexible and Fast Conditional Time Series Generation with Diffusion Models}, 
      author={Aditya Shankar and Lydia Y. Chen and Arie van Deursen and Rihan Hai},
      year={2025},
      eprint={2503.06231},
      archivePrefix={arXiv},
      primaryClass={cs.LG},
      url={https://arxiv.org/abs/2503.06231}, 
}

@inproceedings{hhmm,
author = {Bui, Hung H. and Phung, Dinh Q. and Venkatesh, Svetha},
title = {Hierarchical hidden Markov models with general state hierarchy},
year = {2004},
isbn = {0262511835},
publisher = {AAAI Press},
booktitle = {Proceedings of the 19th National Conference on Artifical Intelligence},
pages = {324–329},
numpages = {6},
location = {San Jose, California},
series = {AAAI'04}
}

@inproceedings{gpht,
author = {Liu, Zhiding and Yang, Jiqian and Cheng, Mingyue and Luo, Yucong and Li, Zhi},
title = {Generative Pretrained Hierarchical Transformer for Time Series Forecasting},
year = {2024},
isbn = {9798400704901},
publisher = {Association for Computing Machinery},
address = {New York, NY, USA},
url = {https://doi.org/10.1145/3637528.3671855},
doi = {10.1145/3637528.3671855},
booktitle = {Proceedings of the 30th ACM SIGKDD Conference on Knowledge Discovery and Data Mining},
pages = {2003–2013},
numpages = {11},
location = {Barcelona, Spain},
series = {KDD '24}
}

@article{tsgbench,
author = {Ang, Yihao and Huang, Qiang and Bao, Yifan and Tung, Anthony K. H. and Huang, Zhiyong},
title = {{TSGBench}: Time Series Generation Benchmark},
year = {2023},
issue_date = {November 2023},
publisher = {VLDB Endowment},
volume = {17},
number = {3},
issn = {2150-8097},
url = {https://doi.org/10.14778/3632093.3632097},
doi = {10.14778/3632093.3632097},
journal = {Proc. VLDB Endow.},
month = nov,
pages = {305–318},
numpages = {14}
}

@misc{tabularargn,
      title={{TabularARGN}: A Flexible and Efficient Auto-Regressive Framework for Generating High-Fidelity Synthetic Data}, 
      author={Paul Tiwald and Ivona Krchova and Andrey Sidorenko and Mariana Vargas Vieyra and Mario Scriminaci and Michael Platzer},
      year={2025},
      eprint={2501.12012},
      archivePrefix={arXiv},
      primaryClass={cs.LG},
      url={https://arxiv.org/abs/2501.12012}, 
}

@article{DBLP:journals/natmi/TurbeBLM23,
  author       = {Hugues Turb{\'{e}} and
                  Mina Bjelogrlic and
                  Christian Lovis and
                  Gianmarco Mengaldo},
  title        = {Evaluation of post-hoc interpretability methods in time-series classification},
  journal      = {Nat. Mac. Intell.},
  volume       = {5},
  number       = {3},
  pages        = {250--260},
  year         = {2023},
  url          = {https://doi.org/10.1038/s42256-023-00620-w},
  doi          = {10.1038/S42256-023-00620-W},
  bibsource    = {dblp computer science bibliography, https://dblp.org}
}

@article{DBLP:journals/jbd/StengerLFKB24,
  author       = {Michael Stenger and
                  Robert Leppich and
                  Ian T. Foster and
                  Samuel Kounev and
                  Andr{\'{e}} Bauer},
  title        = {Evaluation is key: A survey on evaluation measures for synthetic time
                  series},
  journal      = {J. Big Data},
  volume       = {11},
  number       = {1},
  pages        = {66},
  year         = {2024},
  url          = {https://doi.org/10.1186/s40537-024-00924-7},
  doi          = {10.1186/S40537-024-00924-7},
  bibsource    = {dblp computer science bibliography, https://dblp.org}
}

@article{DBLP:journals/corr/abs-2506-06380,
  author       = {Jingyi Gu and
                  Xuan Zhang and
                  Guiling Wang},
  title        = {Beyond the Norm: {A} Survey of Synthetic Data Generation for Rare
                  Events},
  journal      = {CoRR},
  volume       = {abs/2506.06380},
  year         = {2025},
  url          = {https://doi.org/10.48550/arXiv.2506.06380},
  doi          = {10.48550/ARXIV.2506.06380},
  eprinttype    = {arXiv},
  eprint       = {2506.06380},
  bibsource    = {dblp computer science bibliography, https://dblp.org}
}

@article{DBLP:journals/mbec/MiljkovicMPS24,
  author       = {Nadica Miljkovic and
                  Nikola Milenic and
                  Nenad B. Popovic and
                  Jaka Sodnik},
  title        = {Data augmentation for generating synthetic electrogastrogram time
                  series},
  journal      = {Medical Biol. Eng. Comput.},
  volume       = {62},
  number       = {9},
  pages        = {2879--2891},
  year         = {2024},
  url          = {https://doi.org/10.1007/s11517-024-03112-0},
  doi          = {10.1007/S11517-024-03112-0},
  bibsource    = {dblp computer science bibliography, https://dblp.org}
}

@inproceedings{NEURIPS2023_05fb0f4e,
 author = {van Breugel, Boris and Seedat, Nabeel and Imrie, Fergus and van der Schaar, Mihaela},
 booktitle = {Advances in Neural Information Processing Systems},
 editor = {A. Oh and T. Naumann and A. Globerson and K. Saenko and M. Hardt and S. Levine},
 pages = {1889--1904},
 publisher = {Curran Associates, Inc.},
 title = {Can You Rely on Your Model Evaluation? {I}mproving Model Evaluation with Synthetic Test Data},
 url = {https://proceedings.neurips.cc/paper_files/paper/2023/file/05fb0f4e645cad23e0ab59d6b9901428-Paper-Conference.pdf},
 volume = {36},
 year = {2023}
}

@article{gpt1,
  title={Improving Language Understanding by Generative Pre-Training},
  author={Radford, Alec and Narasimhan, Karthik and Salimans, Tim and Sutskever, Ilya},
  year={2018},
  journal={OpenAI Technical Report},
  url={https://cdn.openai.com/research-covers/language-unsupervised/language_understanding_paper.pdf}
}

@InProceedings{resnet,
author = {He, Kaiming and Zhang, Xiangyu and Ren, Shaoqing and Sun, Jian},
title = {Deep Residual Learning for Image Recognition},
booktitle = {Proceedings of the IEEE Conference on Computer Vision and Pattern Recognition (CVPR)},
month = {June},
year = {2016}
}

@inproceedings{
loshchilov2018decoupled,
title={Decoupled Weight Decay Regularization},
author={Ilya Loshchilov and Frank Hutter},
booktitle={International Conference on Learning Representations},
year={2019},
url={https://openreview.net/forum?id=Bkg6RiCqY7},
}

@inproceedings{padts,
author = {Li, Yang and Meng, Han and Bi, Zhenyu and Urnes, Ingolv T. and Chen, Haipeng},
title = {Population aware diffusion for time series generation},
year = {2025},
isbn = {978-1-57735-897-8},
publisher = {AAAI Press},
url = {https://doi.org/10.1609/aaai.v39i17.34038},
doi = {10.1609/aaai.v39i17.34038},
booktitle = {Proceedings of the Thirty-Ninth AAAI Conference on Artificial Intelligence and Thirty-Seventh Conference on Innovative Applications of Artificial Intelligence and Fifteenth Symposium on Educational Advances in Artificial Intelligence},
articleno = {2065},
numpages = {10},
series = {AAAI'25/IAAI'25/EAAI'25}
}

@article{
timeautodiff,
title={{TimeAutoDiff}: A Unified Framework for Generation, Imputation, Forecasting, and Time-Varying Metadata Conditioning of Heterogeneous Time Series Tabular Data},
author={Namjoon Suh and Yuning Yang and Din-Yin Hsieh and Qitong Luan and Shirong Xu and Shixiang Zhu and Guang Cheng},
journal={Transactions on Machine Learning Research},
issn={2835-8856},
year={2025},
url={https://openreview.net/forum?id=bkUd1Dg46c},
note={J2C Certification}
}

@InProceedings{perciever,
  title = 	 {Perceiver: General Perception with Iterative Attention},
  author =       {Jaegle, Andrew and Gimeno, Felix and Brock, Andy and Vinyals, Oriol and Zisserman, Andrew and Carreira, Joao},
  booktitle = 	 {Proceedings of the 38th International Conference on Machine Learning},
  pages = 	 {4651--4664},
  year = 	 {2021},
  editor = 	 {Meila, Marina and Zhang, Tong},
  volume = 	 {139},
  series = 	 {Proceedings of Machine Learning Research},
  month = 	 {18--24 Jul},
  publisher =    {PMLR},
  url = 	 {https://proceedings.mlr.press/v139/jaegle21a.html}
}

@inproceedings{xu2026timemar,
  title={TimeMar: Multi-Scale Autoregressive Modeling for Unconditional Time Series Generation},
  author={Xu, Xiangyu and Zhong, Qingsong and Hu, Jilin},
  booktitle={Proceedings of the ACM Web Conference 2026},
  pages={5132--5143},
  year={2026}
}
